\documentclass{article}
\usepackage{hyperref}
\usepackage{url}
\usepackage{bbm}
\usepackage{amssymb}
\usepackage{enumitem}
\usepackage{algorithm, algorithmic}
\usepackage{wrapfig}
\usepackage{subfigure}
\usepackage{graphicx}
\usepackage{xcolor}
\usepackage{mathrsfs}
\usepackage{geometry}
\usepackage{booktabs}
\usepackage{multirow}
\usepackage{amsthm}
\usepackage{amsmath}

\newtheorem{theorem}{Theorem} % [section]

\newtheorem{definition}{Definition}

\definecolor{darkred}{RGB}{150,0,0}

\def\Algnameunderline{\textcolor{darkred}{T}ransferable \textcolor{darkred}{V}ision \textcolor{darkred}{E}xplainer}
\def\Algnameabbrnormal{TVE}
\def\Algnameabbr{\texttt{TVE}}
\def\Title{\Algnameabbrnormal{}: Learning Meta-attribution for Transferable Vision Explainer}
\title{\Title{}}

\begin{document}

\author{Guanchu Wang$^{1}$, 
Yu-Neng Chuang$^{1}$, 
Fan Yang$^{2}$, 
Mengnan Du$^{3}$,
Chia-Yuan Chang$^{4}$,
\\
Shaochen Zhong$^{1}$,
Zirui Liu$^{1}$,
Zhaozhuo Xu$^{5}$,
Kaixiong Zhou$^{6}$,
Xuanting Cai$^{7}$,
Xia Hu$^{1}$
\\
\small
$^1$Rice University,
$^2$Wake Forest University,
$^3$New Jersey Institute of Technology,
$^4$Texas A\&M University,
\\
\small
$^5$Stevens Institute of Technology,
$^6$North Carolina State University,
$^7$Meta Platforms, Inc.
\\
\small
\texttt{\{gw22,yc146,shaochen.zhong,zl105,xia.hu\}@rice.edu}; \texttt{yangfan@wfu.edu};
\texttt{mengnan.du@njit.edu};
\\
\small
\texttt{cychang@tamu.edu};
\texttt{zxu79@stevens.edu};
\texttt{kzhou22@ncsu.edu};
\texttt{caixuanting@meta.com}
}

\date{}
\maketitle

\begin{abstract}

\centering
\begin{minipage}{0.75\textwidth} 

Explainable machine learning significantly improves the transparency of deep neural networks.
However, existing work is constrained to explaining the behavior of individual model predictions, and lacks the ability to transfer the explanation across various models and tasks.
This limitation results in explaining various tasks being time- and resource-consuming.
To address this problem, we introduce a \Algnameunderline{}~(\Algnameabbr{}) that can effectively explain various vision models in downstream tasks.
Specifically, the transferability of \Algnameabbr{} is realized through a pre-training process on large-scale datasets towards learning the meta-attribution.
This meta-attribution leverages the versatility of generic backbone encoders to comprehensively encode the attribution knowledge for the input instance, which enables \Algnameabbr{} to seamlessly transfer to explain various downstream tasks, without the need for training on task-specific data.
Empirical studies involve explaining three different architectures of vision models across three diverse downstream datasets. 
The experimental results indicate \Algnameabbr{} is effective in explaining these tasks without the need for additional training on downstream data.
The source code is available at 
\url{https://github.com/guanchuwang/TVE}.

\end{minipage}

\end{abstract}

\section{Introduction}
\label{sec:intro}

Explainable machine learning (ML) contributes to enhancing the transparency of deep neural networks~(DNNs) for human comprehension~\cite{du2019techniques}.
It significantly facilitates the deployment of DNNs to high-stake scenarios where model explanations are required, such as loan approvals~\cite{steel2010web}, healthcare~\cite{chang2023towards}, and targeted advertisement~\cite{yang2018towards}.
In these fields, explainable DNN decisions are particularly important, given the practical needs of stakeholders and regulatory requirements, 
such as the General Data Protection Regulation~(GDPR)~\cite{goodman2017european}.

To overcome the black-box nature of DNNs, existing work of explainable ML can be categorized into two groups.
The first group of work focuses on constructing local explanation based on perturbation of the target black-box model, like LIME~\cite{ribeiro2016should}, GradCAM~\cite{selvaraju2017grad}, and Integrated Gradient~\cite{sundararajan2017axiomatic}.
These pieces of work rely on resource-intensive procedures like sampling or backpropagation of the target black-box model~\cite{liu2021synthetic}, leading to undesirable trade-off between the efficiency and interpretation fidelity~\cite{chuang2023efficient}.
Another group leverages the knowledge of explanation values to train DNN-based explainers, such as FastSHAP~\cite{jethani2021fastshap}, CORTX~\cite{chuang2023cortx}, and LARA~\cite{rong2023efficient, wang2022accelerating}.  
Such arts capable of efficiently generating explanations for an entire batch of instances through a single, streamlined feed-forward operation of the DNN-based explainer.
However, they are constrained to explaining individual black box models, and often lack the ability to transfer the explainer across various models or tasks.
These constraints lead to a time and resource-intensive process in practical scenarios, as they require the development and training of separate explainers for each specific task.

\begin{figure}
\centering
\!\!\!\!\!\!\!\!
\subfigure[]{
\includegraphics[width=0.4\linewidth]{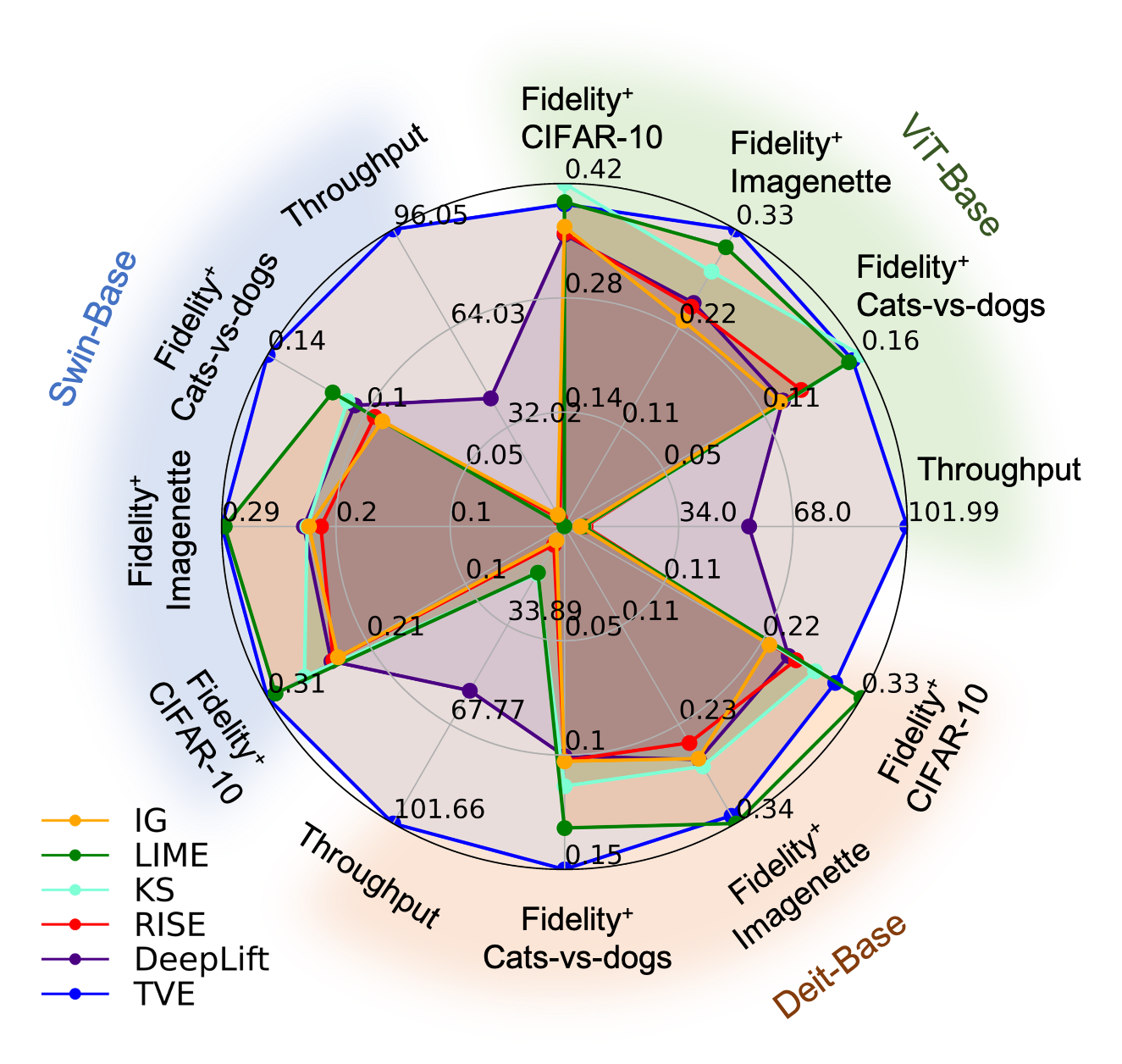}
\label{fig:fidelity-radar}
}
% \!\!\!\!\!\!\!\!
\subfigure[]{
\raisebox{0.1\height}{\includegraphics[width=0.5\textwidth]{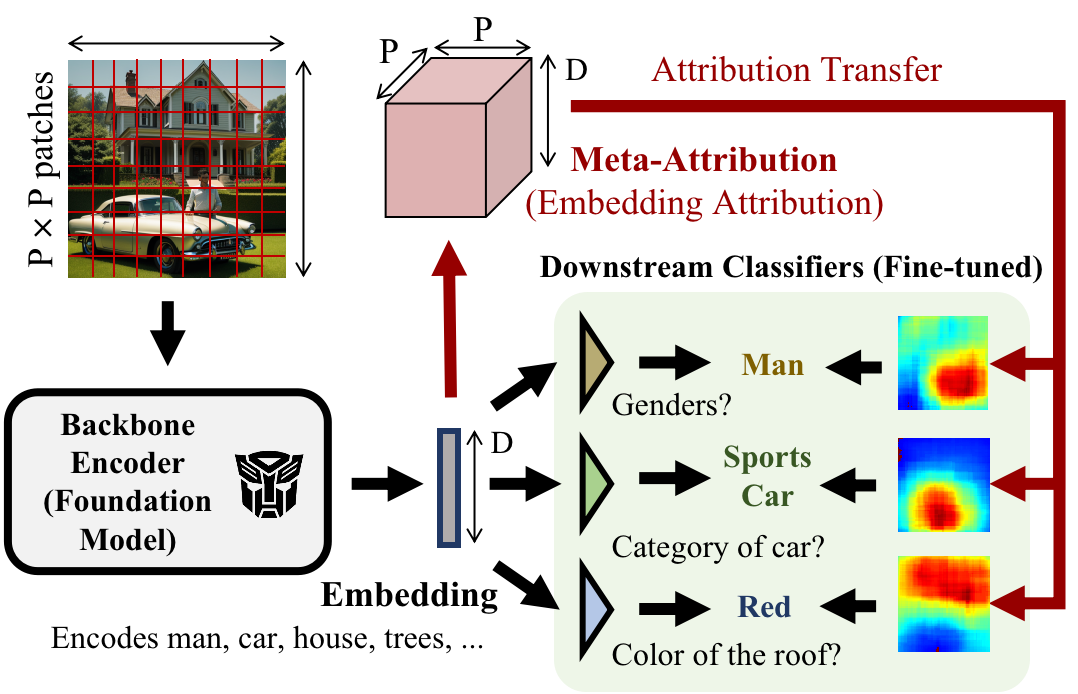}}
\label{fig:geneirc_xai_motivation}
}
\caption{Performance of \Algnameabbr{} in explaining ViT-B, Swin-B, and Deit-B on the Cats-vs-dogs, Imagenette, and CIFAR-10 datasets.
\textit{Fidelity$^+$ score} refers to the area under Fidelity$^+$-sparsity curve. (b) Illustration of attribution transfer. In this framework, the backbone can be a ViT encoder; and the downstream classifiers can be MLPs.
The embedding vector comprehensively encodes the features of input image.
Motivated by this, the meta-attribution comprehensively encapsulates the importance of each input patch to each element of the embedding vector.
This enables it to seamlessly transfer for explaining various downstream tasks.}
\end{figure}

To address the lack of transferability in explainers, we introduce a \Algnameunderline{} (\Algnameabbr{}).
The primary goal of \Algnameabbr{} is to achieve transferability through a pre-training process on large-scale image datasets, such that it can seamlessly explain various downstream tasks, as long as such tasks are within the scope of pre-training data distribution.
The construction of such transferable explainers introduces two non-trivial challenges:
\textbf{CH1.} Without task-specific exposure during the pre-training, how to ensure the universal effectiveness of explainer for various downstream tasks? 
\textbf{CH2.} How to adapt the explainer to a specific  task without fine-tuning on the task-specific data?

Our work effectively tackles these challenges.
To address CH1, we introduce a novel concept, named \textit{meta-attribution}, as a foundation for explaining various downstream tasks.
Specifically, the meta-attribution versatilely encodes the attribution knowledge for the input instance via exhaustively attributing each dimension of instance embedding. 
This knowledge is reusable for explaining various downstream tasks. 
It guides the pre-training of \Algnameabbr{} on large-scale image datasets, ensuring the universal effectiveness of \Algnameabbr{}.
After the pre-training, in response to CH2, we propose a \textit{transfer rule} to adapt the meta-attribution to explaining downstream tasks, without the need for additional training on task-specific data.
Figure~\ref{fig:fidelity-radar} shows the comprehensive performance of \Algnameabbr{} pre-trained on the {\small \texttt{ImageNet}} dataset and transferred to the {\small \texttt{Cats-vs-dogs}}, {\small \texttt{Imagenette}}, and {\small \texttt{CIFAR-10}} datasets, where \Algnameabbr{} shows competitive fidelity and efficiency compared with state-of-the-art methods.
To summarize, our work makes the following contributions:
\begin{itemize}[leftmargin=4mm, topsep=10pt]
% \setlength{\parskip}{1pt}
% \setlength{\parsep}{0pt}
% \setlength{\itemsep}{0pt}

    % \item We propose the definition of patch attribution and demonstrate its consistency with the infidelity of explanation from both theoretical and experimental perspectives.

    \item \textbf{Attribution transfer.}
    We propose a framework of attribution transfer, with a meta-attribution as foundations, and a transfer rule for explaining the downstream tasks.
    
    % for pre-training the transferable explainer.
    % Then, we propose a transfer rule to utilize the meta-attribution for explaining downstream tasks.
    % within the scope of pre-training data distribution.

    \item \textbf{Transferable explainer.} We build a transferable explainer \Algnameabbr{} that explains various downstream tasks without the need for training on the task-specific data.

    \item \textbf{Theoretical foundation.} \!\! We validate the~pre-training of \Algnameabbr{} can minimize the explanation error bound aligned with the $\mathcal{V}$-information-based explanation.
    
    % task-independent generic attribution aligned with the interpretation fidelity for \Algnameabbr{} that can be efficiently calculated and effectively guide the training of transferable explainer.
    % \item We propose a rule-based method to adopt the meta-attribution to explain downstream tasks, and show its theoretical consistency with the fidelity.
    
    % \item We propose \Algnameabbr{} to learn meta-attribution, and theoretically minimizing the explanation error bound on downstream tasks.
    
    % \Algnameabbr{} generates the task-aligned explanation based on a combination of the transferable explainer with the classification head of target model, without additional training on downstream tasks.

    % 
    % \item 

    \item \textbf{Competitive performance in explaining various downstream tasks.}
    The pre-trained \Algnameabbr{} shows promising results in explaining three architectures of vision Transformer across three downstream datasets.
    Significantly, the strong transferability of \Algnameabbr{} facilitates efficient and flexible deployment to various downstream scenarios.
    
    % compared to existing methods that are constrained to explaining individual models
    
    % the experiment results indicate \Algnameabbr{} can effectively explain the tasks without additional training on task-specific data.
    
    % to improve the faithfulness of explanations.
    
\end{itemize}

\section{Notations}

We introduce the notations for the problem formulation.

% In this section, and concept of attribution transfer.
% for the problem formulation and evaluation metrics of interpretation faithfulness and efficiency.
% that motivate the design of our GNN explanation framework.

% 
% \subsection{Notations}
\label{sec:notation}

\paragraph{Target model.}
We focus on the explanation of vision models$: \mathcal{X} \to \mathcal{Y}_t$ in this work, where $\mathcal{X} = \mathbb{N}^{\mathrm{W}\times\mathrm{W}}$ denote the spatial space of $\mathrm{W} \!\times\! \mathrm{W}$ pixels; $\mathbb{N}$ denote the space of a single pixel with three channels; and $\mathcal{Y}_t$ denotes the label space. 
Moreover, we follow most of existing work~\cite{he2022masked} and implementation of DNNs~\cite{wolf-etal-2020-transformers} to consider the target model as $f_{t} = H_t \circ G$, where the backbone encoder ${G}(\bullet) \! : \! \mathbb{N}^{\mathrm{W}\times\mathrm{W}} \! \to \! \mathbb{R}^{\mathrm{D}}$ is pre-trained on large-scale datasets; and the classifier ${H}_t(\bullet) \! : \! \mathbb{R}^{\mathrm{D}} \!\to\! \mathcal{Y}_t$~is finetuned on a specific task $t$.
It maybe worth noting that although we follow the transfer learning setting~\cite{chilamkurthy2017transfer, chen2020simple} to freeze the backbone encoder ${G}(\bullet)$ during the fine-tuning of $f_t$. 
Our experiment results in Section~\ref{sec:adaptation_finetune_backbone} further show that the proposed transferable explanation framework also shows effectiveness in the scenario where the target model is fully fine-tuned on downstream data.

% he2020momentum
% the classifier ${H}_t(\bullet)$ depends on a specific task $t$.

% This framework is effective for most existing vision models such as convolutional neural networks~\cite{he2016deep, huang2017densely} and vision transformers~\cite{dosovitskiy2020image, liu2021swin}.
% Thereby, the pre-training of ${G}(\bullet)$ is considered to be task-independent, and the finetuning of ${H}(\bullet)$ is task dependent.

% We also consider the classification head ${H}(\bullet): \mathcal{H} \to \mathcal{Y}$ finetuned on the downstream task, such that the classifier $f_{t} = {G} \circ {H}$.
% \footnote{The parameters of a generic image encoder can be downloaded from ??? or ???, or ???, etc.}

\paragraph{Image Patching.} 
We follow existing work~\cite{lundberg2017unified} to consider the patch-wise attribution of model prediction, i.e. the importance of each patch.
Specifically, we follow existing work~\cite{chuang2023cortx, jethani2021fastshap} to split each image $\boldsymbol{x}_k$ into $\mathrm{P} \!\times\! \mathrm{P}$ patches in a grid pattern, where each patch has $\mathrm{C} \!\times\! \mathrm{C}$ pixels; and $\mathrm{W} \!=\! \mathrm{C}\mathrm{P}$.
Let $\mathcal{Z}(\boldsymbol{x}_k) \!=\! \{ {z}_{i,j} | 1 \!\leq\! i, j \!\leq\! \mathrm{P} \}$ denote the patches of an image $\boldsymbol{x}_k \!\in\! \mathcal{X}$, where a patch ${z} \!\in\! \mathbb{N}^{\mathrm{C} \times \mathrm{C}}$ aligns with continuous $\mathrm{C} \!\times\! \mathrm{C}$ pixels of the image.
Moreover, we define $\mathcal{N}({z}) \!\subseteq\! \mathcal{Z}(\boldsymbol{x}_k)$ as the neighbors of a patch $z$ within the grid space, because a patch together with its neighbors have richer semantic content for model explanation. 
In this work, we follow the vision transformer~\cite{dosovitskiy2020image} to split the image patches with $\mathrm{P} \!=\! 14$ for $224 \!\times\! 224$ input images from the ImageNet dataset; and we consider $\mathcal{N}({z})$ as the zero-, one-, and two-hop neighbors of the patch $z$.

% where $\mathrm{d}(z, z')$ defines the number of patches between the patches $z$ and $z'$~(including $z'$).
% $\mathcal{N}({z}) = \{ z' ~|~ \mathrm{d}(z, z') \leq \xi \}$ 
% \footnote{\color{red} $\mathrm{C}$ controls the strength of inter-patch interaction taken into account for attribution. The inter-patch interaction is the representation of the semantic information provided by patch $z$ together with its neighbor patches $\mathcal{N}(z)$.}; and 

% The attribution of a patch ${z} \sim \mathcal{Z}(\boldsymbol{x}_k)$ is denoted as $\phi_{k, {z}}$.

% $i$-th row and $j$-th column of the image.

\paragraph{Model Perturbation.}
\!\!\!\! $f_{t}( \mathcal{S}; \boldsymbol{x}_k, y)$ represents the output~of $f_{t}$ on class $y$, with a perturbed instance as the input.
The patch subset $\mathcal{S} \!\subseteq\! \mathcal{Z}(\boldsymbol{x}_k)$ controls the perturbation. 
Specifically, the pixels belonging to the patches $z \!\in\! \mathcal{Z}(\boldsymbol{x}_k) \setminus \mathcal{S}$~are removed and take $0$, which is approximately the average value of normalized pixels.
For example, $f_{t}( \mathcal{N}(z); \!\boldsymbol{x}_k, y)$ defines the output of $f_t$ based on the perturbed input, where the pixels not belonging to the neighbors of patch $z$ take $0$.

\paragraph{Feature Attribution.}
This work focuses on the feature attribution of target models $f_t$ for providing explanations.
The feature attribution process involves generating importance scores, denoted as $\phi_{k,y,z}$ for each patch $z \in \mathcal{Z}(\boldsymbol{x}_k)$ of the input image $\boldsymbol{x}_k \in \mathcal{X}$, to indicate its importance to the model prediction $f_{t}( \mathcal{Z}(\boldsymbol{x}_k); \boldsymbol{x}_k, y)$ on class $y$.

\section{Feature Attribution can Transfer}
\label{sec:advantage}
% In this section, we propose the basic idea of attribution transfer, and show it advantages of flexible to explain various downstream tasks.

The motivation behind attribution transfer stems from model transfer in vision tasks~\cite{chilamkurthy2017transfer, he2020momentum, he2022masked}. 
Specifically, it arises from the observation that a generic backbone encoder possesses the capability to capture essential features of input images and represent them as embedding vectors.
This versatility enables the backbone to effectively adapt to a wide range of downstream tasks within the scope of pre-training data distribution.
As shown in Figure~\ref{fig:geneirc_xai_motivation}, information of \texttt{man}, \texttt{car}, \texttt{house} encoded in the embedding vector enables the detection of \texttt{gender}, \texttt{car}, and \texttt{building} in three different downstream scenarios, respectively.
Despite the demonstrated transferability of the backbone encoder, existing research has challenges in achieving `transferable explainer' across different tasks.
To bridge this gap and streamline the explanation process, we propose a  \emph{meta-attribution} that can be applied across various tasks, resulting in a significant reduction in the cost associated with generating explanations.

% a generic encoder, in essence, represents a potent image encoder that transforms input images into embedding vectors, often pre-trained using self-supervised learning on large-scale datasets to attain high capacity. e.g. the ViT, MAE, and ResNet encoders pre-trained on the ImageNet dataset. 
% for explaining downstream tasks requires to individually analyze the pipeline of pre-trained encoder and fine-tuned classifier in each task.
% In particular, the explaining process often relies on perturbations or gradient sampling of the pipeline, which is inefficient and complicated.

The \emph{meta-attribution} is defined as a tensor that versatilely encodes the reusable attribution knowledge for explaining downstream tasks.
As shown in Figure~\ref{fig:geneirc_xai_motivation}, we illustrate the meta-attribution as a three-dimensional tensor. 
A simple and effective method in this work is attributing the importance of input patches to each element of the embedding vector for the meta-attribution.
As shown in Figure~\ref{fig:geneirc_xai_motivation}, each $\mathrm{P} \!\times\! \mathrm{P}$ slice of this tensor corresponds to $\mathrm{P} \!\times\! \mathrm{P}$ patches within the input image, encoding their importance to a specific dimension of the embedding vector.
In this way, the meta-attribution inherits the adaptability of the embedding vector, making it versatile enough to adapt various explanation tasks in downstream scenarios.
For instance, the meta-attribution encodes the attribution knowledge for the \texttt{man} and \texttt{car} components encoded in the embedding vector, such that it can transfer to explain the \texttt{car} classification and \texttt{gender} detection in downstream scenarios.
The versatility of meta-attribution can effectively address the CH1 described in Section~\ref{sec:intro}.
We formalize the attribution transfer in Sections~\ref{sec:define_attr_transfer}.

% Instead of individually explaining each task, we propose to generate a meta-attribution and subsequently utilizing this versatile information to explain various downstream tasks.

% most of the image's key components, these essential attributions are embedded within the meta-attribution tensors.

% , as depicted in Figure~\ref{fig:geneirc_xai_motivation}.

%%%%%%%%%%%%%%%%%%%%%%%%%%%%%%%%%%%%%%%%%%%%%%%%%%%%%%%%%%%%%%%%%%%%%%
%%%%%%%%%%%%%%%Move to the Introduction%%%%%%%%%%%%%%%%
%%%%%%%%%%%%%%%%%%%%%%%%%%%%%%%%%%%%%%%%%%%%%%%%%%%%%%%%%%%%%%%%%%%%%%

% As the generic encoder adapts to a wide range of downstream tasks, 
% This allows our framework to be easily deployed to real-world scenarios.

% the advantages of generic encoder
% e.g. as shown in Figure~\ref{fig:geneirc_xai_motivation}, 
% that can transfer to different downstream tasks
% This is termed as meta-attribution in this work.
% Figure~\ref{fig:geneirc_xai_motivation}
% The embedding 

\section{Meta-attribution Transfer}
\label{sec:define_attr_transfer}

In this section, we begin with the explanation definition by following the $\mathcal{V}$-information theory~\cite{xu2020theory, hewitt2021conditional, chen2022rev}.
Then, we introduce the definition of meta-attribution in Definition~\ref{definition:transferable_attribution}. 
Finally, we propose a transfer rule to adapt the meta-attribution to explaining specific downstream tasks in Definition~\ref{definition:attribution_transfer}.

% Finally, we demonstrate its consistency with the fidelity in Theorem~\ref{theorem:fidelity}, which provides a theoretical foundation for evaluating the explanation performance of our proposed attribution definition.

% In this section, we first propose the definition of generic attribution in Definition~\ref{definition:generic_explanation}.
% Then, we show its alignment with Condition~\ref{condition:task_agonistic} in Remark~\ref{rk:satisfy_condition1}. 
% This indicates the latent attribution of Definition~\ref{definition:generic_explanation} can be disentangled from downstream tasks. 
% Finally, we show its consistency with the fidelity in Theorem~\ref{theorem:patch-attribution}, which is required by condition~\ref{condition:fidelity_alignment}.
% This provides a theoretical guarantee of interpretation performance for the proposed definition of attribution.

\subsection{$\mathcal{V}$-Information-based Explanation}
\label{sec:v-information}

The importance of a patch $z \!\in\! \mathcal{Z}(\boldsymbol{x}_k)$ to downstream~model $f_t(\boldsymbol{x}_k)$~is formulated into the conditional mutual~information $I(\mathcal{N}(z); Y_t ~|~ B)$ between $\mathcal{N}(z)$ and $Y_t$, given~the state of remaining patches $B \!\subseteq\! \mathcal{Z}(\boldsymbol{x}_k) \!\setminus\! \mathcal{N}(z)$~\cite{chen2022rev}.
Here, $Y_t \!\sim\! f_t(\boldsymbol{x}_k)$ denotes the variable corresponding to~the model ouput.
However, estimating this mutual information accurately poses a challenge due to the unknown distribution of $\mathcal{N}(z)$ and $B$.
To address this challenge, we adopt an information-theoretic framework introduced in works by~\cite{xu2020theory, hewitt2021conditional}, known to as conditional $\mathcal{V}$-information $I_{\mathcal{V}}(\mathcal{N}(z) \!\to\! Y_t)$.
In particular, it redirects the computation of mutual information to a certain predictive model within function space $\mathcal{V}$, as defined by:
{
\begin{align}
I_{\mathcal{V}}(\mathcal{N}(z) \to Y_t ~|~ B ) = H_{\mathcal{V}}(Y_t ~|~ B ) - H_{\mathcal{V}}(Y_t ~|~ \mathcal{N}(z), B),
\nonumber
\end{align}} 
\!\!\! where $\mathcal{V}$-entropy $H_{\mathcal{V}}(Y_t ~|~ B )$ takes the lowest entropy over the function space $\mathcal{V}$, which is given by
{
\begin{align}
\label{eq:V_entropy}
H_{\mathcal{V}}(Y_t ~|~ B ) = \inf_{f \in \mathcal{V}} \mathbb{E}_{y \sim \mathcal{Y}_t} [-\log f( B ; \boldsymbol{x}_k, y)].
\!\!\!\!\!\!\!\!\!\!
\end{align}}
% where $ \inf_{f \in \mathcal{V}} [\bullet]$ takes the minimal entropy over the function space $f \in \mathcal{V}$.

Note that the $\mathcal{V}$-information explanation should align with a pre-trained target model $f_{t} \in \mathcal{V}$ and a specific class label $y$.
The $\mathcal{V}$-entropy should take its value at $f_{t}$ and $y$, instead of the infimum expectation value, for the explanation.
Therefore, we relax the $\mathcal{V}$-entropy terms $H_{\mathcal{V}}(Y_t ~|~ B)$ and $H_{\mathcal{V}}(Y_t ~|~ \mathcal{N}(z), B)$ into $-\log f_{t}(B ; \boldsymbol{x}_k, y)$ and $-\log f_{t}(\mathcal{N}(z) \cup B ; \boldsymbol{x}_k, y)$, respectively~\cite{ethayarajh2022understanding}, for aligning the explanation with the target model $f_{t} \in \mathcal{V}$ and class label $y$.
In this way, the attribution of patch $z$ aligned with class $y$ is defined as follows:
{
\begin{align}
\label{eq:information_attribution3}
\!\!\!\! \phi_{k, y, {z}} &= \mathbb{E}_{B \subseteq \mathcal{Z}(\boldsymbol{x}_k) \setminus \mathcal{N}(z)} [ -\log f_{t}(B ; \boldsymbol{x}_k, y) \!+\! \log f_{t}(\mathcal{N}(z) \!\cup\! B ; \boldsymbol{x}_k, y) ]. \!\!
\end{align}}

% Here, $-\log f_{t}(B ; \boldsymbol{x}_k, y)$ is also to the infimum value $\inf_{f \in \mathcal{V}} [-\log f( B ; \boldsymbol{x}_k, y)]$ as well, because the target model $f_t$ has been well trained on downstream data $\boldsymbol{x}_k \in \mathcal{D}_{t}$ towards minimizing the cross-entropy aligned with the labels, i.e. $f_{t} = \arg\min_{f \in \mathcal{V}} \mathbb{E}_{y \sim \mathcal{Y}} [-\log f(\boldsymbol{x}_k, y)]$.
% , where $y$ takes the predicted label
% Regarding the patches $B \subseteq \mathcal{Z}(\boldsymbol{x}_k) \setminus \mathcal{N}(z)$, we consider the average of ???

It is impossible to enumerate the state of $B$ over $B \subseteq \mathcal{Z}(\boldsymbol{x}_k) \setminus \mathcal{N}(z)$ in Equation~(\ref{eq:information_attribution3}).
We follow existing work \cite{mitchell2022sampling} to approximate it into two antithetical states to simplify the computation~\cite{mitchell2022sampling}. These cases involve considering the state of $B$ to be entirely remaining patches $\mathcal{Z}(\boldsymbol{x}_k) \setminus \mathcal{N}(z)$ or empty set $\varnothing$, narrowing down the enumeration of $B \subseteq \mathcal{Z}(\boldsymbol{x}_k) \setminus \mathcal{N}(z)$ to $B \sim \{ \mathcal{Z}(\boldsymbol{x}_k) \setminus \mathcal{N}(z), \varnothing \}$ in Equation~(\ref{eq:information_attribution3}).
Based on our numerical studies in Appendix~\ref{appendix:B-sampling-experiment}, the approximate attribution shows positive correlation with the exact value, which indicates the approximation does not affect the quality of attribution.
To summarize, we approximate the attribution value of patch $z$ aligned with class $y$ as follows:
{
\begin{align}
\label{eq:information_attribution1}
\!\!\!\!\phi_{k, y, {z}} &\approx \mathbb{E}_{B \sim \{ \mathcal{Z}(\boldsymbol{x}_k) \setminus \mathcal{N}(z), \varnothing \}} [ -\log f_{t}(B ; \boldsymbol{x}_k, y)  + \log f_{t}(\mathcal{N}(z) \cup B ; \boldsymbol{x}_k, y) ],
\\
\label{eq:information_attribution2}
&\!\!\!\! \sim  \log f_{t}(\mathcal{N}(z) ; \boldsymbol{x}_k, y) \!-\! \log f_{t}(\mathcal{Z}(\boldsymbol{x}_k) \!\setminus\! \mathcal{N}(z) ; \boldsymbol{x}_k, y), \!\!\!\!\!\!\!\!
\end{align}}
\!\! where the terms $f_{t}(\mathcal{Z}(\boldsymbol{x}_k) ; \boldsymbol{x}_k, y)$ and $f_{t}(\varnothing ; \boldsymbol{x}_k, y)$ in Equation~(\ref{eq:information_attribution1}) are constant given $\boldsymbol{x}_k$ and $y$, thus being omitted in Equation~(\ref{eq:information_attribution2}).
Intuitively, the explanation of patch $z$ depends on the gap of logit values, where  $\mathcal{N}(z)$ and background patches $\mathcal{Z}(\boldsymbol{x}_k) \setminus \mathcal{N}(z)$ are taken as the input.

\subsection{Definition of Meta-attribution}
\label{sec:transferable_attr}

We introduce the concept of meta-attribution, formally defined in Definition~\ref{definition:transferable_attribution}. 
Note that Equation~(\ref{eq:information_attribution2}) relies on the downstream target model $f_{t}$, which is task-related.
The purpose of meta-attribution is to disentangle the task-specific aspect of the attribution from Equation~(\ref{eq:information_attribution2}). 
This disentanglement renders the meta-attribution to be task-independent, as a foundation for explaining various tasks.

\begin{definition}[\textbf{Meta-attribution}]
\label{definition:transferable_attribution}
\! Given a backbone encoder $G$, the meta-attribution for a patch $z \!\in\! \mathcal{Z}(\boldsymbol{x}_k)$,~$\boldsymbol{x}_k \!\in\! \mathcal{X}$, is represented by two tensors $\textbf{\textsl{g}}_{k,z}$ and $\textbf{\textrm{h}}_{k,z}$ as follows:
{
\begin{align}
\label{eq:transferable_attribution}
\textbf{\textsl{g}}_{k,z} &= {G}( \mathcal{N}({z}); \boldsymbol{x}_k ),
\nonumber
\\
\textbf{\textrm{h}}_{k,z} &= {G}( \mathcal{Z}(\boldsymbol{x}_k) \setminus \mathcal{N}({z}); \boldsymbol{x}_k ).
\end{align}}
\end{definition}

% , and propose a rule-based method for explaining downstream tasks in Definition~\ref{definition:attribution_transfer}.
% replacing the $f_t$ from

Following Definition~\ref{definition:transferable_attribution}, the meta-attribution is defined as the input tensors of the logarithmic functions in Equation~(\ref{eq:information_attribution2}), where the task-specific model $f_t$ is replaced into the backbone encoder $G$ to disentangle the meta-attribution with specific tasks.
This disentanglement enables the meta-attribution to transfer across various downstream tasks.

% is derived using the backbone encoder $G$, facilitating its application across different tasks.
% which is consistent with our illustration in as in Section~\ref{sec:advantage}.
% Intuitively, the patch-level meta-attribution is defined as the perturbed output of generic encoder, where the input image is masked by the neighbors of target patch $z$, as given in Definition~\ref{definition:transferable_attribution}.
% Since each input image $\boldsymbol{x}_k$ is splitted into $\mathrm{P} \times \mathrm{P}$ patches, 
% Moreover, Definition~\ref{definition:transferable_attribution} is independent with the task-aligned classifier $H$, which enables the meta-attribution $\textbf{\textsl{g}}_{k}$ and $\textbf{\textrm{h}}_{k}$ to adapt to multiple downstream tasks.

\subsection{Transfer to Task-aligned Explanation}
\label{sec:attr_transfer}

% \begin{definition}[Task-specific Attribution]
% \label{definition:attribution_transfer}
% For a patch $z \in \mathcal{Z}(\boldsymbol{x}_k)$, $\boldsymbol{x}_k \in \mathcal{X}$, the attribution of task-specific inference $f_{t}(\boldsymbol{x}_k)$ is defined as follow:
% \begin{align}
% \phi_{k, {z}}
% &= \mathbb{E}_{{z}' \sim \mathcal{N}({z})} \big[ f_{t}( \mathbbm{1}_{\mathcal{N}({z}')}  \odot \boldsymbol{x}_k) - f_{t}( \mathbbm{1}_{\mathcal{Z}(\boldsymbol{x}_k) \setminus \mathcal{N}({z}')}  \odot \boldsymbol{x}_k) \big]
% % \\
% % &= \mathbb{E}_{{z}' \sim \mathcal{N}({z})} \big[ (G \circ H )( \mathbbm{1}_{\mathcal{N}({z}')}  \odot \boldsymbol{x}_k) - (G \circ H )(  \mathbbm{1}_{\mathcal{Z}(\boldsymbol{x}_k) \setminus \mathcal{N}({z}')}  \odot \boldsymbol{x}_k) \big], 
% % \nonumber
% % \\
% % &= \mathbb{E}_{{z}' \sim \mathcal{N}({z})} \big[ H(\textbf{\textsl{g}}_{k,z}) - H(\textbf{\textrm{h}}_{k,z}) \big], 
% \end{align}
% where we consider $\mathcal{N}({z}) = \{ z' ~|~ \mathrm{d}(z, z') \leq \mathrm{C} \xi \}$ for neighbor patches of ${z}$; and $\mathrm{d}(z, z')$ denotes the distance central pixel.
% \end{definition}
% To generate the explanation for downstream tasks from the meta-attribution,
To explain the downstream tasks, we propose a transfer rule in Definition~\ref{definition:attribution_transfer} to adapt the meta-attribution to explaining downstream tasks.
This rule-based transfer method can effectively address the CH2 described in Section~\ref{sec:intro}, without the need for additional training on task-specific data. 
% The theoretical analysis of Definition~\ref{definition:attribution_transfer} is given in Theorem~\ref{theorem:patch-attribution}.

% 
\begin{definition}[\textbf{Attribution Transfer}]
\label{definition:attribution_transfer}
If the task-specific function is given by $f_{t} = H_t \circ G$, then the explanation of $f_{t}(\boldsymbol{x}_k)$ on class $y$ is generated by
{
\begin{align}
\label{eq:attribution_trasfer}
% \phi_{k, {z}} = \frac{1}{2\xi^2} \sum_{{z}' \in \mathcal{N}({z})} \sum_{i \in \{ 1,-1 \} } (G \circ H )\Big( \big[ (i+1) \mathbbm{1}_{\mathcal{N}({z})} + (i-1) \mathbbm{1}_{\mathcal{Z}(\boldsymbol{x}_k) \setminus \mathcal{N}({z}')} \big]  \odot \boldsymbol{x}_k \Big) 
% \phi_{k, {z}} = \mathbb{E}_{{z}' \sim \mathcal{N}({z})} \big[ \log H(\textbf{\textsl{g}}_{k,z}; y) - \log H(\textbf{\textrm{h}}_{k,z}; y) \big],
\phi_{k, y, {z}} = \log H_t(\textbf{\textsl{g}}_{k,z}; y) - \log H_t(\textbf{\textrm{h}}_{k,z}; y),
\end{align}} 
\!\!\!\! where $\textbf{\textsl{g}}_{k,z}$ and $\textbf{\textrm{h}}_{k,z}$ are the meta-attribution given by Equation~(\ref{eq:transferable_attribution}); and $G$ and $H_t$ represent the backbone encoder and fine-tuned classifier on task $t$, respectively.
% Proof in Appendix~???.
\end{definition}

% for patch $z \in \mathcal{Z}(\boldsymbol{x}_k)$;
% According to Proposition~\ref{definition:attribution_transfer}, 

Following Definition~\ref{definition:attribution_transfer}, we can straightforwardly achieve the solution of $\phi_{k,y,z}$ to be consistent with Equation~(\ref{eq:information_attribution2})\footnote{We follow Definition~\ref{definition:attribution_transfer} to have $\phi_{k,y,z} = \log H_t(\textbf{\textsl{g}}_{k,z}; y) - \log H_t(\textbf{\textrm{h}}_{k,z}; y) = \log f_{t}(\mathcal{N}(z) ; \boldsymbol{x}_k, y) - \log f_{t}(\mathcal{Z}(\boldsymbol{x}_k) \setminus \mathcal{N}(z) ; \boldsymbol{x}_k, y)$ that is consistent with Equation~(\ref{eq:information_attribution2}).}.
This alignment to $\phi_{k,y,z}$ can effectively explain downstream task $t$ following the definition of conditional $\mathcal{V}$-information $I_{\mathcal{V}}(\mathcal{N}(z) \to Y_t ~|~ B)$, as described in Section~\ref{sec:v-information}.

\section{Learning Meta-attribution}

In this section, we introduce the details of \Algnameunderline{}~(\Algnameabbr{}).
Specifically, \Algnameabbr{} pre-trains a DNN-based transferable explainer $E(\bullet ~|~ \theta)$ on large-scale image dataset to comprehensively learn the knowledge of meta-attribution. 
After the pre-training, \Algnameabbr{} can transfer to various downstream tasks for end-to-end generating task-aligned explanation.
To assess its performance, we theoretically analyze
the explanation error in Theorem~\ref{theorem:head-attribution}.

% , which is significantly more efficient than individual explaining the tasks
% defined in Definition~\ref{definition:transferable_attribution}
% In this way, the downstream explanation can be generated though a pipeline of generating meta-attribution and rule-based adaptation to downstream tasks, which is significantly more efficient than individual explaining different tasks.
% Moreover, the pre-trained transferable explainer can adapt to a
% wide range of downstream tasks where the generic encoder has been deployed,
% without additional training processes. This adaptability enhances practicality
% for real-world applications
% Following this direction, \Algnameabbr{} employs a DNN-based transferable explainer $E(\bullet ~|~ \theta)$ to generate the perturbed embedding.

% In this way, the transferable explainer $E(\bullet ~|~ \theta)$ can be deployed to various downstream tasks for explaining the target models \emph{without additional training on the downstream data}.

% Notably, the training process of $E(\bullet ~|~ \theta)$ is decoupled from the task-specific classifier $H(\bullet)$. 
% 

\subsection{Explainer Pre-training}

% Moreover, let $\textbf{\textsl{g}}_{k} = [\textbf{\textsl{g}}_{k,z} ~|~ z \in \mathcal{Z}(\boldsymbol{x}_k)]$ and $\textbf{\textrm{h}}_{k} = [\textbf{\textrm{h}}_{k,z} ~|~ z \in \mathcal{Z}(\boldsymbol{x}_k)]$ denote the meta-attribution for the instance $\boldsymbol{x}_k$, which is a collection of patch-level meta-attribution given in Definition~\ref{definition:transferable_attribution}.
% Representing the meta-attribution for instance $\boldsymbol{x}_k$, $\textbf{\textsl{g}}_{k}$ and $\textbf{\textrm{h}}_{k}$ are 3D tensors in the shape of $\mathrm{P} \times \mathrm{P} \times \mathrm{D}$, where $\mathrm{D}$ denotes the output dimension of the backbone encoder $G$, as defined in Section~\ref{sec:notation}. 

% : \mathcal{X} \to \mathbb{R}^{2 \times \mathrm{P} \times \mathrm{P} \times \mathrm{D}}
 % \in \mathbb{R}^{\mathrm{P} \times \mathrm{P} \times \mathrm{D}}

\Algnameabbr{} employs a DNN-based explainer $E(\bullet ~|~ \theta)$ to generate the meta-attribution tensors.
Specifically, the explainer $E(\bullet \mid \theta)$ produces two tensors for the meta-attribution, denoted as $[\hat{\textbf{\textsl{g}}}_{k}, \hat{\textbf{h}}_{k}] = E(\boldsymbol{x}_k \mid \theta)$, where $\hat{\textbf{\textsl{g}}}_{k} \!=\! [\hat{\textbf{\textsl{g}}}_{k, z} \!\in\! \mathbb{R}^{\mathrm{D}} | z \!\in\! \mathcal{Z}(\boldsymbol{x}_k)]$ and $\hat{\textbf{h}}_{k} \!=\! [\hat{\textbf{h}}_{k, z} \!\in\! \mathbb{R}^{\mathrm{D}} | z \!\in\! \mathcal{Z}(\boldsymbol{x}_k)]$ represent collections of meta-attribution for an instance $\boldsymbol{x}_k$.
Each pair of elements $( \hat{\textbf{\textsl{g}}}_{k, z}, \hat{\textbf{h}}_{k, z} )$ contribute to predicting the meta-attribution $( \textbf{\textsl{g}}_{k, z}, \textbf{h}_{k, z})$ defined in Definition~\ref{definition:transferable_attribution}.
Pursuant to this objective, \Algnameabbr{} updates the parameters of explainer $E(\bullet \mid \theta)$ to minimize the following loss function:
{
\setlength\abovedisplayskip{2mm}
\setlength\belowdisplayskip{2mm}
\begin{align}
% \vspace{2mm}
\label{eq:pretrain_loss_function}
% \mathcal{L}_{\theta}(\boldsymbol{x}_k) = \mathbb{E}_{{z} \sim \mathcal{Z}(\boldsymbol{x}_k)} \Big[ \big|\big| \hat{\textbf{\textsl{g}}}_{k, {z}} - {G}( \mathbbm{1}_{\mathcal{N}({z})}  \odot \boldsymbol{x}_k ) \big|\big|_2^2 
% + \big|\big| \hat{\textbf{h}}_{k, {z}} - {G}( \mathbbm{1}_{\mathcal{Z}(\boldsymbol{x}_k) \setminus \mathcal{N}({z})}  \odot \boldsymbol{x}_k ) \big|\big|_2^2 \Big]
\!\! \mathcal{L}_{\theta}(\boldsymbol{x}_k) \!=\! \mathbb{E}_{{z} \sim \mathcal{Z}(\boldsymbol{x}_k)} \big[ || \hat{\textbf{\textsl{g}}}_{k, {z}} \!-\! \textbf{\textsl{g}}_{k, {z}} \big|\big|_2^2 
\!+\! \big|\big| \hat{\textbf{h}}_{k, {z}} \!-\! \textbf{h}_{k, {z}} ||_2^2 \big], \!\!
\end{align}}
\!\! where $\textbf{\textsl{g}}_{k, {z}}$ and $\textbf{h}_{k, {z}}$ are defined in Definition~\ref{definition:transferable_attribution}.

Algorithm~\ref{alg:pretrain} summarizes one epoch of pre-training the transferable explainer $E(\bullet \mid \theta)$.
Specifically, \Algnameabbr{} first samples a mini-batch of image patches~(lines 2); then follows Definition~\ref{definition:transferable_attribution} to generate the meta-attribution~(lines 3); finally updates the parameters of $E(\bullet \mid \theta)$ to minimize the loss function given by Equation~(\ref{eq:pretrain_loss_function})~(line 4).
The iteration ends with the convergence of $E(\bullet \mid \theta)$.
Notably, the pre-training of $E(\bullet ~|~ \theta)$ is guided by the meta-attribution instead of specific tasks.
This empowers the trained $E(\bullet ~|~ \theta)$ to remain impartial towards specific tasks, providing the flexibility for seamless adaptation across various downstream tasks.

% output of backbone encoder $G$ 
% such that it has the capacity of being transferred to multiple downstream tasks after the pre-training.
% According to Equation~(\ref{eq:pretrain_loss_function}), 

% transferable explainer
% \begin{wrapfigure}{R}{0.99\textwidth}
% \vspace{-15mm}
\begin{flushright}
% \begin{minipage}{0.99\textwidth}

\begin{algorithm}
\caption{One epoch of \Algnameabbr{} pre-training}
\label{alg:pretrain}
\textbf{Input:} Pre-training dataset $\mathcal{D}$. \\
\textbf{Output:} Transferable explainer $E(\bullet \mid \theta^*)$.\\
\vspace{-3mm}
\begin{algorithmic}[1]

\FOR{$\boldsymbol{x}_k \sim \mathcal{D}$}

% \STATE Sample an image $\boldsymbol{x}_k \sim \mathcal{D}$.
% \STATE Sample images $\boldsymbol{x}_k \sim \mathcal{X}$.

\STATE Sample patches ${z} \sim \mathcal{Z}(\boldsymbol{x}_k)$.

\STATE Generate $\textbf{\textsl{g}}_{k,z}$ and $\textbf{h}_{k, {z}}$ following Definition~\ref{definition:transferable_attribution}.
% \begin{equation} 
% \pi(\boldsymbol{x}_k, {z}) = {G}\big( \mathbbm{1}_{\mathcal{N}({z})}  \odot \boldsymbol{x}_k ~\big|~ \theta_{G} \big) \otimes \mathbbm{1}_{\{ {z} \}}
% \nonumber
% \end{equation}

% \STATE Predict the generic attribution 
% \begin{equation} 
% \hat{\pi}(\boldsymbol{x}_k, {z}) = E(\boldsymbol{x}_k \mid \theta) \odot ( \mathbbm{1}_{\mathrm{D}} \otimes \mathbbm{1}_{\{ {z} \}} )
% \nonumber
% \end{equation}

\STATE Update $E(\bullet ~|~ \theta)$ to minimize Equation~(\ref{eq:pretrain_loss_function}).
% \begin{equation} 
%     \theta^*_{\textsl{E}} = \arg\min_{\theta \in \Theta} \big|\big| \hat{\pi}(\boldsymbol{x}_k, {z}) - \pi(\boldsymbol{x}_k, {z}) \big|\big|_2^2.
%     \nonumber
% \end{equation}
% Minimize the loss function given by Equation~(\ref{eq:pretrain_loss_function}).

\ENDFOR
\end{algorithmic}
\end{algorithm}
% \end{minipage}

\end{flushright}
% \vspace{-6mm}
% \end{wrapfigure}

\subsection{Generating Task-aligned Explanation}

\Algnameabbr{} follows Definition~\ref{definition:attribution_transfer} to generate the task-aligned explanation.
Specifically, to explain the inference process $(H_t \circ G)(\boldsymbol{x}_{k})$ in task $t$, \Algnameabbr{} first adopts the pre-trained transferable explainer to generate the meta-attribution $[\hat{\textbf{\textsl{g}}}_{k}, \hat{\textbf{h}}_{k}] = E(\boldsymbol{x}_k \mid \theta)$; then takes the value of $\hat{\textbf{\textsl{g}}}_{k, z}$ and $\hat{\textbf{h}}_{k, z}$ into Equation~(\ref{eq:attribution_trasfer}) to estimate the importance of each patch $z \in \mathcal{Z}(\boldsymbol{x}_k)$ to the inference result on class $y$.
To summarize, \Algnameabbr{} generates the attribution of a patch $z \in \mathcal{Z}(\boldsymbol{x}_k)$ by
{\setlength\abovedisplayskip{2mm}
\setlength\belowdisplayskip{2mm}\begin{align}
\label{eq:estimate_generic_attribution}
\hat{\phi}_{k, y, {z}}
= \log H_t( \hat{\textbf{\textsl{g}}}_{k, z}; y ) - \log H_t(  \hat{\textbf{h}}_{k, z}; y ).
\end{align}}
\!\! Let $\hat{\boldsymbol{\phi}}_{k,y} = [\hat{\phi}_{k,y,z} ~|~ z \!\in\! \mathcal{Z}(\boldsymbol{x}_k)]$ denote the $\mathrm{P} \!\times\! \mathrm{P}$ explanation heatmap for the image $\boldsymbol{x}_k$, indicating the importance of all patches in $\boldsymbol{x}_k$ to class $y$.
\Algnameabbr{} can efficiently generate the entire heatmap $\hat{\boldsymbol{\phi}}_{k,y}$ for the image $\boldsymbol{x}_k$ through a single feed forward pass
: $\hat{\boldsymbol{\phi}}_{k,y} = \log H_t( \hat{\textbf{\textsl{g}}}_{k}; y ) - \log H_t( \hat{\textbf{h}}_{k}; y )$, where $\hat{\textbf{\textsl{g}}}_{k}$ and $\hat{\textbf{h}}_{k}$ are generated by $[\hat{\textbf{\textsl{g}}}_{k}, \hat{\textbf{h}}_{k}] = E(\boldsymbol{x}_k ~|~ \theta)$.

% , where each element $\hat{\phi}_{k,y,z}$ is given by Equation~(\ref{eq:estimate_generic_attribution})
% {\setlength\abovedisplayskip{2mm}
% \setlength\belowdisplayskip{2mm}\begin{align}
% \label{eq:heatmap}
% \hat{\boldsymbol{\phi}}_{k,y} = \log H_t( \hat{\textbf{\textsl{g}}}_{k}; y ) - \log H_t( \hat{\textbf{h}}_{k}; y ).
% \end{align}}

In particular, $H_t( \bullet; y )$ in Equation~(\ref{eq:estimate_generic_attribution}) encodes the knowledge of downstream task $t$.
This knowledge significantly enables the explanation to align with the task $t$ \emph{without the need for additional training on the task-specific data}.

\subsection{Theoretical Analysis}

The theoretical analysis focuses on understanding the behavior of estimation error $|\hat{\phi}_{k, y, {z}} - \phi_{k, y, {z}}|$ during the \Algnameabbr{} pre-training, where $\phi_{k, y, {z}}$ takes the $\mathcal{V}$-Information-aligned explanation defined in Section~\ref{sec:v-information}.
Specifically, we examine the following two distinct cases to understand how the reduction in the pre-training loss function $\mathcal{L}_{\theta}(\boldsymbol{x}_k)$ diminishes the estimation error~{\small $|\hat{\phi}_{k, y, {z}} \!-\! \phi_{k, y, {z}}|$}.

% As the pre-training loss function  $\mathcal{L}_{\theta}(\boldsymbol{x}_k)$ approaches zero, Equation~(\ref{eq:pretrain_loss_function}) indicates $\hat{\textbf{\textsl{g}}}_{k,z}$ and $\hat{\textbf{h}}_{k,z}$ gradually converge~towards $\textbf{\textsl{g}}_{k,z}$ and $\textbf{h}_{k,z}$, respectively.

% 
\paragraph{Ideal Case.}
We ideally consider $\mathcal{L}_{\theta}(\boldsymbol{x}_k) \!\!\to\!\! 0$ in this case. 
According to Equation~(\ref{eq:pretrain_loss_function}), we have that $\hat{\textbf{\textsl{g}}}_{k,z} \!\!\to\!\! \textbf{\textsl{g}}_{k,z}$ and $\hat{\textbf{h}}_{k,z} \!\!\to\!\! \textbf{h}_{k,z}$. 
Then, the relations {\tiny $\frac{{H}_t(\hat{\textbf{\textsl{g}}}_{k,z}; y)}{{H}_t(\textbf{\textsl{g}}_{k,z}; y)} \!\!\to\!\! 1$} and {\tiny $\frac{{H}_t(\textbf{h}_{k,z}; y)}{{H}_t(\hat{\textbf{h}}_{k,z}; y)} \!\!\to\!\! 1$} are established.
In this context, we have $| \hat{\phi}_{k, y, {z}} \!-\! \phi_{k, y, {z}} | \!\to\! 0$ according to Equations~(\ref{eq:attribution_trasfer}) and (\ref{eq:estimate_generic_attribution}).
This indicates $\hat{\phi}_{k, y, {z}}$ exactly converges to $\phi_{k, y, {z}}$ in the ideal scenario.

\paragraph{Practical Case.}
Without loss of generality, we consider $\mathcal{L}_{\theta}(\boldsymbol{x}_k)$ is not reduced to zero in this case.
Specifically, Equation~(\ref{eq:pretrain_loss_function}) indicates the reduction of $\mathcal{L}_{\theta}(\boldsymbol{x}_k)$ leads to {\small $\hat{\textbf{\textsl{g}}}_{k,z}$} and {\small $\hat{\textbf{h}}_{k,z}$} gradually approach {\small $\textbf{\textsl{g}}_{k,z}$} and {\small $\textbf{h}_{k,z}$}, respectively.
As a result, the values of {\tiny $\frac{{H}_t(\hat{\textbf{\textsl{g}}}_{k,z}; y)}{{H}_t(\textbf{\textsl{g}}_{k,z}; y)}$} and {\tiny $\frac{{H}_t(\textbf{h}_{k,z}; y)}{{H}_t(\hat{\textbf{h}}_{k,z}; y)}$} gradually converge to~a narrower range around 1.
We formulate this trend by~assuming their values to be bounded within a range of {\tiny $1\!-\!\epsilon \!\leq\! \frac{{H}_t(\hat{\textbf{\textsl{g}}}_{k,z}; y)}{{H}_t(\textbf{\textsl{g}}_{k,z}; y)}, \! \frac{{H}_t(\textbf{h}_{k,z}; y)}{{H}_t(\hat{\textbf{h}}_{k,z}; y)} \!\leq\! 1\!+\!\epsilon$}, where $0 \leq \epsilon \ll 1$.
Under these assumptions, we establish the upper bound of $|\hat{\phi}_{k, y, {z}} \!-\! \phi_{k, y, {z}}|$ in Theorem~\ref{theorem:head-attribution}, with a detailed proof in Appendix~\ref{appendix:theorem-proof}. 
This allows us to understand the behavior of estimation error in practical cases where $\mathcal{L}_{\theta}(\boldsymbol{x}_k)$ is not reduced to zero.

\begin{theorem}[Explanation Error Bound]
\label{theorem:head-attribution}
Given the classifier ${H}_t(\bullet; \bullet)$ of the downstream task, if the output of classifier ${H}_t(\hat{\textbf{\textsl{g}}}_{k,z}; y)$ and ${H}_t(\hat{\textbf{h}}_{k,z}; y)$ fall within the range of $1-\epsilon \leq \frac{{H}_t(\hat{\textbf{\textsl{g}}}_{k,z}; y)}{{H}_t(\textbf{\textsl{g}}_{k,z}; y)}, \frac{{H}_t(\textbf{h}_{k,z}; y)}{{H}_t(\hat{\textbf{h}}_{k,z}; y)} \leq 1+\epsilon$, then, the upper bound of explanation error is given by 
{\setlength\abovedisplayskip{2mm}
\setlength\belowdisplayskip{2mm}\begin{equation}
\label{eq:estimation_error_bound}
\mathbb{E}_{\boldsymbol{x}_k \sim \mathcal{D}_{t}, y \sim \mathcal{Y}_t, {z} \sim \mathcal{Z}(\boldsymbol{x}_k)} | \hat{\phi}_{k, y, {z}} - \phi_{k, y, {z}} | \leq \frac{2 \epsilon}{1-\epsilon},
\end{equation}}
\!\! where $\hat{\phi}_{k, y, {z}}$ and $\phi_{k, y, {z}}$ are given by Equation~(\ref{eq:estimate_generic_attribution}) and~(\ref{eq:attribution_trasfer}), respectively; and $\mathcal{D}_{t}$ denotes the downstream dataset.
\end{theorem}

\paragraph{Intuition of Theorem~\ref{theorem:head-attribution}.}
The value of $\epsilon$ reduces as~the~pre-training loss function  $\mathcal{L}_{\theta}(\boldsymbol{x}_k)$ decreases.
This reduction in $\epsilon$ explicitly lowers the estimation error bound $\frac{2 \epsilon}{1-\epsilon}$ aligned with the $\mathcal{V}$-Information-aligned explanation $\phi_{k, y, {z}}$ on downstream tasks. 
This underscores the \Algnameabbr{} pre-training can significantly enhance the explanations for downstream tasks.

\section{Experiment Results}

In this section, we conduct experiments to evaluate \Algnameabbr{} by answering the following research questions:
\textbf{RQ1:} How does \Algnameabbr{} perform compared with state-of-the-art baseline methods in terms of the fidelity?
\textbf{RQ2:} How does \Algnameabbr{} perform in explaining fully fine-tuned target model on down-stream datasets?
\textbf{RQ3:} How is the transferability of \Algnameabbr{} across different downstream datasets?
\textbf{RQ4:} Do both pre-training and attribution transfer in \Algnameabbr{} contribute to explaining downstream tasks?

% terms $\textbf{\textsl{g}}_k$ and $\textbf{h}_k$ in Definition~\ref{definition:transferable_attribution} both contribute to explaining down-stream tasks?

 % compared with existing amortized explanation methods

\begin{figure*}
\centering
\subfigcapskip=-0mm
\subfigure[\small Imagenette]{
\includegraphics[width=0.3\linewidth]{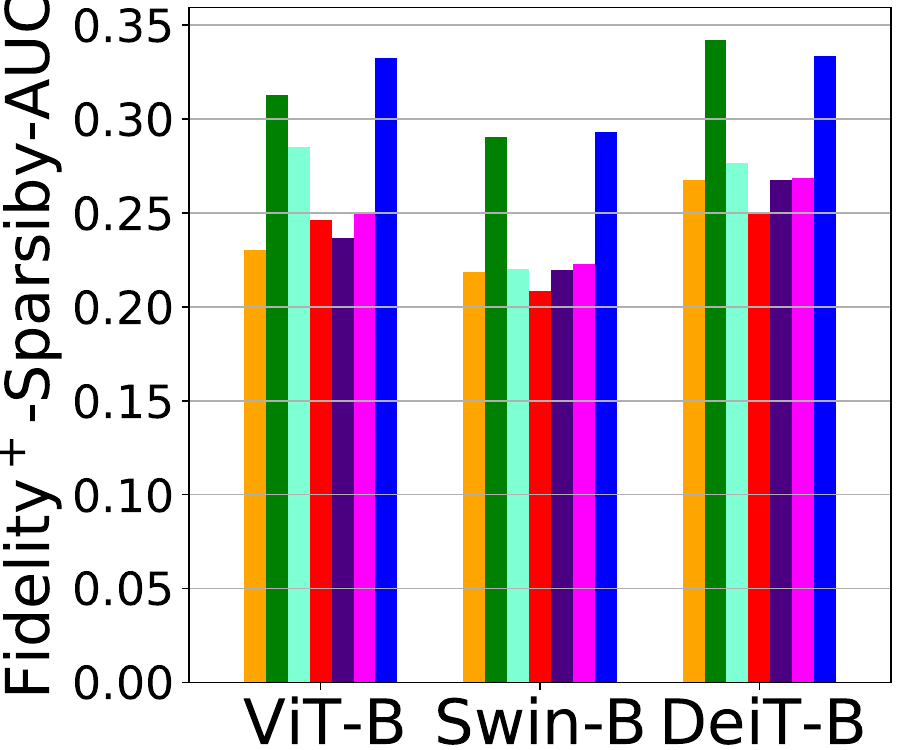}
}
\subfigure[\small Cats-vs-dogs]{
\includegraphics[width=0.3\linewidth]{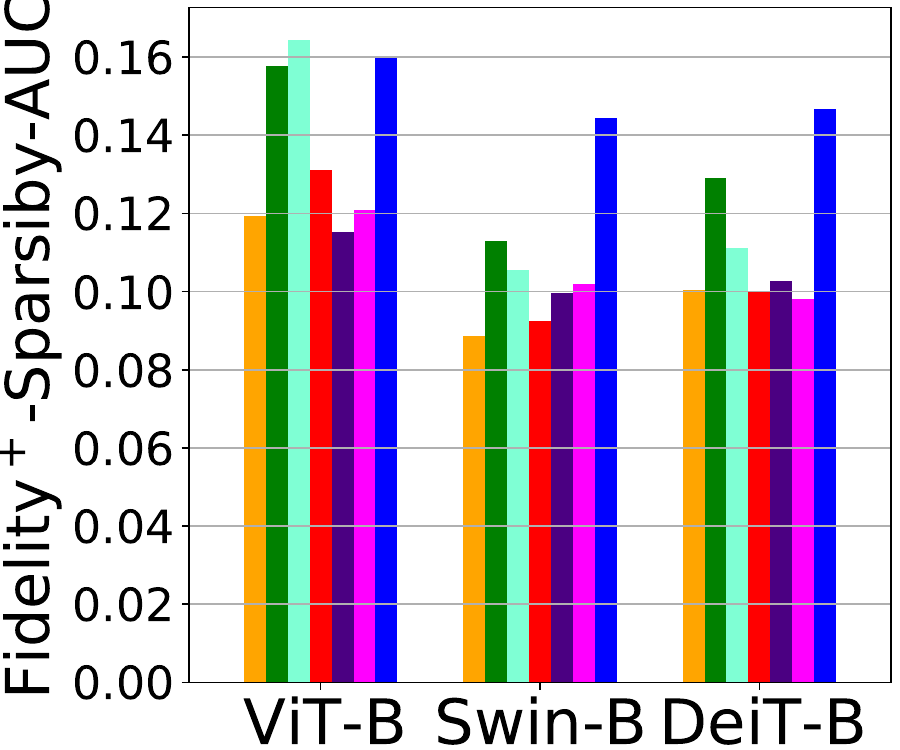}
}
\subfigure[\small CIFAR-10]{
\includegraphics[width=0.3\linewidth]{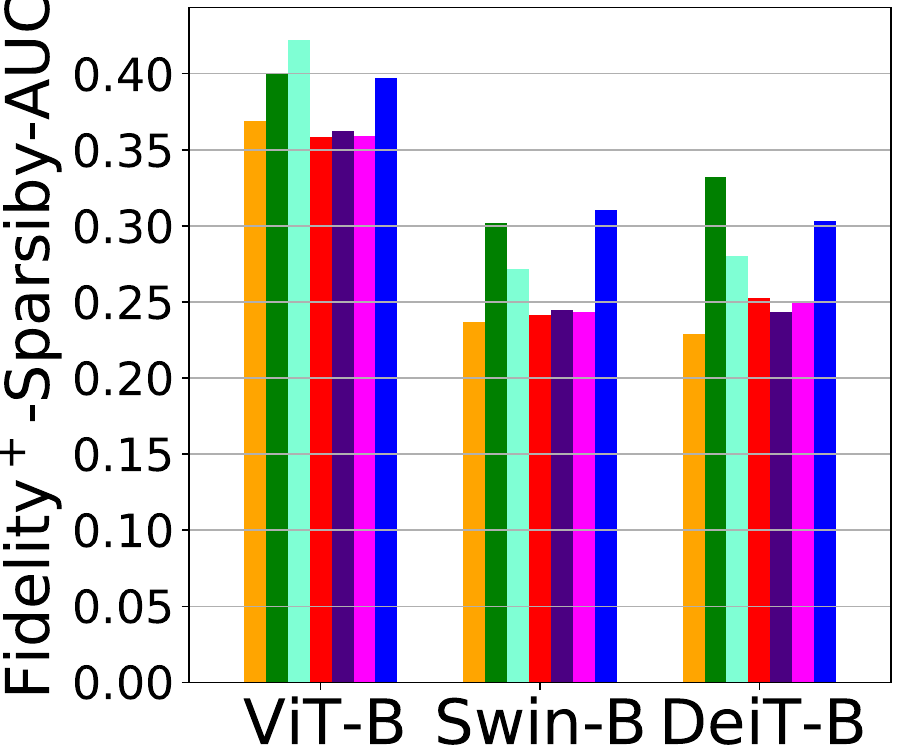}
}
\subfigure[\small Imagenette]{
\includegraphics[width=0.3\linewidth]{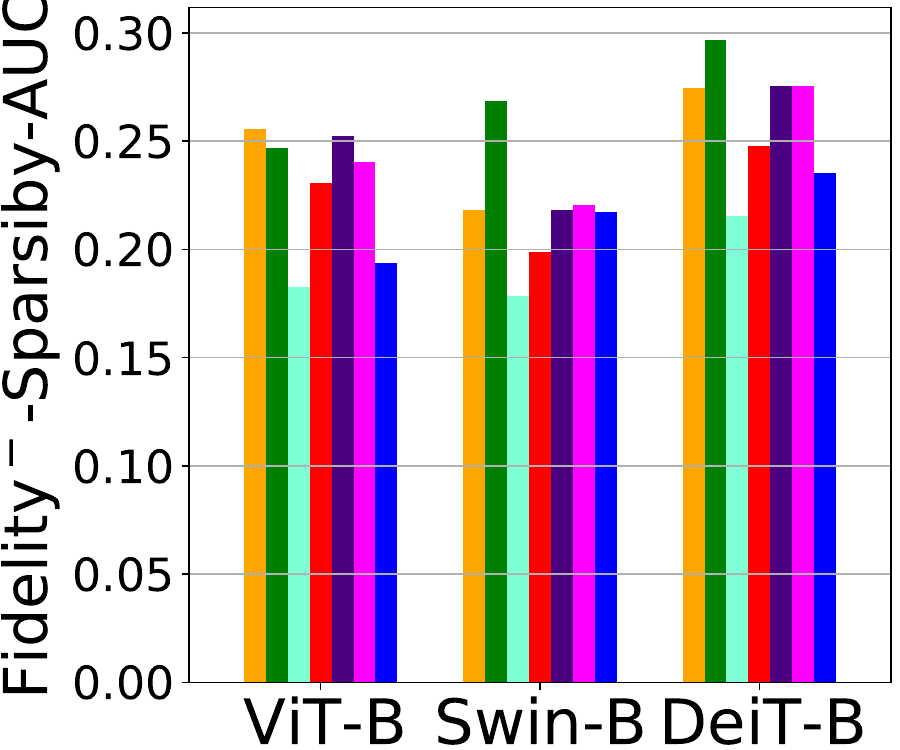}
}
\subfigure[\small Cats-vs-dogs]{
\includegraphics[width=0.3\linewidth]{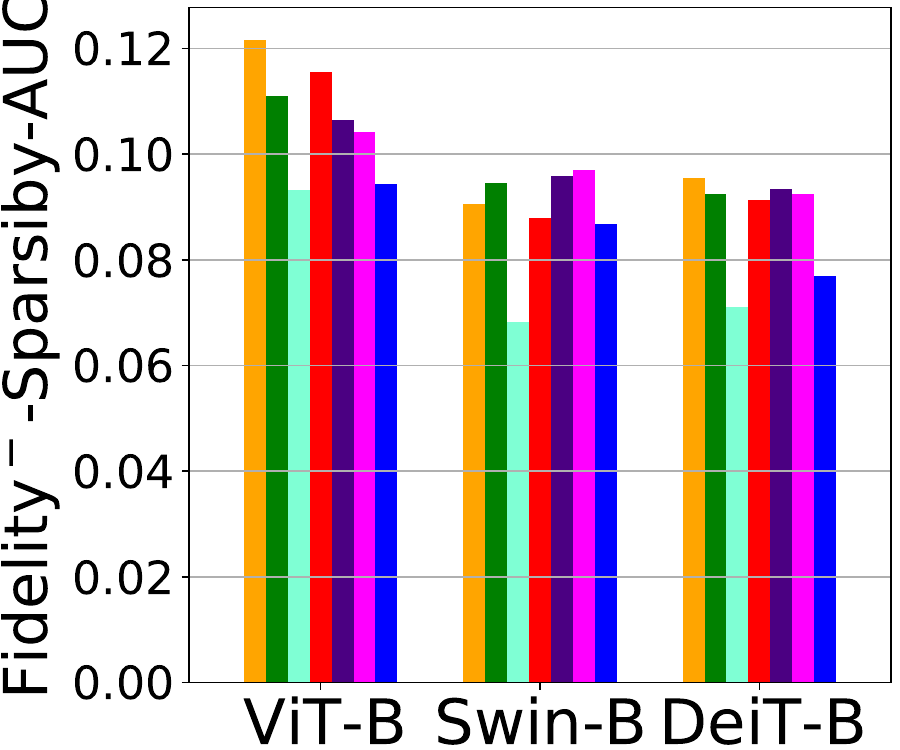}
}
\subfigure[\small CIFAR-10]{
\includegraphics[width=0.3\linewidth]{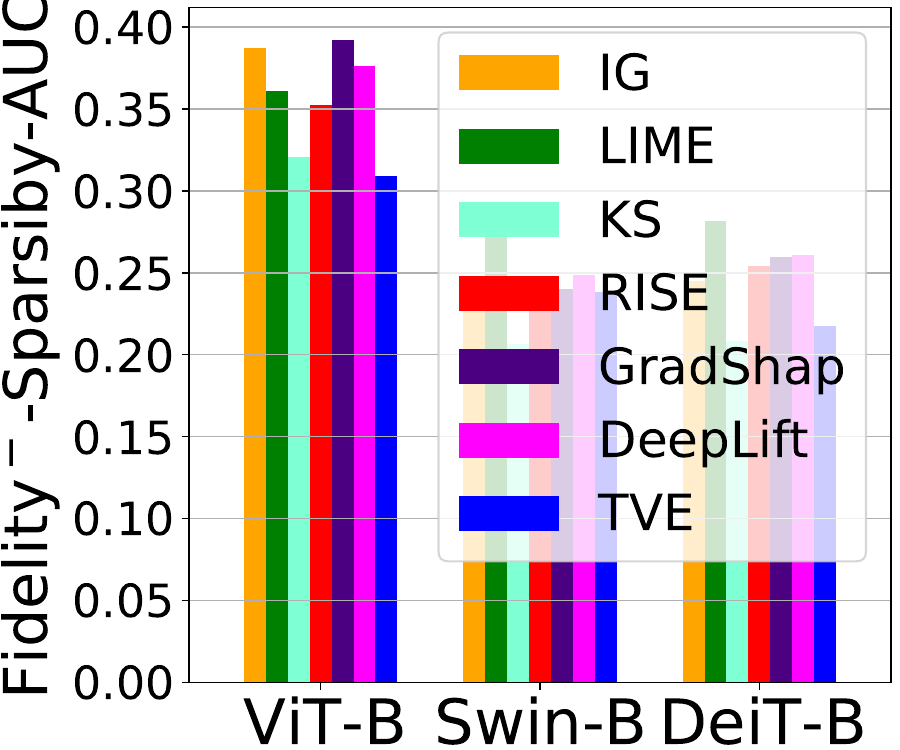}
}

\caption{\label{fig:fidelity_result} \small $\mathrm{Fidelity}^+$-Sparsity-AUC($\uparrow$) on the \texttt{Imagenette}~(a), \texttt{Cat-vs-dogs}~(b), and \texttt{CIFAR-10}~(c) datasets. 
$\mathrm{Fidelity}^-$-Sparsity-AUC($\downarrow$) on the \texttt{Imagenette}~(d), \texttt{Cat-vs-dogs}~(e), and \texttt{CIFAR-10}~(f) datasets.}

\end{figure*}

\begin{figure*}
\centering
\subfigcapskip=-1mm
\subfigure[]{
\includegraphics[width=0.3\linewidth]{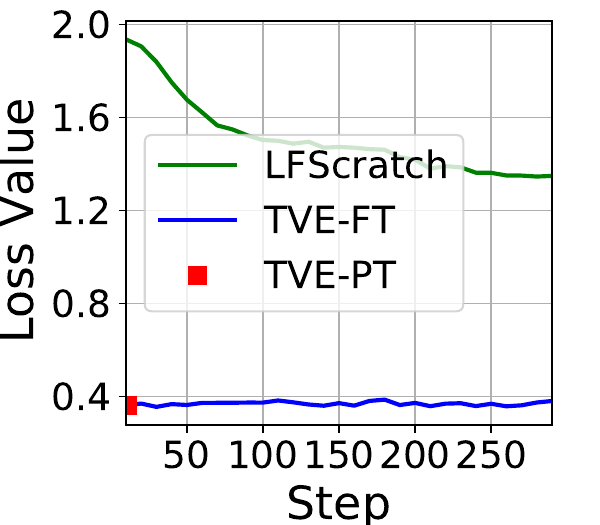}
}
\subfigure[]{
\includegraphics[width=0.28\linewidth]{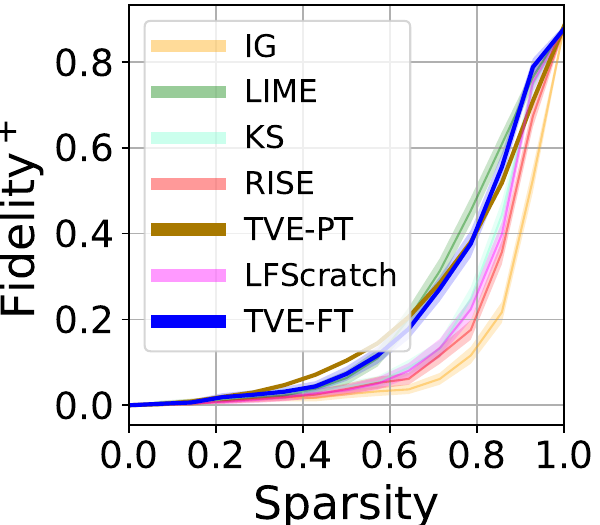}
}
\subfigure[]{
\includegraphics[width=0.28\linewidth]{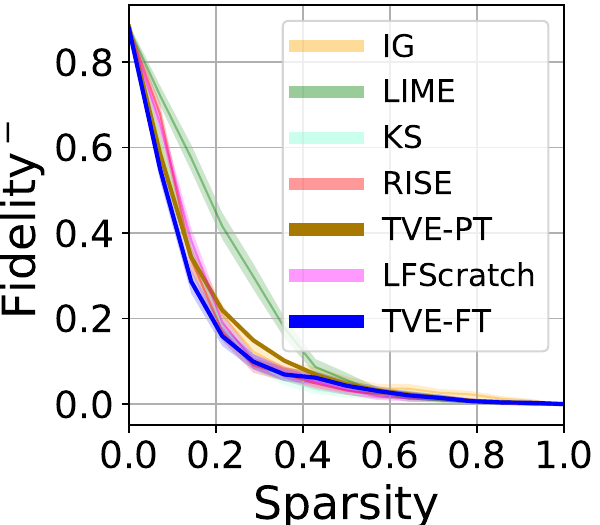}
}
\subfigure[]{
\includegraphics[width=0.3\linewidth]{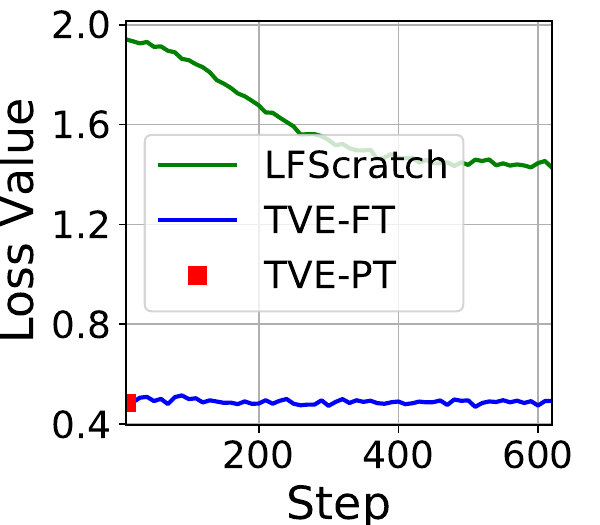}
}
\subfigure[]{
\includegraphics[width=0.28\linewidth]{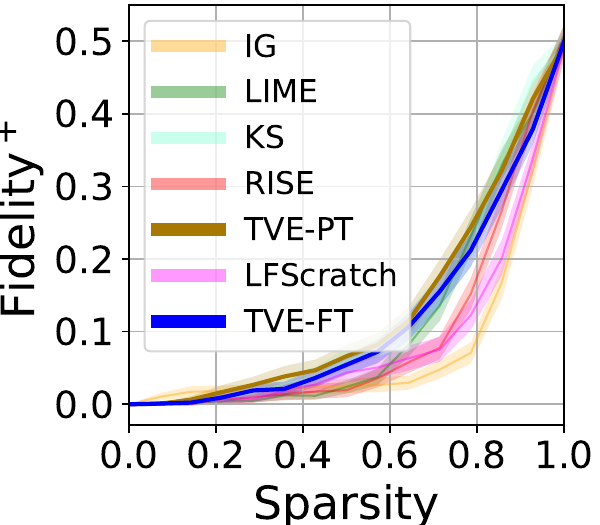}
}
\subfigure[]{
\includegraphics[width=0.28\linewidth]{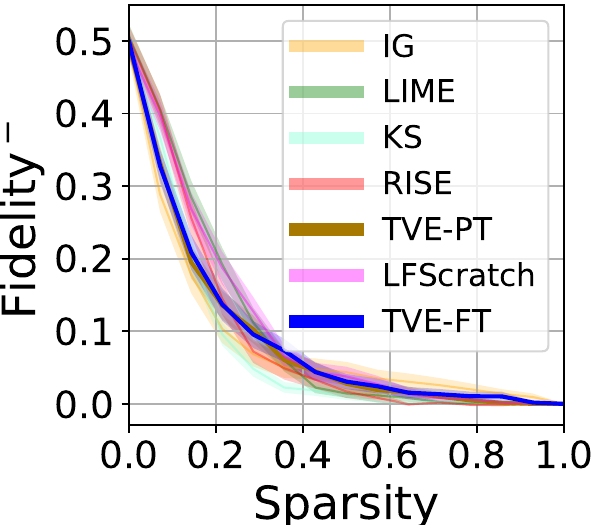}
}

\caption{\label{fig:fidelity_result_full_finetuing} \small \!\!\! Fine-tuning loss versus epoch~(a), {\footnotesize $\mathrm{Fidelity}^+ \! \uparrow$} versus Sparsity~(b), and {\footnotesize $\mathrm{Fidelity}^- \! \downarrow$} versus Sparsity~(c) on the Imagenette dataset. 
Fine-tuning loss versus epoch~(d), {\footnotesize $\mathrm{Fidelity}^+ \uparrow$} versus Sparsity~(e), and {\footnotesize $\mathrm{Fidelity}^- \downarrow$} versus Sparsity~(f) on the cats-vs-dogs dataset.}

\end{figure*}

\subsection{Experiment Setup}

We clarify the datasets, target models, hyper-parameter settings in this section.
More details about the baseline methods, evaluation metrics and implementation details are given in Appendixes~\ref{appendix:baseline}, ~\ref{appendix:eval_metric}, and \ref{appendix:implement-details}, respectively.

\paragraph{Datasets.} We consider the large-scale \texttt{ImageNet} dataset for \Algnameabbr{} pre-training; and the \texttt{Cats-vs-dogs}~\cite{asirra-a-captcha}, \texttt{CIFAR-10}~\cite{krizhevsky2009learning}, and \texttt{Imagenette}~\cite{Howard_Imagenette_2019} datasets for the downstream explaining tasks.
Further details about the datasets are given in Appendix~\ref{appendix:datasets}.
% \cite{deng2009imagenet} 

\paragraph{Target Models.} 
We comprehensively consider three architectures of vision transformers for downstream classification tasks, including the \texttt{ViT-Base}~\cite{dosovitskiy2020image}, \texttt{Swin-Base}~\cite{liu2021swin}, \texttt{Deit-Base}~\cite{touvron2021training} transformers.
We consider two settings of fine-tuning target models: \emph{classifier-tuning} and \emph{full-fine-tuning}.
More details about the target model are given in Appendix~\ref{appendix:target_model}.

\paragraph{Hyper-parameter Settings.}
The experiment follows the pipeline of \Algnameabbr{} pre-training, explanation generation and evaluation on multiple downstream datasets.
Specifically, \Algnameabbr{} adopts the Mask-AutoEncoder~\cite{he2022masked} as the backbone, followed by multiple Feed-Forward~(FFN) layers\footnote{A Mask-AutoEncoder consists of a ViT encoder followed by a ViT decoder; and an FFN layer consists of Linear layers, Layer-norm, and activation function, which are widely used in the Transformer structure. More details about the architecture are given in Appendix~\ref{appendix:implement-details}.} to generate the meta-attribution.
More details about the explainer architecture and hyper-parameters of pre-training \Algnameabbr{} are given in Appendix~\ref{appendix:implement-details}.
When deploying \Algnameabbr{} to explaining downstream tasks, the explanation aligns with the prediction class given by the target model.

\subsection{Evaluation of Fidelity~(RQ1)}
\label{sec:fidelity_eval}

% In this section, we evaluate \Algnameabbr{} in terms of the fidelity-sparsity curve.
% It requires 18 figures to show the $\mathrm{Fidelity}^+$-sparsity curve as well as the $\mathrm{Fidelity}^-$-sparsity curve Swin-Base, DeiT-Base, and ViT-Base on the Cats-vs-dogs, Imagenette, and CIFAR-10 datasets, which are given in Appendix~??? due to the limitation of main content.
% To save the space, we simplify the evaluation of fidelity-sparsity curves into its Area Under the Curve~(AUC), due to its consistency with overall fidelity values.
% Specifically, higher $\mathrm{Fidelity}^+$-sparsity-AUC($\uparrow$) indicates higher $\mathrm{Fidelity}^+$ in most cases of the sparsity, as a result of a more faithful explanation.
% Similarly, lower $\mathrm{Fidelity}^-$-sparsity-AUC($\downarrow$) also indicates a more faithful explanation.
% On the Cats-vs-dogs, Imagenette, and CIFAR-10 datasets, the $\mathrm{Fidelity}^+$-sparsity-AUC($\uparrow$) of the explanation are given in Figures~\ref{fig:fidelity_result}~(a)-(c), respectively; as well as $\mathrm{Fidelity}^-$-sparsity-AUC($\downarrow$) given in Figures~\ref{fig:fidelity_result}~(e)-(g), respectively.

In this section, we evaluate the fidelity of \Algnameabbr{} under the classifier-tuning setting. 
Due to the space constraints, we present 18 figures illustrating the $\mathrm{Fidelity}^+$-sparsity curve($\uparrow$) and the $\mathrm{Fidelity}^-$-sparsity curve($\downarrow$) for explaining the \texttt{ViT-Base}, \texttt{Swin-Base}, and \texttt{Deit-Base} models on the \texttt{Cats-vs-dogs}, \texttt{Imagenette}, and \texttt{CIFAR-10} datasets in Appendix~\ref{appendix:fidelity-sparsity-curve}.
To streamline our evaluation, we simplify the assessment of fidelity-sparsity curves by calculating its Area Under the Curve~(AUC) over the sparsity from zero to one, which aligns with the average fidelity value. Intuitively, a higher $\mathrm{Fidelity}^+$-sparsity-AUC($\uparrow$) indicates superior $\mathrm{Fidelity}^+$($\uparrow$) across most sparsity levels, reflecting a more faithful explanation. Similarly, a lower $\mathrm{Fidelity}^-$-sparsity-AUC($\downarrow$) signifies a more faithful explanation.
More details about the fidelity-sparsity-AUC are given in Appendix~\ref{appendix:auc-eval-metric}.
On the \texttt{Cats-vs-dogs}, \texttt{Imagenette}, and \texttt{CIFAR-10} datasets, we present the $\mathrm{Fidelity}^+$-sparsity-AUC($\uparrow$) for explanations in Figures~\ref{fig:fidelity_result}~(a)-(c), respectively, as well as the $\mathrm{Fidelity}^-$-sparsity-AUC($\downarrow$) in Figures~\ref{fig:fidelity_result}~(d)-(f), respectively.
We have the following observations:
\begin{itemize}[leftmargin=4mm, topsep=10pt]

    \item \emph{\Algnameabbr{} consistently exhibits promising performance in terms of both $\mathrm{Fidelity}^+$($\uparrow$) and $\mathrm{Fidelity}^-$($\downarrow$)}, outperforming the majority of baseline methods. 
    This underscores \Algnameabbr{} faithfully explains various downstream tasks within the scope of pre-training data distribution.
    
    % Among the baseline methods, KernelSHAP shows a satisfactory performance, which is consistent with the experimental observation in existing benchmark work of explainable ML~\cite{liu2021synthetic}.

    \item \emph{\Algnameabbr{} exhibits significant strengths in both $\mathrm{Fidelity}^+$($\uparrow$) and $\mathrm{Fidelity}^-$($\downarrow$)}, highlighting its effectiveness in identifying both important and non-important features.
    In contrast, the baseline methods fail to simultaneously achieve high $\mathrm{Fidelity}^+$ and low $\mathrm{Fidelity}^-$. 
    For example, consider \texttt{LIME}'s performance when explaining the \texttt{Deit-Base} model on the \texttt{CIFAR-10} dataset. 
    While \texttt{LIME} excels in $\mathrm{Fidelity}^+$, it falls short in $\mathrm{Fidelity}^-$.
    % suggesting its proficiency in identifying important features but a limitation in distinguishing non-important ones.

    % \item \textbf{Cross-model Comparison:} For different architectures of target models applied to common downstream task, although the fidelity scores of \Algnameabbr{} are not exactly same, 
    
\end{itemize}

\begin{table*}[t!]
% 
    % \footnotesize
    \caption{Explanation $\mathrm{Fidelity}^+$-Sparsity-AUC($\uparrow$) and $\mathrm{Fidelity}^-$-Sparsity-AUC($\downarrow$) for \texttt{Deit-Base}, \texttt{Swin-Base}, and \texttt{Deit-Base} target models on the \texttt{Cat-vs-dogs}, \texttt{Imagenette}, and \texttt{CIFAR-10} datasets.}
    \centering
    \resizebox{\textwidth}{!}{
    \renewcommand{\arraystretch}{0.1}
    \begin{tabular}{c|c|c|c|c|c|c|c}
    \toprule[0.5pt]
         & Datasets & \multicolumn{2}{c|}{Cats-vs-dogs} & \multicolumn{2}{c|}{Imagenette} & \multicolumn{2}{c}{CIFAR-10} \\
    \midrule[0.3pt]
    Target model & Method & $\mathrm{Fidelity}^+$($\uparrow$) & $\mathrm{Fidelity}^-$($\downarrow$) & $\mathrm{Fidelity}^+$($\uparrow$) & $\mathrm{Fidelity}^-$($\downarrow$) & $\mathrm{Fidelity}^+$($\uparrow$) & $\mathrm{Fidelity}^-$($\downarrow$) \\
    \midrule[0.3pt]
    \multirow{3}{*}{ViT-Base} & \texttt{ViTShapley} & 0.11$\pm$\tiny{0.09} & 0.13$\pm$\tiny{0.10} & 0.25$\pm$\tiny{0.13} & 0.25$\pm$\tiny{0.14} & 0.36$\pm$\tiny{0.17} & 0.36$\pm$\tiny{0.17} \\ 
    & \Algnameabbr{}-$H_\textsl{g}$ & 0.14$\pm$\tiny{0.11} & 0.10$\pm$\tiny{0.08} & 0.29$\pm$\tiny{0.14} & \textbf{0.18}$\pm$\tiny{0.10} & 0.39$\pm$\tiny{0.18} & 0.34$\pm$\tiny{0.17} \\
    & \Algnameabbr{} & \textbf{0.16}$\pm$\tiny{0.13} & \textbf{0.09}$\pm$\tiny{0.07} & \textbf{0.33}$\pm$\tiny{0.16} & 0.19$\pm$\tiny{0.12} & \textbf{0.40}$\pm$\tiny{0.18} & \textbf{0.31}$\pm$\tiny{0.16} \\
    \midrule[0.3pt]
    \multirow{3}{*}{Swin-Base} & \texttt{ViTShapley} & 0.09$\pm$\tiny{0.05} & 0.11$\pm$\tiny{0.07} & 0.24$\pm$\tiny{0.07} & 0.24$\pm$\tiny{0.09} & 0.25$\pm$\tiny{0.11} & 0.28$\pm$\tiny{0.14} \\ 
    & \Algnameabbr{}-$H_\textsl{g}$ & \textbf{0.14}$\pm$\tiny{0.09} & 0.10$\pm$\tiny{0.07} & \textbf{0.29}$\pm$\tiny{0.08} & 0.24$\pm$\tiny{0.07} & 0.26$\pm$\tiny{0.12} & 0.27$\pm$\tiny{0.13} \\
    & \Algnameabbr{} & \textbf{0.14}$\pm$\tiny{0.10} & \textbf{0.09}$\pm$\tiny{0.05} & \textbf{0.29}$\pm$\tiny{0.10} & \textbf{0.22}$\pm$\tiny{0.06} & \textbf{0.31}$\pm$\tiny{0.14} & \textbf{0.24}$\pm$\tiny{0.12} \\
    \midrule[0.3pt]
    \multirow{3}{*}{DeiT-Base} & \texttt{ViTShapley} & 0.12$\pm$\tiny{0.08} & 0.1$\pm$\tiny{0.07} & 0.22$\pm$\tiny{0.09} & 0.29$\pm$\tiny{0.11} & 0.28$\pm$\tiny{0.13} & 0.24$\pm$\tiny{0.13} \\ 
    & \Algnameabbr{}-$H_\textsl{g}$ & 0.13$\pm$\tiny{0.08} & 0.09$\pm$\tiny{0.06} & \textbf{0.33}$\pm$\tiny{0.10} & 0.25$\pm$\tiny{0.08} & \textbf{0.32}$\pm$\tiny{0.14} & 0.24$\pm$\tiny{0.13} \\
    & \Algnameabbr{} & \textbf{0.15}$\pm$\tiny{0.10} & \textbf{0.08}$\pm$\tiny{0.06} & \textbf{0.33}$\pm$\tiny{0.10} & \textbf{0.24}$\pm$\tiny{0.08} & 0.30$\pm$\tiny{0.13} & \textbf{0.22}$\pm$\tiny{0.12} \\
    \bottomrule[0.5pt]
    \end{tabular}
    }
    \label{tab:transferibility}
\end{table*}

\begin{figure*}
\centering
\begin{minipage}[t]{0.44\textwidth}
\!\!
\includegraphics[width=0.5\linewidth]{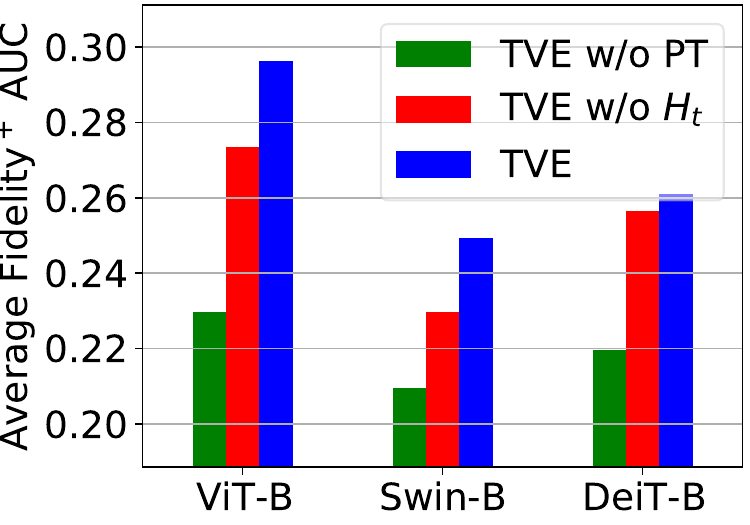}
\!\!
\includegraphics[width=0.5\linewidth]{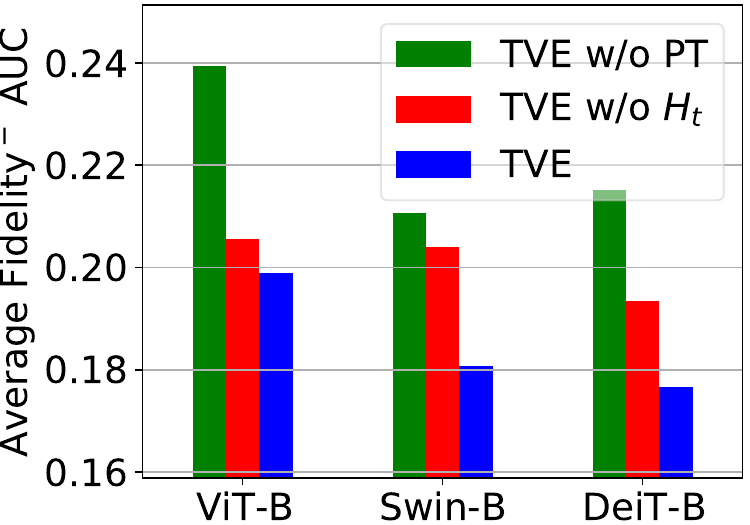}

\caption{\label{fig:ablation_result} Fidelity of ablation studies.}
\end{minipage}
% \quad
\begin{minipage}[t]{0.54\textwidth}
\centering
\includegraphics[width=1.0\linewidth]{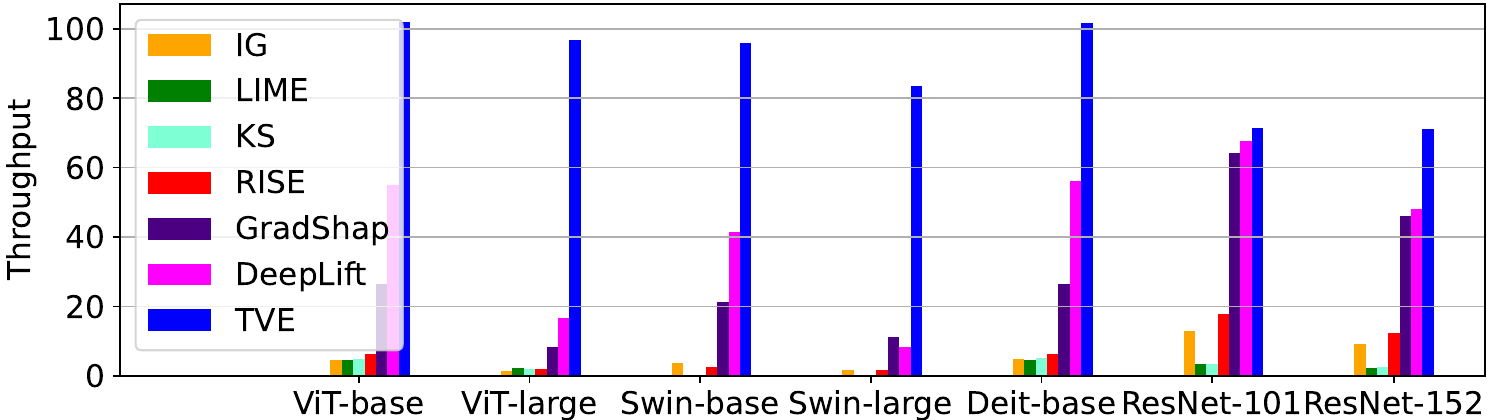}

\caption{\label{fig:throughput_result} Throughput of explaining different architectures.}
\end{minipage}

\end{figure*}

\subsection{Explaining Fully Fine-tuned Models~(RQ2)}
\label{sec:adaptation_finetune_backbone}

In this section, we evaluate the fidelity of \Algnameabbr{} under the full-fine-tuning setting to demonstrate its generalization ability.
Notably, the \texttt{ViT-Base} classification model including both the backbone and classifier are fine-tuned on downstream data, which are not available to \Algnameabbr{} pre-training.
The explanation considers three methods: learning from scratch~(\texttt{LFScratch}), \Algnameabbr{} pre-training~(\Algnameabbr{}-PT), and \Algnameabbr{} fine-tuning~(\Algnameabbr{}-FT).
To adapt to the fully fine-tuned target model, \texttt{LFScratch} trains the explainer on the downstream dataset for one epoch;
\Algnameabbr{}-PT simply transfers the pre-trained explainer to explaining the down-stream tasks; \Algnameabbr{}-FT follows Algorithm~\ref{alg:pretrain} to fine-tune the explainer using the fine-tuned backbone encoder on the downstream dataset for one epoch.
Here, we consider the \texttt{Imagenette} and \texttt{Cat-vs-dogs} datasets for the downstream tasks.
Further details about fine-tuning the target models and explainers are given in Appendixes~\ref{appendix:target_model} and \ref{appendix:implement-details}, respectively.
The loss value of \texttt{LFScratch} and \Algnameabbr{}-FT versus the fine-tuning steps are shown in Figures~\ref{fig:fidelity_result_full_finetuing}~(a) and (d).
The fidelity-sparsity curves of all methods are given in Figures~\ref{fig:fidelity_result_full_finetuing}~(b), (c), (e), and (f).
We have the following observations:

\begin{itemize}[leftmargin=4mm, topsep=10pt]

    \item \emph{\Algnameabbr{} pre-training provides a good initial explainer for adaption to fully fine-tuned encoders.} 
    According to Figures~\ref{fig:fidelity_result_full_finetuing}~(a,d), the \Algnameabbr{} pre-trained explainer shows lower training loss than learning from scratch in the early epochs.
    This indicates the pre-training provides a good initial explainer for explaining downstream tasks.

    \item \emph{\Algnameabbr{}-PT can effectively explain the fully fine-tuned target model, even without fine-tuning the explainer on downstream datasets.}
    According to Figures~\ref{fig:fidelity_result_full_finetuing}~(b,c,e,f), \Algnameabbr{}-PT shows competitive fidelity when comparing with \Algnameabbr{}-FT and other baseline methods, and a significant improvement over \texttt{LFScratch}.
    This indicates the strong generalization ability of \Algnameabbr{}, acquired through pre-training on the large-scale \texttt{ImageNet} dataset.

    \item \emph{The pre-training of transferable explainer and fine-tuning of backbone encoder can be executed independently and parallelly.}
    Specifically, \Algnameabbr{} pre-trains the transferable explainer based on open-sourced pre-trained backbone encoders and large-scale \texttt{ImageNet} dataset; meanwhile, the encoder can be fine-tuned in parallel on downstream datasets. 
    This can significantly improve the efficiency and flexibility of deploying \Algnameabbr{} to practical scenarios.
    
\end{itemize}

% \begin{figure}
% \centering
% \subfigure[Imagenette]{
% \includegraphics[width=0.3\linewidth]{figure/Imagenette_fidelity_pos_bar_vit_shapley.pdf}
% }
% \subfigure[Cats-vs-dogs]{
% \includegraphics[width=0.3\linewidth]{figure/Cats_vs_dogs_fidelity_pos_bar_vit_shapley.pdf}
% }
% \subfigure[CIFAR-10]{
% \includegraphics[width=0.3\linewidth]{figure/cifar10_fidelity_pos_bar_vit_shapley.pdf}
% }
% \subfigure[Imagenette]{
% \includegraphics[width=0.3\linewidth]{figure/Imagenette_fidelity_neg_bar_vit_shapley.pdf}
% }
% \subfigure[Cats-vs-dogs]{
% \includegraphics[width=0.3\linewidth]{figure/Cats_vs_dogs_fidelity_neg_bar_vit_shapley.pdf}
% }
% \subfigure[CIFAR-10]{
% \includegraphics[width=0.3\linewidth]{figure/cifar10_fidelity_neg_bar_vit_shapley.pdf}
% }
% % 
% \caption{\label{fig:fidelity_result} $\mathrm{Fidelity}^+$-Sparsity-AUC($\uparrow$) on the Imagenette~(a), Cat-vs-dogs~(b), and CIFAR-10~(c) datasets. 
% $\mathrm{Fidelity}^-$-Sparsity-AUC($\downarrow$) on the Imagenette~(d), Cat-vs-dogs~(e), and CIFAR-10~(f) datasets.}
% % 
% \end{figure}

\subsection{Evaluation of Transferability~(RQ3)}
\label{sec:eval_transfer}

 % in Appendixes~\ref{appendix:baseline}

We evaluate the transferability of \Algnameabbr{} compared with \texttt{ViT-Shapley}~\cite{covert2022learning}, a state-of-the-art DNN-based explainer for vision models.
Specifically, \texttt{ViT-Shapley} pre-trains the explainer on the large-scale \texttt{ImageNet} dataset, and deploys it to the \texttt{Cat-vs-dogs}, \texttt{Imagenette}, and \texttt{CIFAR-10} datasets to generate the explanations.
Different from \texttt{ViT-Shapley}, \Algnameabbr{} transfers the explainer to downstream datasets via taking the task-specific classifier $H_t$ into Equation~(\ref{eq:estimate_generic_attribution}).
Moreover, we also consider a \Algnameabbr{}-$H_{\textsl{g}}$ method to study whether the pre-training of \Algnameabbr{} contributes to explaining downstream tasks. 
Different from \Algnameabbr{}, \Algnameabbr{}-$H_{\textsl{g}}$ takes a general classifier~(pre-trained on the \texttt{ImageNet} dataset) into Equation~(\ref{eq:estimate_generic_attribution}) to generate the explanation.
We follow Section~\ref{sec:fidelity_eval} to adopt the fidelity-sparsity AUC to evaluate the average fidelity.
Table~\ref{tab:transferibility} illustrates the fidelity for explaining the \texttt{ViT-Base}, \texttt{Swin-Base}, and \texttt{Deit-Base} models on the \texttt{Cat-vs-dogs}, \texttt{Imagenette}, and \texttt{CIFAR-10} datasets.
We have the following insights:
\begin{itemize}[leftmargin=4mm, topsep=10pt]

    \item \emph{\Algnameabbr{} has stronger transferability than \texttt{ViT-Shapley}.}
    Both \Algnameabbr{} and \texttt{ViT-Shapley} are pre-trained on the large-scale \texttt{ImageNet} dataset, and transferred to the downstream datasets without additional training.
    Table~\ref{tab:transferibility} shows \Algnameabbr{} has higher $\mathrm{Fidelity}^+$($\uparrow$) and lower $\mathrm{Fidelity}^-$($\downarrow$) than \texttt{ViT-Shapley}.

    \item \emph{The pre-training of \Algnameabbr{} significantly contributes to explaining downstream tasks.} 
    \Algnameabbr{}-$H_{\textsl{g}}$ adopts the generally pre-trained explainer and classifier to explain downstream tasks, and achieves a reasonable fidelity on most of the datasets.
    This indicates the pre-training of \Algnameabbr{} captures the transferable features across various datasets for explaining downstream tasks.

    \item \emph{It is more faithful to explain downstream tasks based on the task-specific classifiers.} 
    \Algnameabbr{} outperforms \Algnameabbr{}-$H_{\textsl{g}}$ on most architectures and datasets, which indicates the attribution transfer had better take the classifier aligned with the downstream task for $H_t$ in Definition~\ref{definition:attribution_transfer}.

\end{itemize}

% To be concrete, \Algnameabbr{} takes a task-specific classifier $H_t$ into Definition~\ref{definition:attribution_transfer}; 
% Moreover, we note that the \emph{ViT-Shapley method requires to train the amortized explainer separately on each of the datasets}. 
% In contrast, \emph{\Algnameabbr{} simply follows Theorem~\ref{definition:attribution_transfer} to explain the classification on each of the datasets without additional training.}
% This is more applicable and easy-to-deploy to real-world scenarios.

\begin{figure*}
\centering
\subfigure[Cats]{
\begin{minipage}{0.3\textwidth}
\centering
\includegraphics[width=0.24\linewidth]{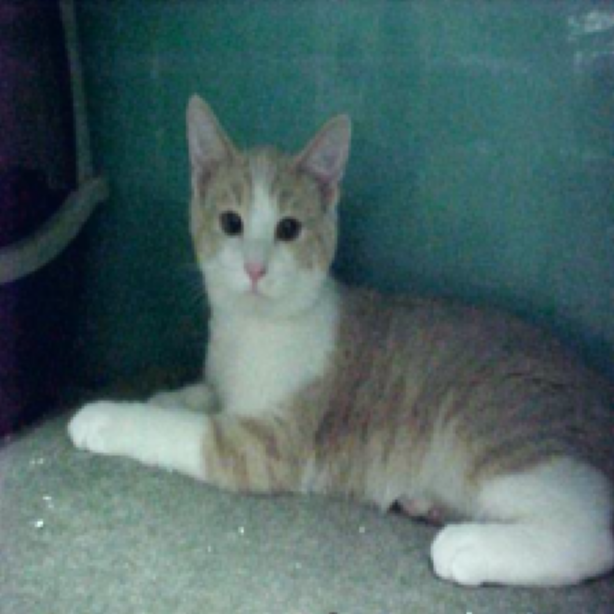}
\!\!\!
\includegraphics[width=0.24\linewidth]{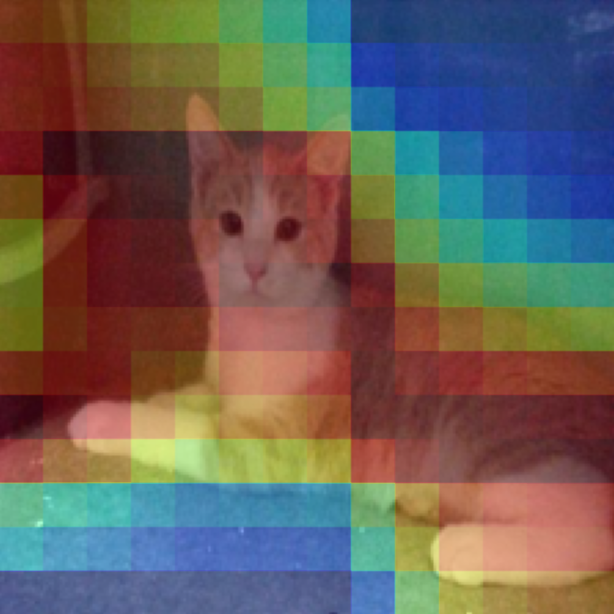}
\!\!\!
\includegraphics[width=0.24\linewidth]{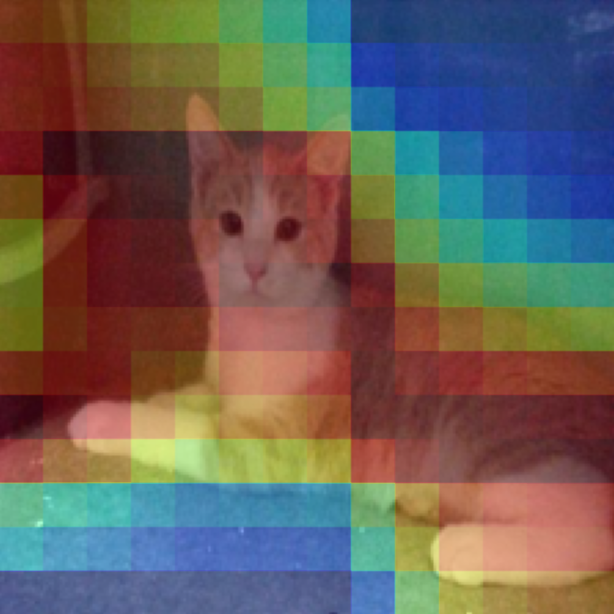}
\!\!\!
\includegraphics[width=0.24\linewidth]{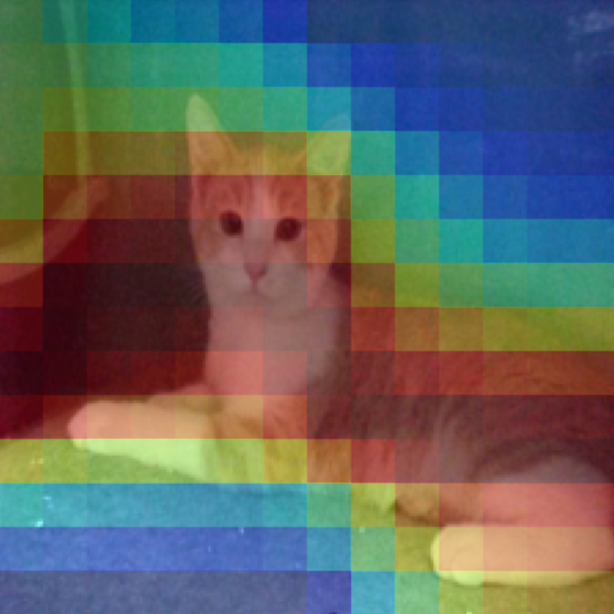}
\end{minipage}
}
% \!\!\!\!\!\!
\subfigure[Dogs]{
\begin{minipage}{0.3\textwidth}
\includegraphics[width=0.24\linewidth]{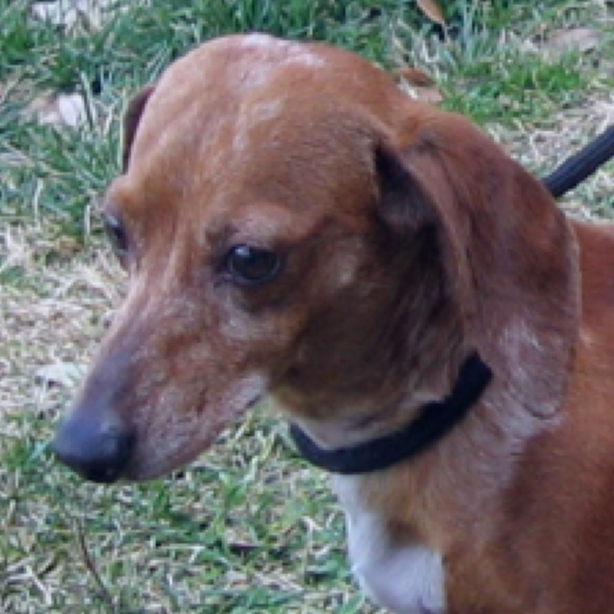}
\!\!\!
\includegraphics[width=0.24\linewidth]{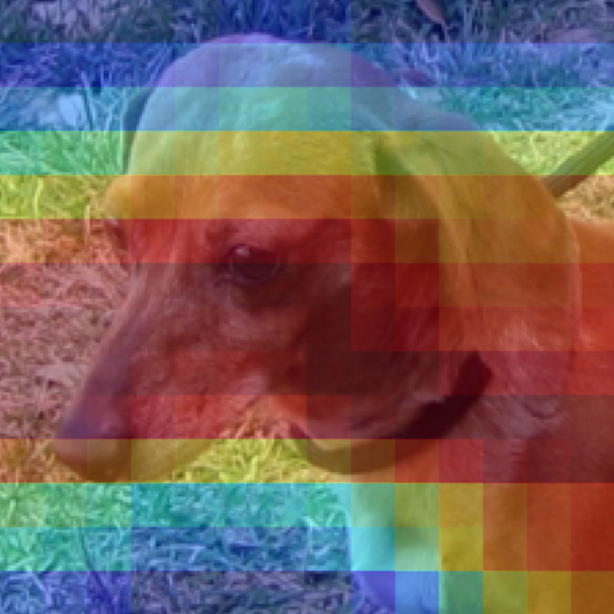}
\!\!\!
\includegraphics[width=0.24\linewidth]{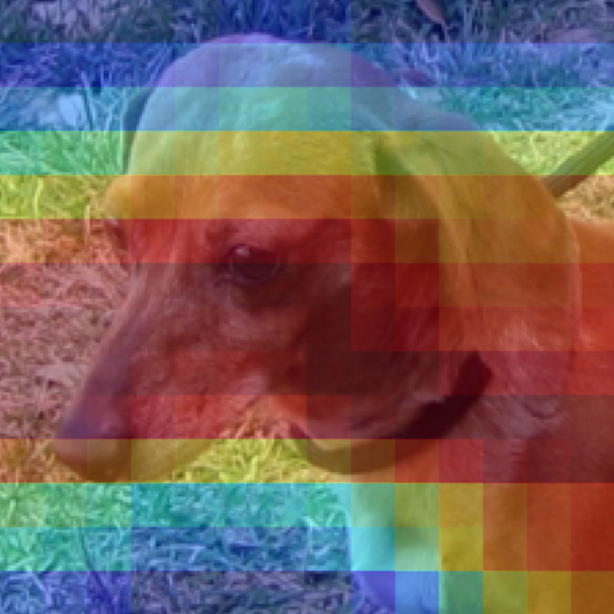}
\!\!\!
\includegraphics[width=0.24\linewidth]{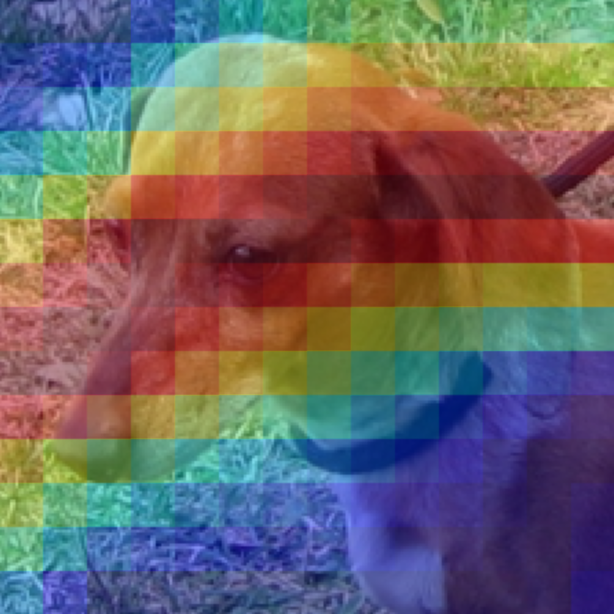}
\end{minipage}
}
% \!\!\!\!\!\!
\subfigure[Cats]{
\begin{minipage}{0.3\textwidth}
\includegraphics[width=0.24\linewidth]{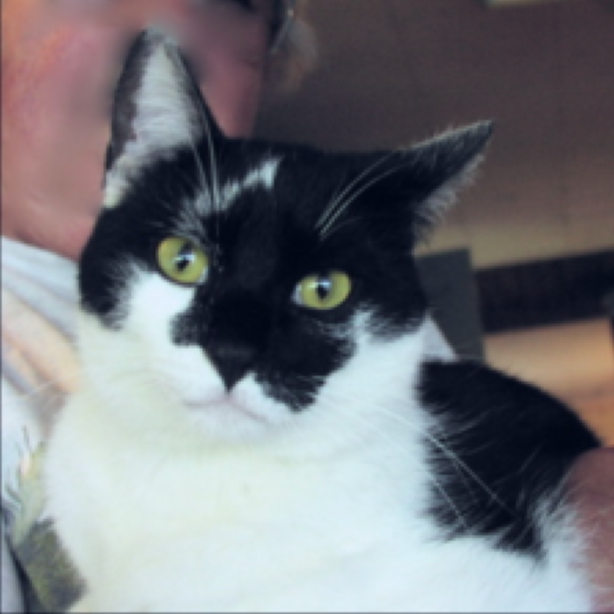}
\!\!\!
\includegraphics[width=0.24\linewidth]{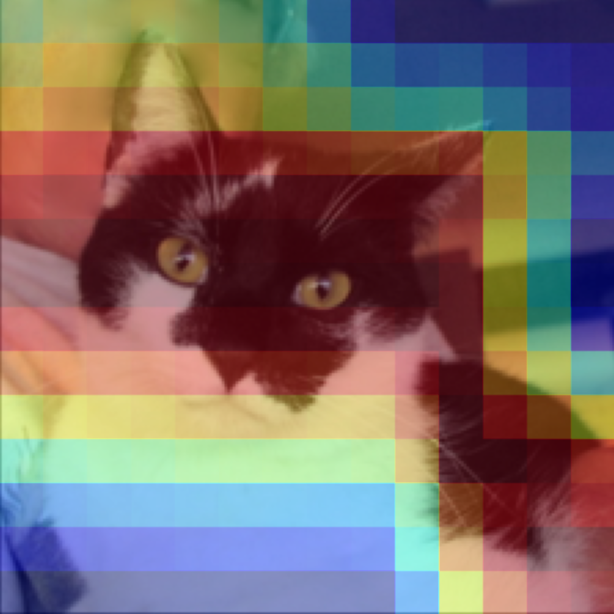}
\!\!\!
\includegraphics[width=0.24\linewidth]{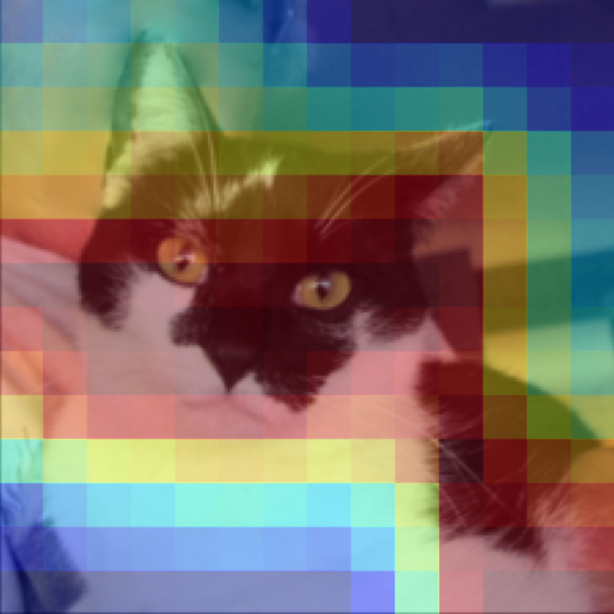}
\!\!\!
\includegraphics[width=0.24\linewidth]{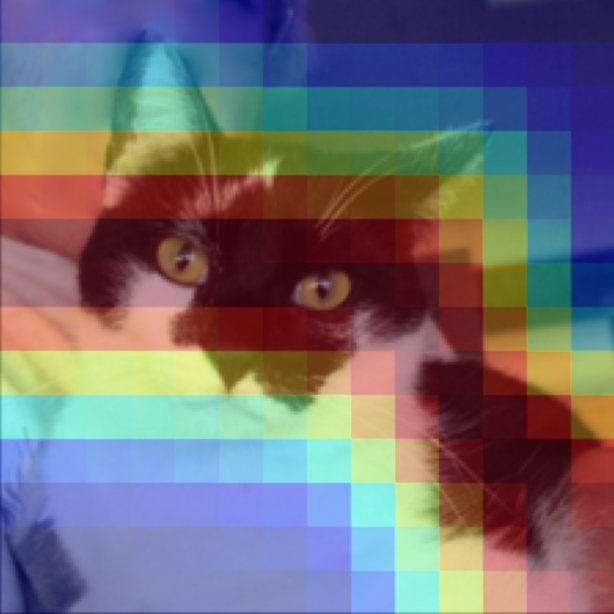}
\end{minipage}
}
\\

\subfigure[Church]{
\begin{minipage}{0.3\textwidth}
\centering
\includegraphics[width=0.24\linewidth]{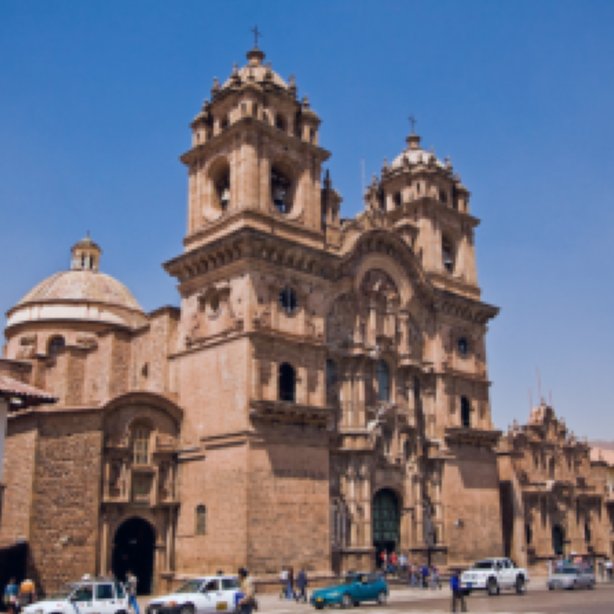}
\!\!\!
\includegraphics[width=0.24\linewidth]{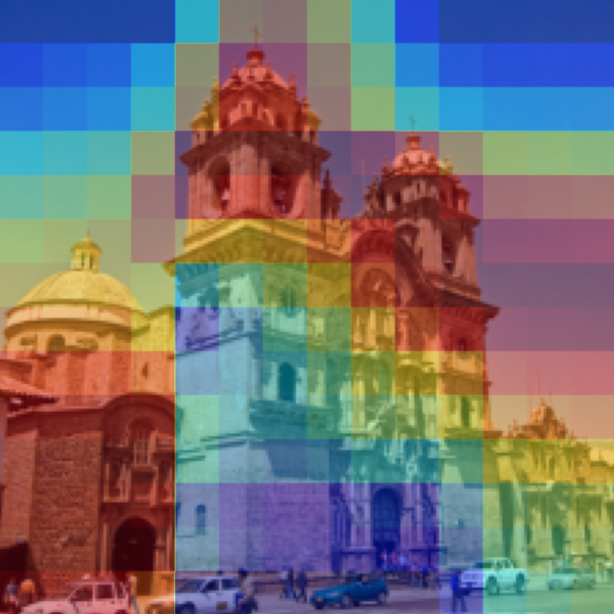}
\!\!\!
\includegraphics[width=0.24\linewidth]{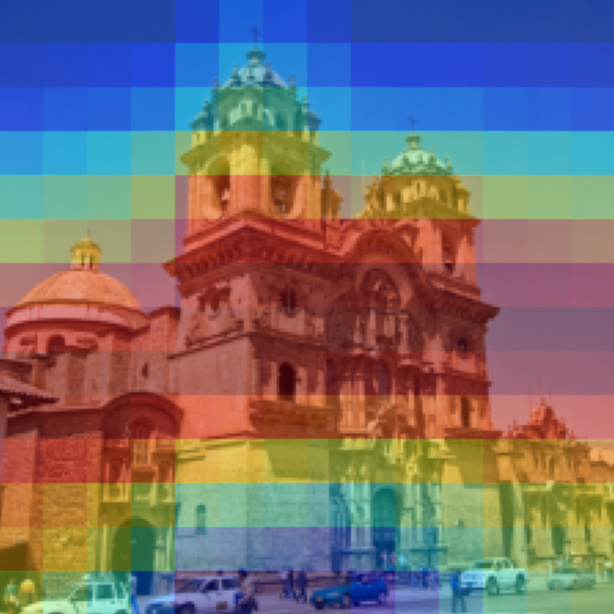}
\!\!\!
\includegraphics[width=0.24\linewidth]{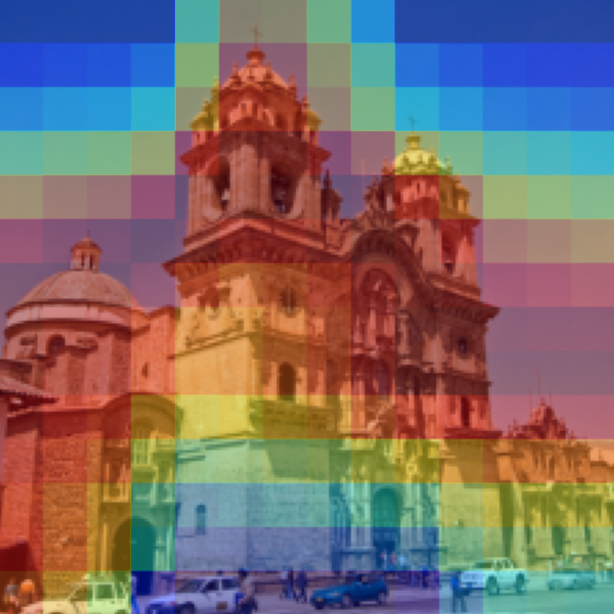}
\end{minipage}
}
% \!\!\!\!\!\!
\subfigure[Parachute]{
\begin{minipage}{0.3\textwidth}
\centering
\includegraphics[width=0.24\linewidth]{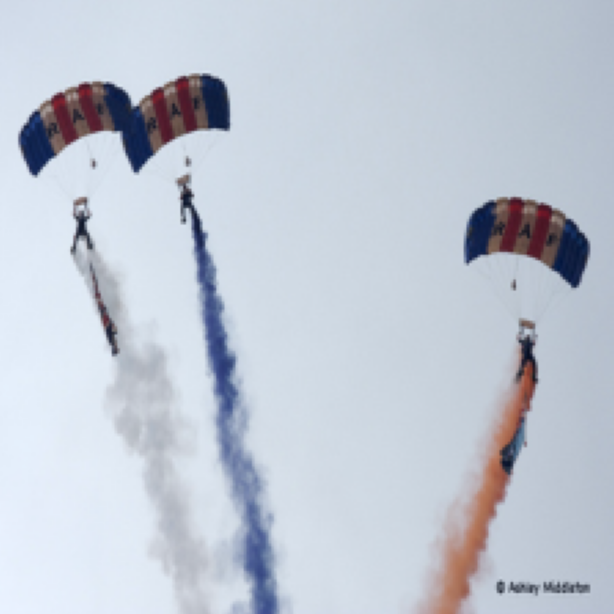}
\!\!\!
\includegraphics[width=0.24\linewidth]{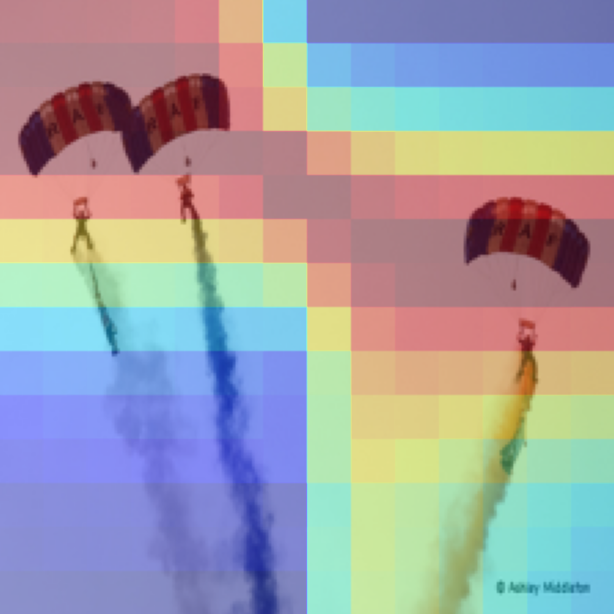}
\!\!\!
\includegraphics[width=0.24\linewidth]{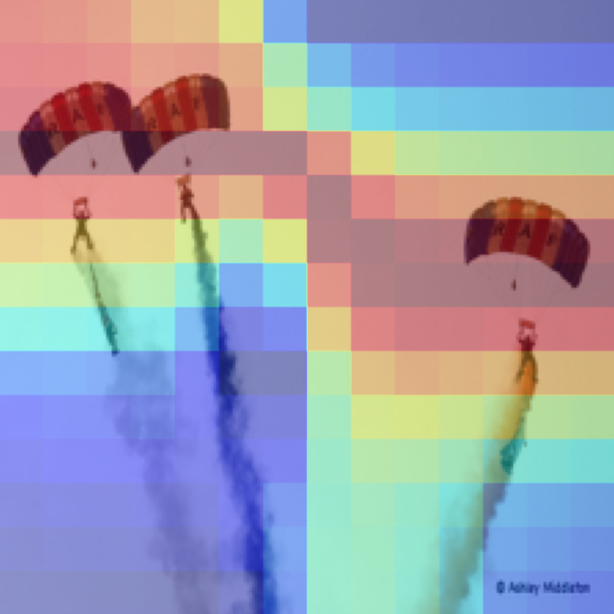}
\!\!\!
\includegraphics[width=0.24\linewidth]{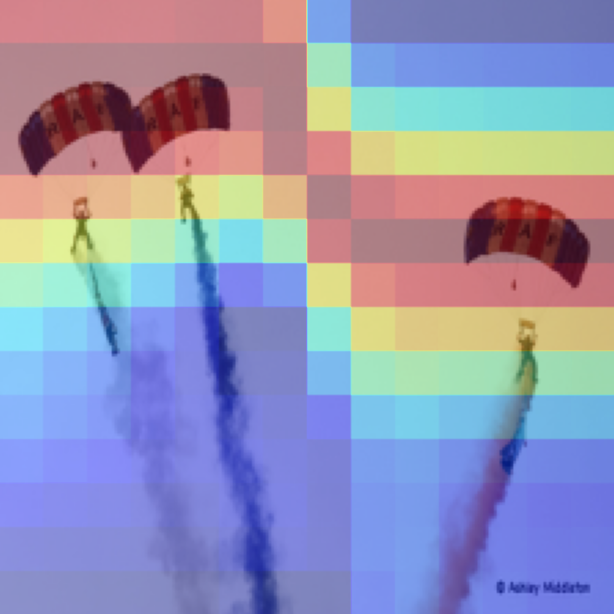}
\end{minipage}
}
% \!\!\!\!\!\!
\subfigure[Garbage truck]{ % Golf ball
\begin{minipage}{0.3\textwidth}
\centering
\includegraphics[width=0.24\linewidth]{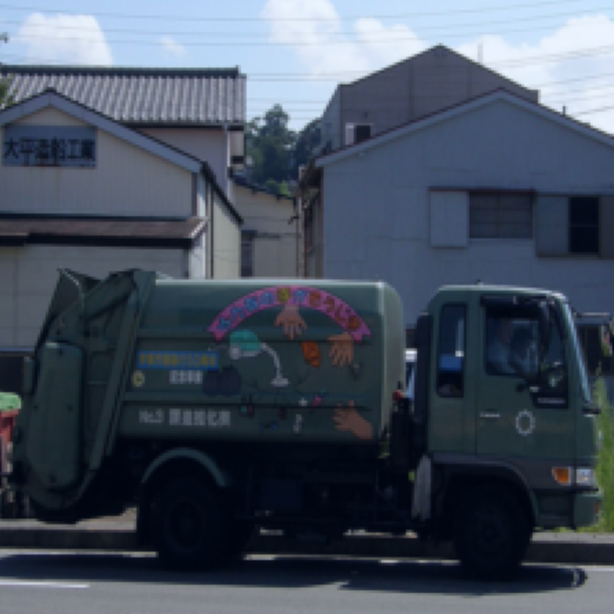}
\!\!\!
\includegraphics[width=0.24\linewidth]{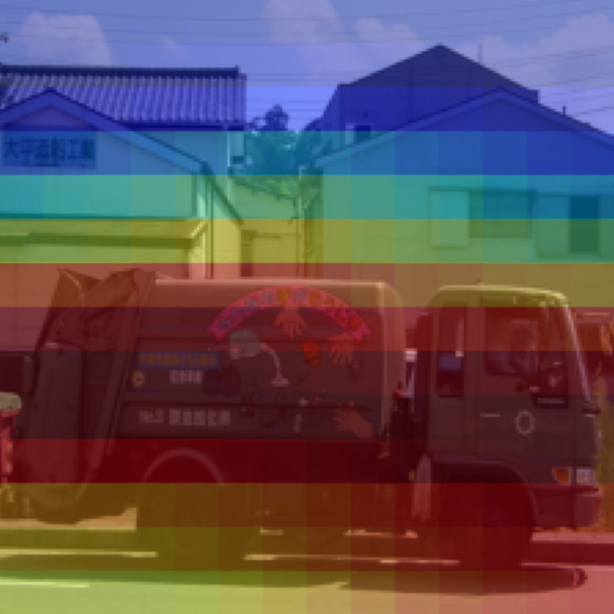}
\!\!\!
\includegraphics[width=0.24\linewidth]{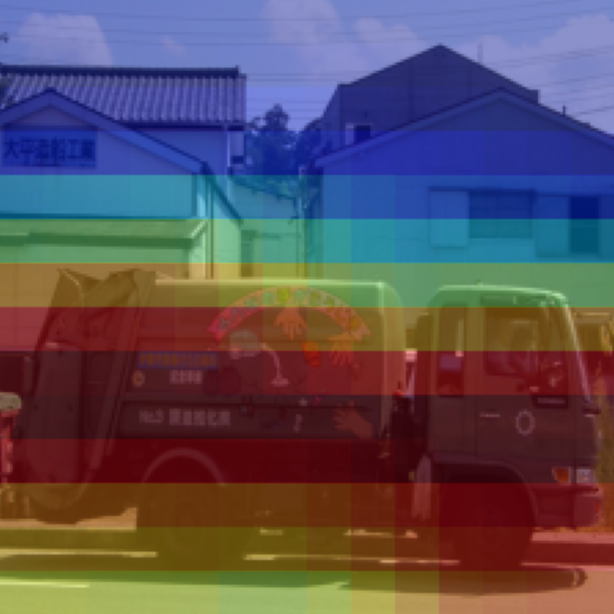}
\!\!\!
\includegraphics[width=0.24\linewidth]{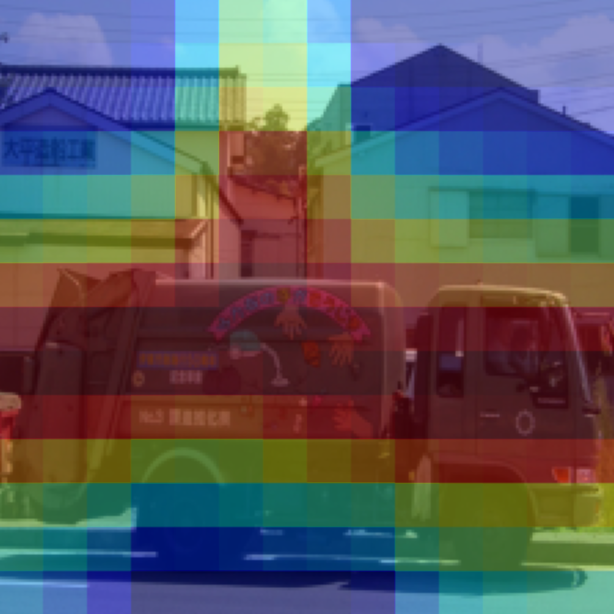}
\end{minipage}
}
\\

\subfigure[Ship]{
\begin{minipage}{0.3\textwidth}
\centering
\includegraphics[width=0.24\linewidth]{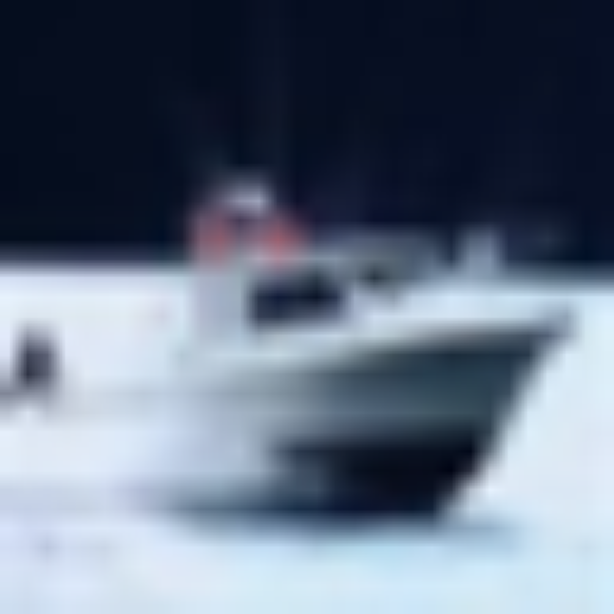}
\!\!\!
\includegraphics[width=0.24\linewidth]{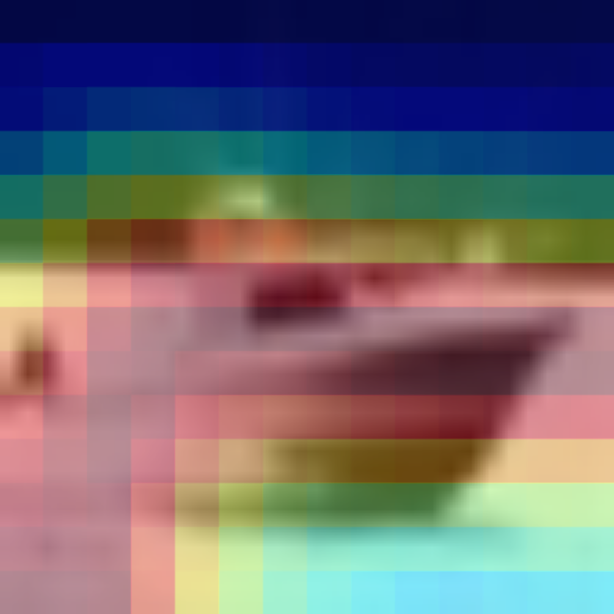}
\!\!\!
\includegraphics[width=0.24\linewidth]{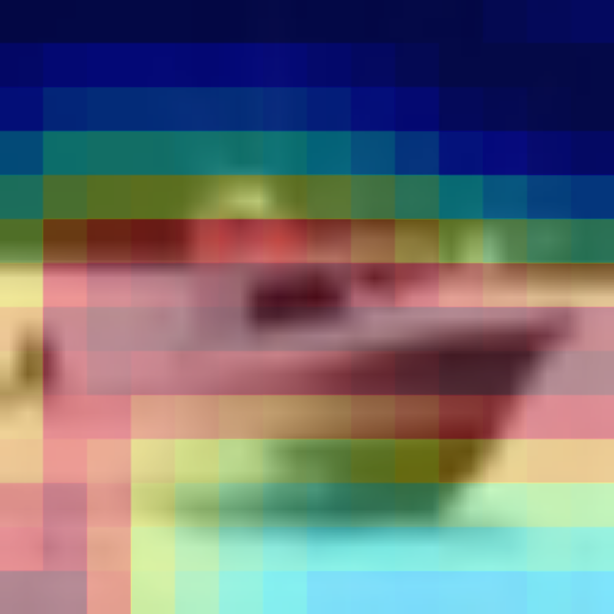}
\!\!\!
\includegraphics[width=0.24\linewidth]{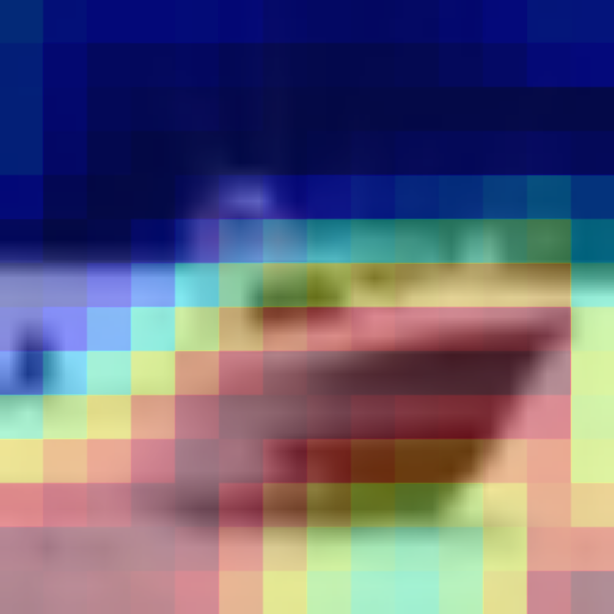}
\end{minipage}
}
% \!\!\!\!\!\!
\subfigure[Airplane]{
\begin{minipage}{0.3\textwidth}
\centering
\includegraphics[width=0.24\linewidth]{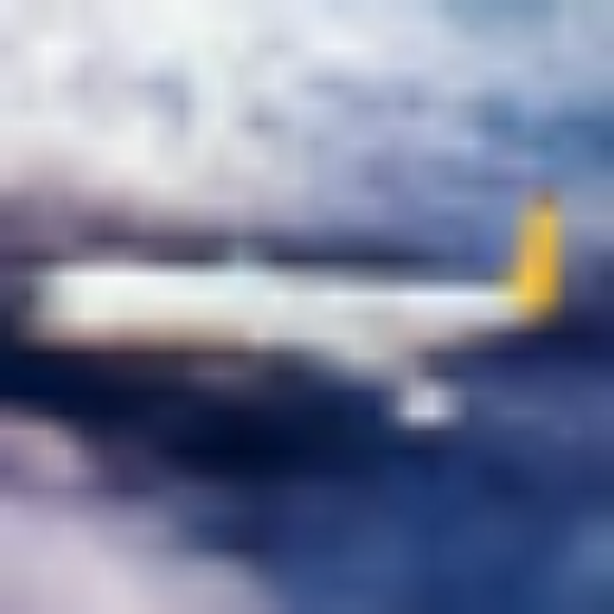}
\!\!\!
\includegraphics[width=0.24\linewidth]{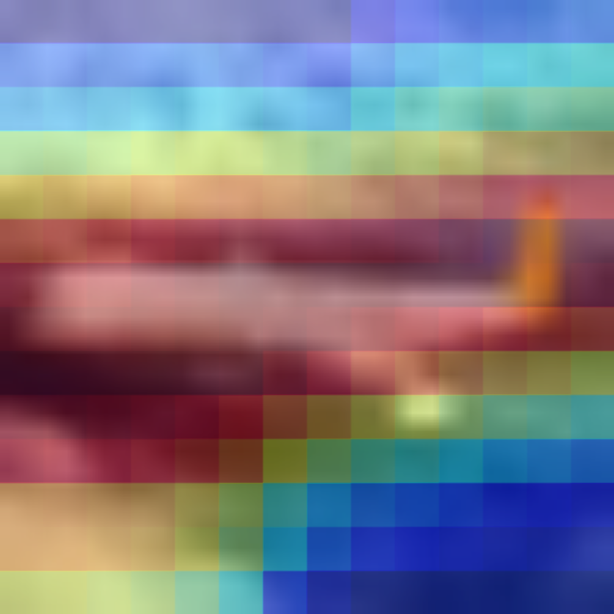}
\!\!\!
\includegraphics[width=0.24\linewidth]{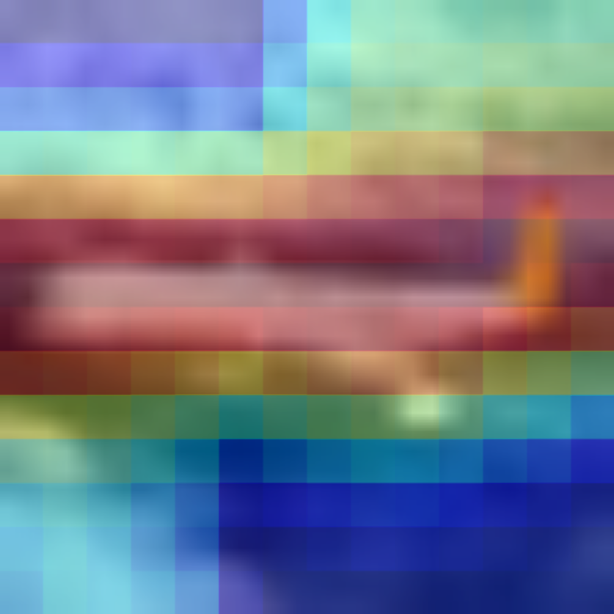}
\!\!\!
\includegraphics[width=0.24\linewidth]{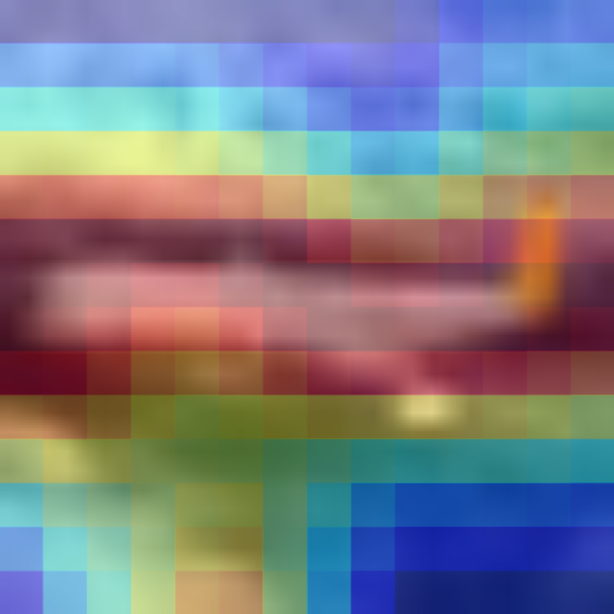}
\end{minipage}
}
% \!\!\!\!\!\!
\subfigure[Automobile]{
\begin{minipage}{0.3\textwidth}
\centering
\includegraphics[width=0.24\linewidth]{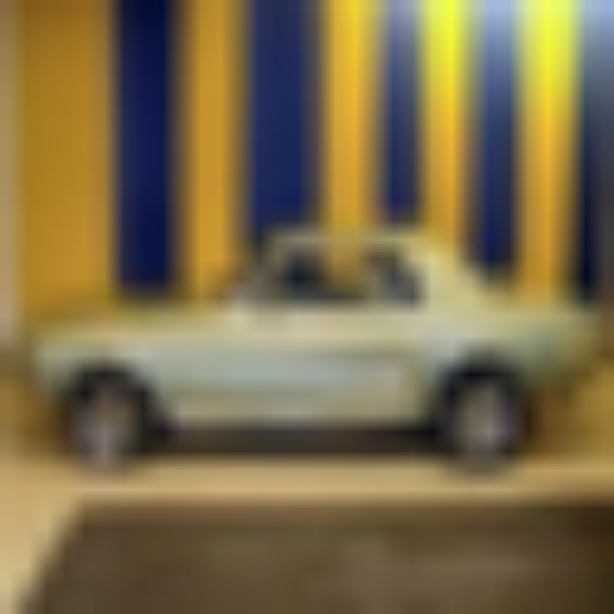}
\!\!\!
\includegraphics[width=0.24\linewidth]{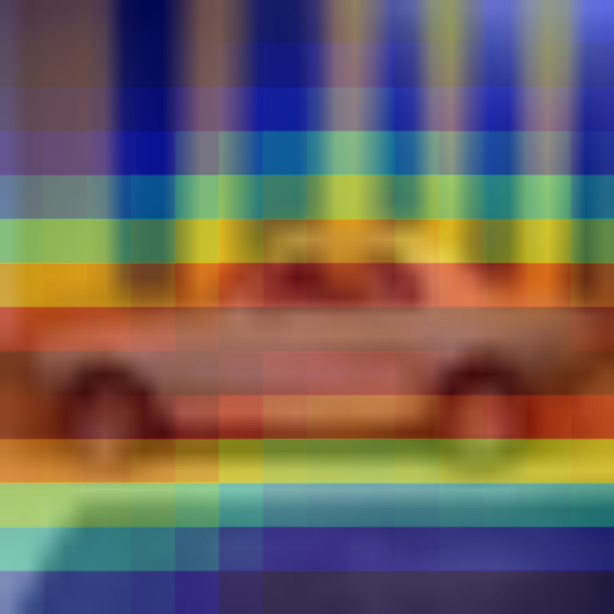}
\!\!\!
\includegraphics[width=0.24\linewidth]{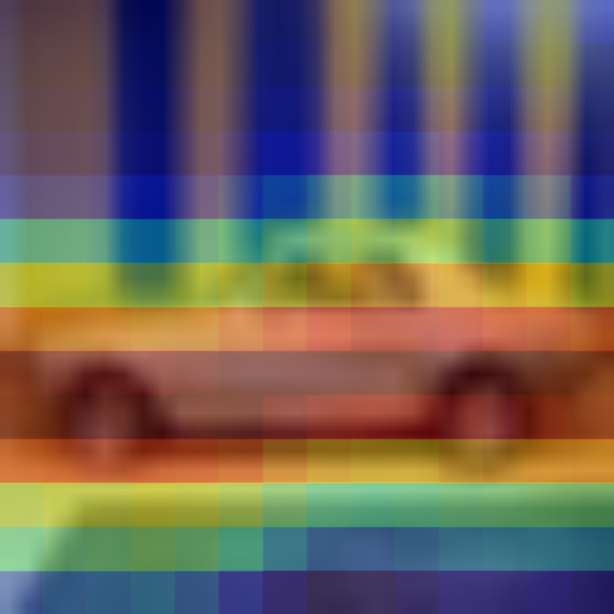}
\!\!\!
\includegraphics[width=0.24\linewidth]{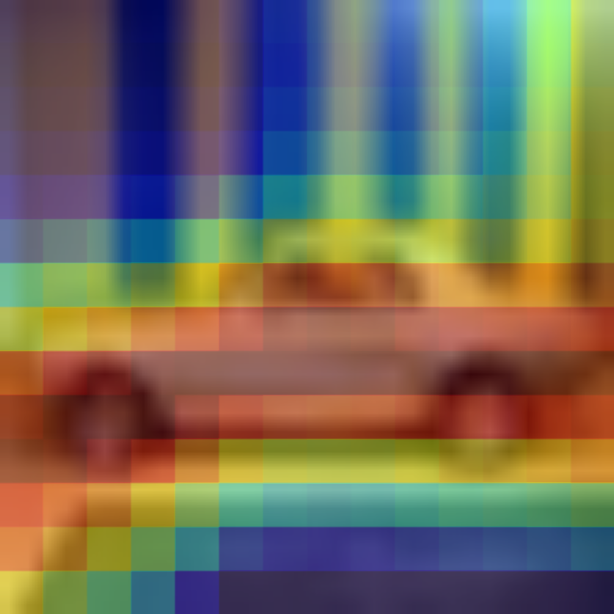}
\end{minipage}
}

\caption{\label{fig:case_study} \small Visualization of explanation on the \texttt{Cats-vs-dogs}~(a)-(c), \texttt{Imagenette}~(d)-(f), and \texttt{CIFAR-10}~(g)-(i) datasets. From the left to the right, each heatmap explains the inference of the \texttt{Swin-Base}, \texttt{Deit-Base}, and \texttt{ViT-Base} models, respectively.}

\end{figure*}

\subsection{Ablation Studies~(RQ4)}
\label{sec:ablation_study}

We ablatedly study the contribution of the key steps in \Algnameabbr{} to explaining downstream tasks, including the pre-training of transferable explainer and attribution transfer aligned to each task. % Proposition~\ref{definition:attribution_transfer}.
For our evaluation, we consider three methods: \Algnameabbr{} w/o Pre-training~(PT), \Algnameabbr{} w/o $H_t$, and \Algnameabbr{}.
Specifically, for \Algnameabbr{} w/o PT, the explainer is randomly initialized without pre-training, and attribution transfer follows Definition~\ref{definition:attribution_transfer}.
For \Algnameabbr{} w/o $H_t$, the transferable explainer is pre-trained following Algorithm~\ref{alg:pretrain}, and the explanation for each task is generated by $\hat{\boldsymbol{\phi}}_k = \log H_\textsl{g}( \hat{\textbf{\textsl{g}}}_{k}; y ) - \log H_\textsl{g}( \hat{\textbf{h}}_{k}; y )$, where $H_\textsl{g}$ takes a general classifier pre-trained on the \texttt{ImageNet} datasets, instead of being fine-tuned corresponding to the task.
Figure~\ref{fig:ablation_result} illustrates the results of $\mathrm{Fidelity}^+$-Sparsity-AUC($\uparrow$) and $\mathrm{Fidelity}^-$-Sparsity-AUC($\downarrow$) for each method, where the fidelity score represents the averaged value on the \texttt{Cats-vs-dogs}, \texttt{Imagenette}, and \texttt{CIFAR-10} datasets.
Other configurations remain consistent with Appendix~\ref{appendix:implement-details}.
Overall, we have the following observations:
\begin{itemize}[leftmargin=4mm, topsep=10pt]

    \item \emph{\Algnameabbr{} pre-training significantly contributes to explaining the downstream tasks.} 
    This can be verified by the fidelity degradation observed from \Algnameabbr{} w/o PT in Figure~\ref{fig:ablation_result}.
    
    % the two sub-figures of Figures~\ref{fig:ablation_result}.
    % This verifies the necessity of $\boldsymbol{\textsl{g}}_{k,z}$ and $\boldsymbol{\textsl{h}}_{k,z}$ to \Algnameabbr{}.

    \item \emph{The classifier $H_t$ for attribution transfer should align with the explaining task $t$.} 
    It is observed in Figure~\ref{fig:ablation_result} that \Algnameabbr{} outperforms \Algnameabbr{} w/o $H_{t}$.
    This indicates the task-aligned $H_t$ is better than general classifiers for the attribution transfer to a specific task $t$.
    
\end{itemize}

\subsection{Evaluation of Latency}
\label{sec:eval_latency}

In this section, we evaluate the latency of \Algnameabbr{} compared with baseline methods.
Specifically, we adopt the metric {\small $\mathrm{Throughput} = \frac{N_{\text{test}}}{T}$($\uparrow$)} to evaluate the explanation latency, where $N_{\text{test}}$ takes the number of testing instances and $T$ signifies the total time consumed during the explanation process.
Details about our computational infrastructure are given in Appendix~\ref{appendix:hardware}.
Figure~\ref{fig:throughput_result} shows the throughput of different methods explaining the {\small \texttt{ViT-Base/Large}, \texttt{Swin-Base/Large}, \texttt{Deit-Base}}, and {\small \texttt{ResNet-101/152}} models on the \texttt{ImageNet} dataset. 
Overall, we observe:

\begin{itemize}[leftmargin=4mm, topsep=10pt]

    \item \emph{\Algnameabbr{} is more efficient than state-of-the-art baseline methods,} by generating explanations through a single feed-forward pass of the explainer. In contrast, the baseline methods rely on intensive samplings of the forward or backward passes of the target model, resulting in a considerably slower explanation process.
    For example, although \texttt{KernelSHAP} exhibits comparable $\mathrm{Fidelity}^-$($\downarrow$) with \Algnameabbr{}, as shown in Figure~\ref{fig:fidelity_result}, its significantly lower throughput limits its practicality in real-world scenarios.
    % Although reducing the sampling number of KernelSHAP can speed it up, its fidelity turns further worse than \Algnameabbr{}.

    \item \emph{\Algnameabbr{} exhibits the most negligible decrease in throughput as the size of the target model grows}, as seen when transitioning from \texttt{ViT-Base} to \texttt{ViT-Large}.
    This advantage stems from the fact that \Algnameabbr{}'s latency is contingent upon the explainer's model size, rather than the target model.
    In contrast, the baseline methods suffer from notable performance slowdown as the size of the target model increases, due to the necessity of sampling the target model to generate explanations.
    
    % Our analysis reveals that the speed of the baseline methods decreases in tandem with the target model's size due to the necessity of sampling the target model to generate explanations.
    % \item Among the baseline methods, Gradient-based methods including DeepLift and GradShap achieve higher throughput than perturbation-based methods: RISE, LIME, SHAP.

\end{itemize}

\subsection{Case Studies}
\label{sec:case_study}

In this section, we visualize the explanations generated by \Algnameabbr{}, demonstrating its power in helping human users understand vision models.
Specifically, we randomly sample three instances from the \texttt{Cats-vs-dogs}, \texttt{Imagenette}, and \texttt{CIFAR-10} datasets, and visualize the explanations of \texttt{Swin-Base}, \texttt{Deit-Base}, and \texttt{ViT-Base} models in Figure~\ref{fig:case_study}, where sub-figures~(a)-(c), (d)-(f), and (g)-(i) correspond to the \texttt{Cats-vs-dogs}, \texttt{Imagenette}, and \texttt{CIFAR-10} datasets, respectively.
In each sub-figure, from the left-side to the right-side, the three heatmaps explain the inference of the \texttt{Swin-Base}, \texttt{Deit-Base}, and \texttt{ViT-Base} model, respectively.
Notably, \Algnameabbr{} generates the explanation heatmap in an end-to-end manner \emph{without pre- or post-processing}.
More case studies on the \texttt{ImageNet} dataset are shown in Appendix~\ref{appendix:case_study}.
According to the case study, we observe:

\begin{itemize}[leftmargin=4mm, topsep=10pt]

    % \item \emph{The important patches highlighted by the explanation form semantically meaningful patterns.} e.g. as illustrated in Figure~\ref{fig:case_study}~(a), (d), and (g), the Swin-Base model focus on the head, tower, and bow to identify a cat, church, and ship.

    \item \emph{The salient patches emphasized by \Algnameabbr{}'s explanation reveal semantically meaningful patterns.} For example, as depicted in Figures~\ref{fig:case_study}~(d), (e), and (g), the \texttt{Swin-Base} model concentrates on the tower, canopy and bow, respectively, to identify a church, parachute, and ship.

    \item \emph{\Algnameabbr{} does not rely on pre-processing of the  image or post-processing of the explanation heatmap.}
    In contrast, existing work EAC~\cite{sun2023explain} requires SAM~\cite{kirillov2023segment} to segment the input image before explaining, which is less flexible than \Algnameabbr{}.

    % \item \emph{The predictions by different model architectures are based on different elements in the images.} e.g. according to Figure~\ref{fig:case_study}~(h), on the CIFAR-10 datasets, the Swin-Base and DeiT-Base models focus on the body to identify the aircraft. In contrast, the ViT-Base model also the aircraft tail for the prediction.

    \item \emph{Different model architectures make predictions based on distinct image elements.}
    For instance, as illustrated in Figure~\ref{fig:case_study}~(g), the \texttt{Swin-Base} and \texttt{Deit-Base} models primarily emphasize the ship's bow for identification. In contrast, the \texttt{ViT-Base} model takes into account the ship's keel for its prediction.
    
\end{itemize}

\section{Conclution}

In this work, we propose a framework of attribution transfer, incorporating a meta-attribution to extract the foundation knowledge and a transfer rule to utilize this knowledge for explaining various downstream tasks.
Building upon this framework, we introduce \Algnameabbr{}, a transferable explainer pre-trained on large-scale image datasets.
Notably, \Algnameabbr{} shows strong transferability to effectively explain various downstream tasks without the need for training on task-specific data.
Experiment results validate the promising performance of \Algnameabbr{} in explaining three architectures of vision Transformer across three downstream datasets.
Significantly, the strong transferability of \Algnameabbr{} facilitates efficient and flexible deployment to various downstream scenarios.

% Specifically, \Algnameabbr{} guides the pre-training towards learning the meta-attribution that comprehensively encodes the essential attribution for explaining various downstream tasks; and introduces a rule-based adaptation of the meta-attribution for explaining downstream tasks, without requiring task-specific data.
% This enables efficient and flexible deployment across various downstream tasks. 
% Our theoretical analysis verifies that the pre-training of \Algnameabbr{} contributes to minimizing the explanation error bound aligned with the $\mathcal{V}-$information-based explanation of downstream tasks.
% Moreover, experiment results demonstrate that the pre-training of \Algnameabbr{} can effectively explain the models fine-tuned in various downstream scenarios, achieving competitive results compared with state-of-the-art explanation methods.

% \section{Impact Statements}

% This paper presents work whose goal is to advance the field of Machine Learning. There are many potential societal consequences of our work, none which we feel must be specifically highlighted here.

\bibliography{generic_xai}
\bibliographystyle{plain}

\clearpage
\appendix
\setcounter{theorem}{0}
\onecolumn

\section*{Appendix}

\section{Related Work}

Explainable machine learning (ML) has made significant advancements, leading to model transparency and
better human understanding of deep neural networks (DNNs)~\cite{du2019techniques}.
Specifically, existing work of explainable ML can be categorized into two groups: local explainers and DNN-based explainers~\cite{chuang2023efficient}.

\paragraph{Local Explainer.}
Local explainer focuses on constructing local explanation based on perturbation of the target black-box model, like KernelSHAP~\cite{lundberg2017unified}, LIME~\cite{ribeiro2016should}, GradCAM~\cite{selvaraju2017grad}, and Integrated Gradient~\cite{sundararajan2017axiomatic}.
Specifically, KernelSHAP approximates the Shapleyvalue by learning an explainable surrogate (linear) model based on the DNN output of reference input for each feature;
LIME generates the explanation by sampling points around the input instance and using DNN output at these points to learn a surrogate (linear) model; Integrated Gradients estimates the explanation by the integral of the gradients of DNN output with respect to the inputs, along the pathway from specified references to the inputs.
These pieces of work rely on resource-intensive procedures like sampling or backpropagation of the target black-box model~\cite{liu2021synthetic}, leading to undesirable trade-off between the efficiency and interpretation fidelity~\cite{chuang2023efficient}.

\paragraph{DNN-based Explainer.}
% Another group leverages the knowledge of explanation values to train DNN-based explainers, FastSHAP~\cite{jethani2021fastshap}, CORTX~\cite{chuang2023cortx}, LARA~\cite{rong2023efficient, wang2022accelerating}.  
% These pieces of work enable to efficiently generate explanations for an entire batch of instances through a single, streamlined feed-forward operation of the DNN-based explainer.
% However, they are constrained to explaining individual black box models, and they lack the ability to transfer the explanation across various models and tasks.
% This limitation results in the explanation of various tasks in practical scenarios becoming time- and resource-consuming due to the necessity of training different explainers for each task.

This branch of work leverages the training process to acquire proficiency in constructing a DNN-based explainer, utilizing explanation values as training labels~\cite{chuang2023efficient}. 
This innovative strategy empowers the simultaneous generation of explanations for an entire batch of instances through a single, streamlined feedforward operation of the DNN-based explainer.
Exemplifying this progress are innovative approaches like FastSHAP~\cite{jethani2021fastshap}, ViT-Shapley~\cite{covert2022learning}, CORTX~\cite{chuang2023cortx}, LARA~\cite{rong2023efficient, wang2022accelerating}, and HarsanyiNet~\cite{chen2023harsanyinet}. 
To be concrete, FastSHAP and ViT-Shapley adopt a DNN as the explainer to learn the Shapley value, which relies on task-specific training and cannot be transferred across different tasks; and CoRTX arguments the training of DNN-based explainer through a contrastive pre-training framework, and adopt the true Shapley value to fine-tune the explainer.
The DNN-based explainer have played a pivotal role in significantly streamlining the deployment of DNN explanations within real-time applications.
However, they are constrained to explaining individual black box models, and they lack the ability to transfer the explanation across various models and tasks. 
This limitation results in the explanation of various tasks in practical scenarios becoming time- and resource-consuming due to the necessity of training different explainers for each task.

\section{Approximation of Attribution}
\label{appendix:B-sampling-experiment}

\begin{figure}
\centering
\subfigure[Sampling number 16]{
\includegraphics[width=0.3\linewidth]{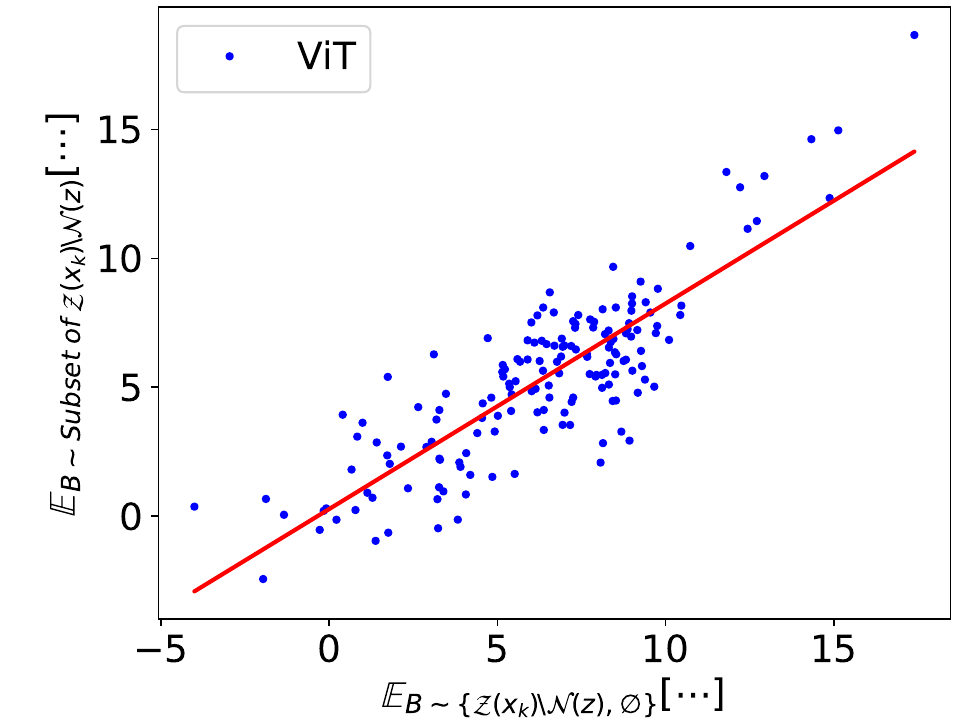}
}
% \!\!\!\!
\subfigure[Sampling number 16]{
\includegraphics[width=0.3\linewidth]{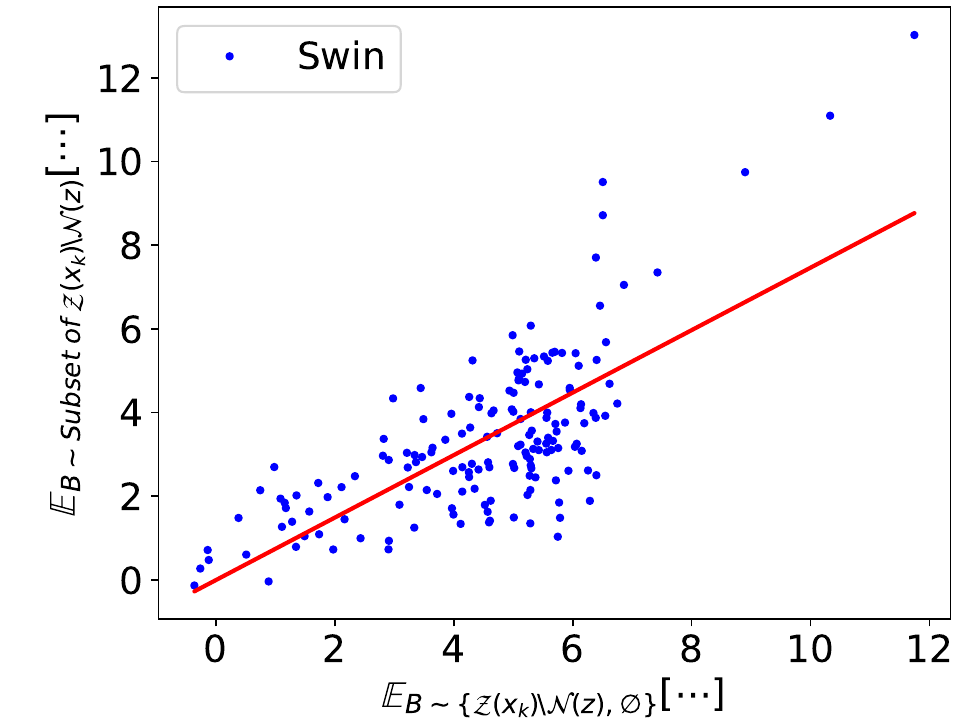}
}
% \!\!\!\!
\subfigure[Sampling number 16]{
\includegraphics[width=0.3\linewidth]{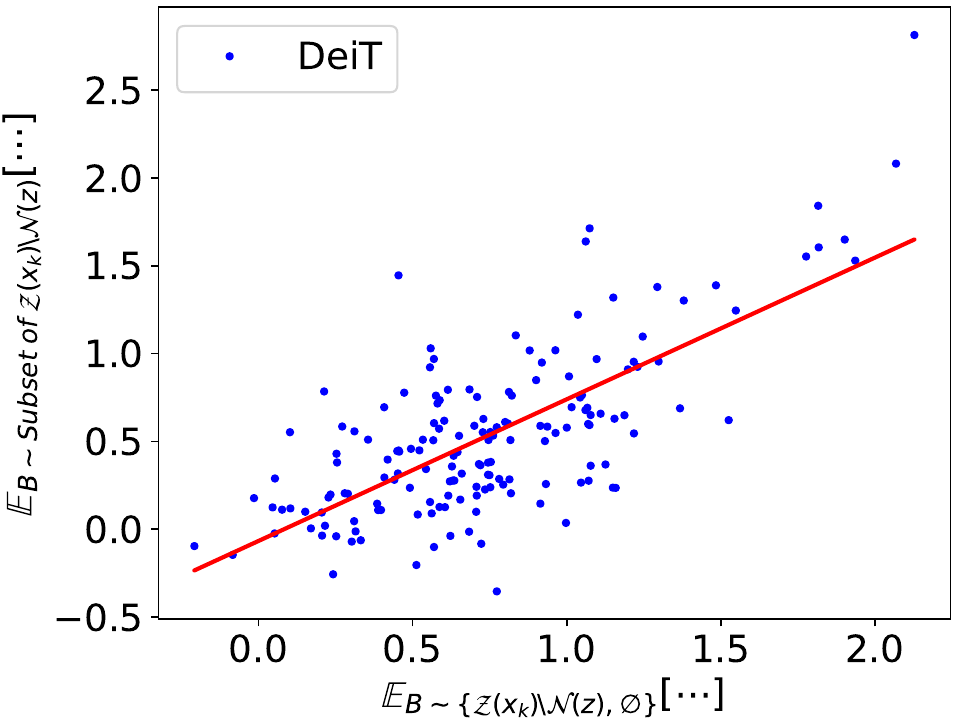}
}
% \!\!\!\!
\subfigure[Sampling number 32]{
\includegraphics[width=0.3\linewidth]{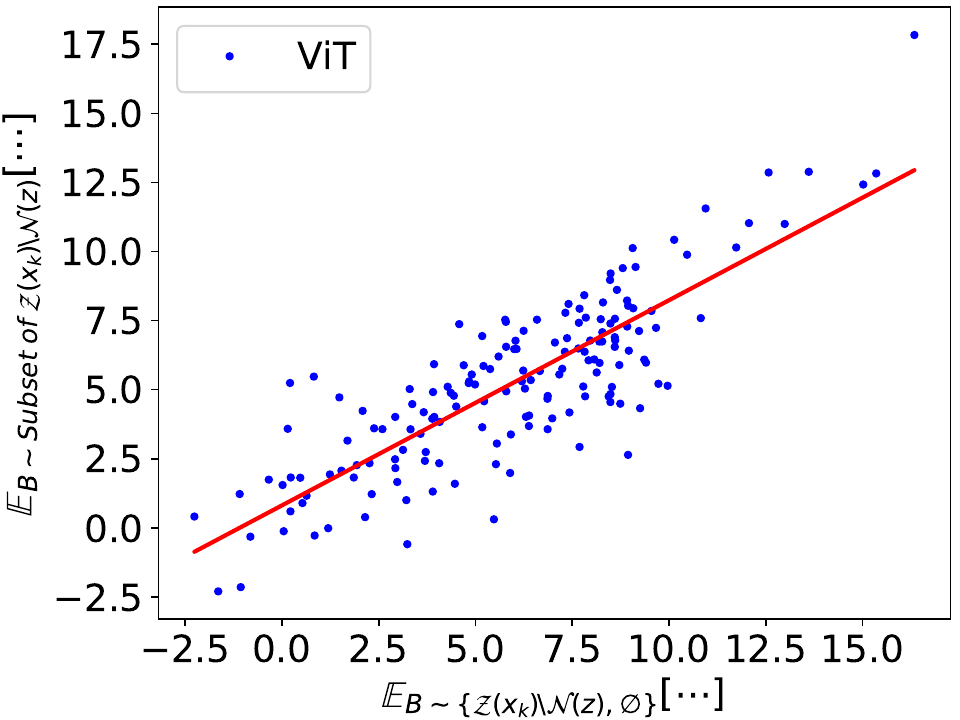}
}
% \!\!\!\!
\subfigure[Sampling number 32]{
\includegraphics[width=0.3\linewidth]{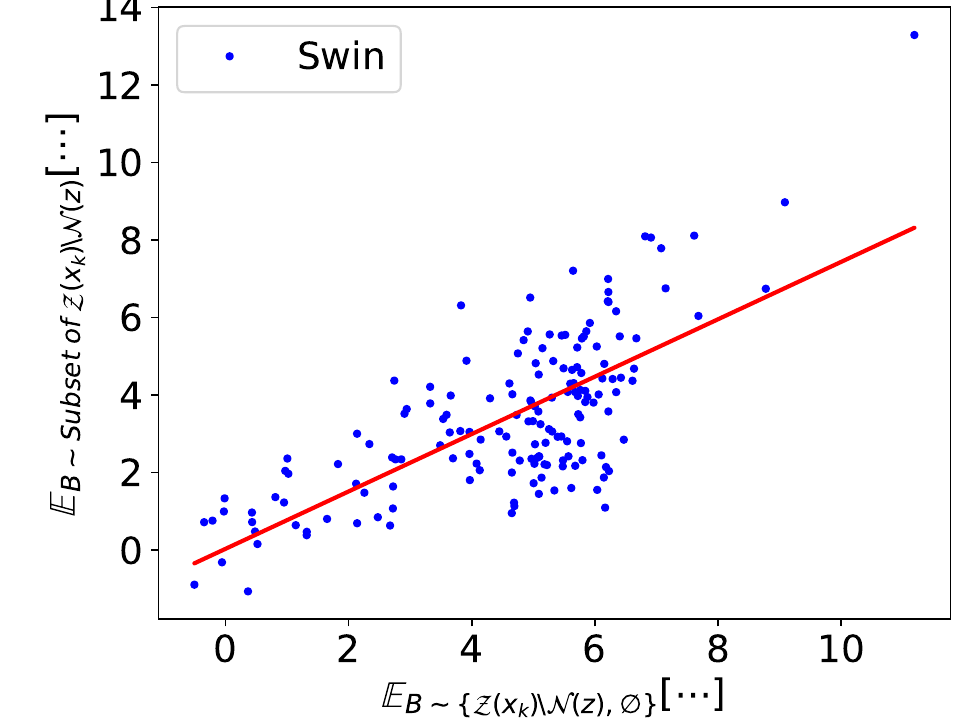}
}
% \!\!\!\!
\subfigure[Sampling number 32]{
\includegraphics[width=0.3\linewidth]{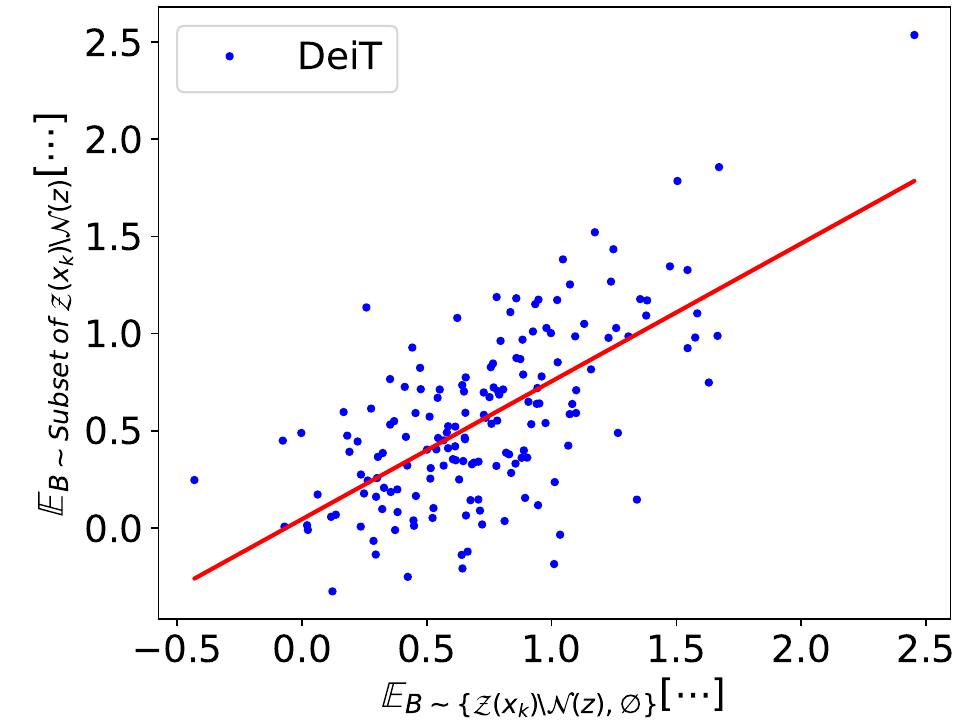}
}
\caption{\label{fig:b_sample} $\mathbb{E}_{B \sim \text{Subset of } \mathcal{Z}(\boldsymbol{x}_k) \setminus \mathcal{N}(z), \varnothing }[ \cdots ]$ versus $\mathbb{E}_{B \sim \{ \mathcal{Z}(\boldsymbol{x}_k) \setminus \mathcal{N}(z), \varnothing \}}[ \cdots ]$, where $\cdots$ is the abbreviation of $-\log f_{t}(B ; \boldsymbol{x}_k, y) + \log f_{t}(\mathcal{N}(z) \cup B ; \boldsymbol{x}_k, y)$; and $f_{t}$ takes the trained \texttt{ViT-Base}(a, d), \texttt{Swin-Base}(b, e), and \texttt{Deit-Base}(c, f) models on the ImageNet dataset.
The sampling number of $B \sim \text{Subset of } \mathcal{Z} (\boldsymbol{x}_k) \setminus \mathcal{N}(z)$ is 16 and 32 for Sub-figures~(a)-(c) and~(d)-(f), respectively.
}
\end{figure}

% The philosophy behind our experiment is to demonstrate the   aligning with the  

We conduct experiments to study the relationship between the approximate attribution $\mathbb{E}_{B \sim \{ \mathcal{Z}(\boldsymbol{x}_k) \setminus \mathcal{N}(z), \varnothing \}}[ \cdots ]$ and its exact value $\mathbb{E}_{B \sim \text{Subset of } \mathcal{Z}(\boldsymbol{x}_k) \setminus \mathcal{N}(z) }[ \cdots ]$ on the \texttt{ImageNet} dataset, where $\cdots$ is the abbreviation of $-\log f_{t}(B ; \boldsymbol{x}_k, y) + \log f_{t}(\mathcal{N}(z) \cup B ; \boldsymbol{x}_k, y)$.
Specifically, we collect the samples of $\mathbb{E}_{B \sim \{ \mathcal{Z}(\boldsymbol{x}_k) \setminus \mathcal{N}(z), \varnothing \}}[ \cdots ]$ and $\mathbb{E}_{B \sim \text{Subset of } \mathcal{Z}(\boldsymbol{x}_k) \setminus \mathcal{N}(z) }[ \cdots ]$, where $\boldsymbol{x}_k$ take 100 instances randomly sampled from the \texttt{ImageNet} dataset; and the target models $f_{t}$ take the \texttt{ViT-Base}(a, d), \texttt{Swin-Base}(b, e), and \texttt{Deit-Base}(c, f) models trained on the \texttt{ImageNet} dataset.
The samples of $\mathbb{E}_{B \sim \text{Subset of } \mathcal{Z}(\boldsymbol{x}_k) \setminus \mathcal{N}(z) }[ \cdots ]$ versus $\mathbb{E}_{B \sim \{ \mathcal{Z}(\boldsymbol{x}_k) \setminus \mathcal{N}(z), \varnothing \}}[ \cdots ]$ is plotted in Figure~\ref{fig:b_sample}.
It is observed that the value of $\mathbb{E}_{B \sim \{ \mathcal{Z}(\boldsymbol{x}_k) \setminus \mathcal{N}(z), \varnothing \}}[ \cdots ]$ after the approximation shows positive linear correlation with $\mathbb{E}_{B \sim \text{Subset of } \mathcal{Z}(\boldsymbol{x}_k) \setminus \mathcal{N}(z) }[ \cdots ]$.
This indicates the approximate value $\mathbb{E}_{B \sim \{ \mathcal{Z}(\boldsymbol{x}_k) \setminus \mathcal{N}(z), \varnothing \}}[ \cdots ]$ can take the place of $\mathbb{E}_{B \sim \text{Subset of } \mathcal{Z}(\boldsymbol{x}_k) \setminus \mathcal{N}(z) }[ \cdots ]$ for the function of attribution.

\section{Proof of Theorem~\ref{theorem:head-attribution}}
\label{appendix:theorem-proof}

We prove Theorem~\ref{theorem:head-attribution} in this section.
\begin{theorem}[Explanation Error Bound]
Given the classifier ${H}_t(\bullet)$ of the downstream task, if the output of classifier ${H}_t(\hat{\textbf{\textsl{g}}}_{k,z}; y)$ and ${H}_t(\hat{\textbf{h}}_{k,z}; y)$ fall within the range of $1-\epsilon \leq \frac{{H}_t(\hat{\textbf{\textsl{g}}}_{k,z}; y)}{{H}_t(\textbf{\textsl{g}}_{k,z}; y)}, \frac{{H}_t(\textbf{h}_{k,z}; y)}{{H}_t(\hat{\textbf{h}}_{k,z}; y)} \leq 1+\epsilon$, then, the upper bound of explanation error is given by 
\begin{equation}
\label{eq:estimation_error_bound_appendix}
\mathbb{E}_{\boldsymbol{x}_k \sim \mathcal{D}_{t}, y \sim \mathcal{Y}_t, {z} \sim \mathcal{Z}(\boldsymbol{x}_k)} | \hat{\phi}_{k, y, {z}} - \phi_{k, y, {z}} | \leq \frac{2 \epsilon}{1 - \epsilon},
\end{equation}
where $\hat{\phi}_{k, y, {z}}$ and $\phi_{k, y, {z}}$ are given by Equation~(\ref{eq:estimate_generic_attribution}) and~(\ref{eq:attribution_trasfer}), respectively; and $\mathcal{D}_{t}$ denotes the downstream dataset.
\end{theorem}

\begin{proof}
To achieve the explanation error bound, we first have the upper bound of $\hat{\phi}_{k,y,z} - \phi_{k,y,z}$ given by
\begin{align}
\hat{\phi}_{k,y,z} - \phi_{k,y,z}
&= \log H_t(\hat{\textbf{\textsl{g}}}_{k,z}; y) - \log H_t(\textbf{\textsl{g}}_{k,z}; y) + \log H_t(\textbf{\textrm{h}}_{k,z}; y) - \log H_t(\hat{\textbf{\textrm{h}}}_{k,z}; y),
\\
&= \log \frac{H_t(\hat{\textbf{\textsl{g}}}_{k,z}; y)}{H_t(\textbf{\textsl{g}}_{k,z}; y)} + \log \frac{H_t(\textbf{\textrm{h}}_{k,z}; y)}{H_t(\hat{\textbf{\textrm{h}}}_{k,z}; y)} 
\leq \frac{H_t(\hat{\textbf{\textsl{g}}}_{k,z}; y)}{H_t(\textbf{\textsl{g}}_{k,z}; y)} - 1 + \frac{H_t(\textbf{\textrm{h}}_{k,z}; y)}{H_t(\hat{\textbf{\textrm{h}}}_{k,z}; y)} - 1,
\\
\label{eq:error_upper_bound}
&\leq \frac{H_t(\hat{\textbf{\textsl{g}}}_{k,z}; y)}{H_t(\textbf{\textsl{g}}_{k,z}; y)} - 1 + \frac{H_t(\textbf{\textrm{h}}_{k,z}; y)}{H_t(\hat{\textbf{\textrm{h}}}_{k,z}; y)} - 1 
\leq \epsilon + \epsilon,
\end{align}

Then, we have the lower bound of $\hat{\phi}_{k,y,z} - \phi_{k,y,z}$ as follows,
\begin{align}
\hat{\phi}_{k,y,z} - \phi_{k,y,z}
&= -\log \frac{H_t(\textbf{\textsl{g}}_{k,z}; y)}{H_t(\hat{\textbf{\textsl{g}}}_{k,z}; y)} - \log \frac{H_t(\hat{\textbf{\textrm{h}}}_{k,z}; y)}{H_t(\textbf{\textrm{h}}_{k,z}; y)} 
\geq 1 - \frac{H_t(\textbf{\textsl{g}}_{k,z}; y)}{H_t(\hat{\textbf{\textsl{g}}}_{k,z}; y)} + 1 - \frac{H_t(\hat{\textbf{\textrm{h}}}_{k,z}; y)}{H_t(\textbf{\textrm{h}}_{k,z}; y)},
\\
\label{eq:error_lower_bound}
&= 2 - \bigg( \frac{H_t(\textbf{\textsl{g}}_{k,z}; y)}{H_t(\hat{\textbf{\textsl{g}}}_{k,z}; y)} + \frac{H_t(\hat{\textbf{\textrm{h}}}_{k,z}; y)}{H_t(\textbf{\textrm{h}}_{k,z}; y)} \bigg)
\geq 2 - \frac{1}{1-\epsilon} - \frac{1}{1-\epsilon} = \frac{-2 \epsilon}{1-\epsilon}
\end{align}

Combining Equations~(\ref{eq:estimation_error_bound_appendix}) and (\ref{eq:error_lower_bound}), we achieve the upper bound of estimation error given by
\begin{align}
|\hat{\phi}_{k,y,z} - \phi_{k,y,z}| \leq \max \Big\{ 2 \epsilon, \frac{2 \epsilon}{1-\epsilon} \Big\} = \frac{2 \epsilon}{1-\epsilon}.
\end{align}
\end{proof}

% \section{Experiment Setup}
% \label{appendix:exp_setup}

% We provide details about benchmark datasets, baseline methods, evaluation metrics, and implementation details in this section.

% \paragraph{Datasets.} We consider the large-scale ImageNet dataset~\cite{deng2009imagenet} for \Algnameabbr{} pre-training; and the \texttt{Cats-vs-dogs}~\cite{asirra-a-captcha}, \texttt{CIFAR-10}~\cite{krizhevsky2009learning}, and \texttt{Imagenette}~\cite{Howard_Imagenette_2019} datasets for the downstream task of  explanation.
% More details about the datasets are given in Appendix~\ref{appendix:datasets}.

% \footnote{\tiny \url{https://github.com/huggingface/transformers/blob/main/src/transformers/models/vit/modeling_vit.py}}
% \footnote{\tiny \url{https://github.com/huggingface/transformers/blob/main/src/transformers/models/swin/modeling_swin.py}}
% \footnote{\tiny \url{https://github.com/huggingface/transformers/blob/main/src/transformers/models/deit/modeling_deit.py}}

\section{Details about the Datasets}
\label{appendix:datasets}

We consider the large-scale ImageNet dataset~\cite{deng2009imagenet} for \Algnameabbr{} pre-training; and the \texttt{Cats-vs-dogs}~\cite{asirra-a-captcha}, \texttt{CIFAR-10}~\cite{krizhevsky2009learning}, and \texttt{Imagenette}~\cite{Howard_Imagenette_2019} datasets for the downstream task of explanation.
\textbf{ImageNet}~\cite{deng2009imagenet}: A large scale image dataset which has over one million color images covering 1000 categories, where each image has $224 \!\times\! 224$ pixels.
\textbf{Cats-vs-dogs}~\cite{asirra-a-captcha}: A dataset of cats and dogs images. It has 25000 training instances and 12500 testing instances. 
\textbf{CIFAR-10}~\cite{krizhevsky2009learning}: An image dataset with 60,000 color images in 10 different classes, where each image has $32 \!\times\! 32$ pixels.
\textbf{Imagenette}~\cite{Howard_Imagenette_2019}: A benchmark dataset of explainable machine learning for vision models. It contains 10 classes of the images from the Imagenet.

% \textbf{InputXGrad~\cite{shrikumar2016not}:} This method generates the explanation multiplying the input with the gradient with respect to input.

\section{Details about Target Models for Downstream Classification.}
\label{appendix:target_model}

\subsection{Setup of Fine-tuning the Target Models}

For downstream classification tasks, we comprehensively consider three architectures of vision transformers as the backbone encoders, including the \texttt{ViT-Base/Large}~\cite{dosovitskiy2020image}, \texttt{Swin-Base/Large}~\cite{liu2021swin}, \texttt{Deit-Base}~\cite{touvron2021training} transformers.
The classification models~(to be explained) consist of one of the backbone encoders with \texttt{ImageNet} pre-trained weights and a linear classifier.
For the task-specific fine-tuning of target models, we consider two mechanisms: \emph{classifier-tuning} and \emph{full-fine-tuning}.
Specifically, the classifier-tuning follows the transfer learning setting~\cite{chilamkurthy2017transfer, he2020momentum, chen2020simple} to freeze the parameters of backbone encoder during the fine-tuning; and the full-fine-tuning updates all parameters during the finetuning.
Note that the classifier-tuning can not only be more efficient but also prevent the over-fitting problem on downstream data due to fewer trainable parameters~\cite{sun2022singular}.
We consider the classifier-tuning for most of our experiments including Sections~\ref{sec:fidelity_eval}, \ref{sec:eval_transfer}, \ref{sec:ablation_study}, and \ref{sec:case_study}; and consider the full-fine-tuning in Section~\ref{sec:adaptation_finetune_backbone}; while these two mechanisms yield the same result for Section~\ref{sec:eval_latency}.
The hyper-parameters of task-specific fine-tuning are given in Appendix~\ref{appendix:target-model-hyperparam}.

\subsection{Hyper-parameter Setting of Fine-tuning the Target Models on Downstream Tasks}
\label{appendix:target-model-hyperparam}

The downstream classification models consist of the backbones of \texttt{ViT-Base/Large}, \texttt{Swin-Base/Large}, \texttt{Deit-Base} transformers, and a linear classifier.
The hyper-parameters of fine-tuning the classification models on the \texttt{Cats-vs-dogs}, \texttt{CIFAR-10}, and \texttt{Imagenette} datasets are given in Table~\ref{tab:target_model_hyperparam}.
After the fine-tuning, the classification accuracy on each downstream dataset is given in Table~\ref{tab:target_model_accurcy}.

\begin{table}[H]
    \centering
    \caption{Hyper-parameters of fine-tuning the target model on downstream datasets.}
    \begin{tabular}{lccc}
    \toprule
         Datasets & \texttt{Cats-vs-dogs} & \texttt{CIFAR-10} & \texttt{Imagenette} \\
         % Datasets & {Census Income} & {German Credit} & {GAS Sensor} & {NATICUSdroid} \\
    \midrule
         Target backbone & \multicolumn{3}{c}{\texttt{ViT-Base}, 
         \texttt{Swin-Base}, and \texttt{Deit-Base}} \\
         Classifier & \multicolumn{3}{c}{Linear classifier} \\
         Fine-tuning mechanism & \multicolumn{3}{c}{classifier-tuning and full-fine-tuning} \\
         Optimizer & \multicolumn{3}{c}{ADAM} \\
         Learning rate & \multicolumn{3}{c}{$2 \times 10^{-4}$} \\
         Mini-batch size & \multicolumn{3}{c}{256} \\
         Scheduler & \multicolumn{3}{c}{Linear} \\
         Warm-up-ratio & \multicolumn{3}{c}{0.05} \\
         Weight-decay & \multicolumn{3}{c}{0.05} \\
         Epoch & \multicolumn{3}{c}{5} \\
    \bottomrule     
    \end{tabular}
    \label{tab:target_model_hyperparam}
\end{table}

\begin{table}[H]
\centering
\caption{Accuracy of the target model on downstream datasets.}
\begin{tabular}{l|cc|cc|cc}
\toprule
 Model Architecture & \multicolumn{2}{c|}{\texttt{ViT-Base}} & \multicolumn{2}{c|}{\texttt{Swin-Base}} & \multicolumn{2}{c}{\texttt{Deit-Base}} \\
\midrule
Tunable parameters & $\theta_H$ & $\theta_H, \theta_G$ & $\theta_H$ & $\theta_H, \theta_G$ & $\theta_H$ & $\theta_H, \theta_G$ \\ 
\midrule
\texttt{Cats-vs-dogs} & 99.6\% & 99.5\% & 99.6\% & 99.7\% & 99.4\% & 98.1\% \\
\texttt{Imagenette} & 99.3\% & 99.3\% & 99.8\% & 99.7\%  & 99.8\% & 99.4\% \\
\texttt{CIFAR-10} & 92.2\% & 98.9\% & 97.0\% & 98.6\% & 94.2\% & 98.1\% \\
\bottomrule
\end{tabular}
\label{tab:target_model_accurcy}
\end{table}

\section{Details about the Baseline Methods}
\label{appendix:baseline}

% \paragraph{Baseline Methods.}
We consider seven baseline methods for comparison, which include general
explanation methods: \texttt{LIME}~\cite{ribeiro2016should}, \texttt{IG}~\cite{sundararajan2017axiomatic}, \texttt{RISE}~\cite{petsiuk2018rise}, and \texttt{DeepLift}~\cite{ancona2017towards}; Shapley explanation methods: \texttt{KernelSHAP}~(\texttt{KS})~\cite{lundberg2017unified}, and \texttt{GradShap}~\cite{lundberg2017unified}; and DNN-based explainer: \texttt{ViT-Shapley}~\cite{covert2022learning} in our experiment. 
% More details about the baseline methods are given in Appendix~\ref{appendix:baseline}.

% We give the details about the baseline methods in this section.
\textbf{ViT-Shapley:} This work adopts vision transformers as the explainer to learn the Shapley value.
This work requires task-specific data to train the explainer.
\textbf{RISE:} RISE randomly perturbs the input, and average all the masks weighted by the perturbed DNN output for the final saliency map. 
The sampling number takes the default value $50$.
\textbf{IG:} Integrated Gradients estimates the explanation by the integral of the gradients of DNN output with respect to the inputs, along the pathway from specified references to the inputs.
\textbf{DeepLift:} DeepLift generates the explanation by decomposing DNN output on a specific input by backpropagating the contributions of all neurons in the network to every feature of the input.
\textbf{KernelSHAP:} KernelSHAP approximates the Shapley value by learning an explainable surrogate (linear) model based on the DNN output of reference input for each feature. The sampling number takes the default value $25$ for each instance according to the \texttt{captum.ai}~\cite{kokhlikyan2020captum}.
\textbf{GradShap:} GradShap estimates the importance features by computing the expectations of gradients by randomly sampling from the distribution of references.
\textbf{LIME~\cite{ribeiro2016should}:} LIME generates the explanation by sampling points around the input instance and using DNN output at these points to learn a surrogate (linear) model. The sampling number takes the default value $25$ according to the \texttt{captum.ai}.
For implementation, we take the IG, DeepLift, and GradShap algorithms on the \texttt{captum.ai}, where the \texttt{multiply\_by\_inputs} factor takes false to achieve the local attribution for each instance.

\section{Evaluation Metrics}
\label{appendix:eval_metric}

\paragraph{Fidelity-sparsity Curve:}
\label{appendix:eval_fidelity}
We consider the fidelity to evaluate the explanation following existing work~\cite{yang2022tutorial, chuang2023cortx}.
Specifically, the fidelity evaluates the explanation via \textit{removing the important or trivial patches} from the input instance and \textit{collecting the prediction difference of the target model} $f_{t}$. 
These two perspectives of evluation are formalized into $\mathrm{Fidelity}^+$ and $\mathrm{Fidelity}^-$, respectively.
Specifically, provided a subset of patches $\mathcal{S}^* \subseteq \mathcal{Z}(\boldsymbol{x}_k)$ that are important to the target model $f_{t}$ by an explanation method,
the $\mathrm{Fidelity}^+$ and $\mathrm{Fidelity}^-$ evaluates the explanation following
\begin{align}
&\uparrow \mathrm{Fidelity}^+ = \frac{1}{|\mathcal{D}_{\text{task}}|} \sum_{\boldsymbol{x} \in \mathcal{D}_{\text{task}}} f_{t}(\mathcal{Z}(\boldsymbol{x}_k); \boldsymbol{x}_k, y) - f_{t}( \mathcal{Z}(\boldsymbol{x}_k) \setminus \mathcal{S}^*; \boldsymbol{x}_k, y ),
\nonumber
\\
&\downarrow \mathrm{Fidelity}^- = \frac{1}{|\mathcal{D}_{\text{task}}|} \sum_{\boldsymbol{x} \in \mathcal{D}_{\text{task}}} f_{t}(\mathcal{Z}(\boldsymbol{x}_k); \boldsymbol{x}_k, y) - f_{t}( \mathcal{S}^*; \boldsymbol{x}_k, y).
\nonumber
\end{align}
Higher $\mathrm{Fidelity}^+$ indicates a better explanation for prediction $y$, since the truly important patches of image $\boldsymbol{x}_k$ have been removed, leading to a significant difference of model prediction.
Moreover, lower $\mathrm{Fidelity}^-$ implies a better explanation for prediction $y$, since the truly important patches have been preserved in $\mathcal{S}^*$ to keep the prediction similar to the original one. 
The fidelity should be compared at the same level of sparsity $|\mathcal{S}^*|/|\mathcal{U}|$.
Consequently, we consider the evaluation of fidelity versus the sparsity in most cases.

% (\mathcal{S}^*, \boldsymbol{x}_k, y)

\paragraph{Fidelity-sparsity-AUC Metric}
\label{appendix:auc-eval-metric}
% We give details about the metrics $\mathrm{Fidelity}^+$-sparsity-AUC and $\mathrm{Fidelity}^-$-sparsity-AUC in this section.
To streamline our evaluation, we simplify the assessment of fidelity-sparsity curves by calculating its Area Under the Curve~(AUC) over the sparsity from zero to one, which aligns with the average fidelity value. In the last paragraph, we have shown that higher $\mathrm{Fidelity}^+$ and lower $\mathrm{Fidelity}^-$ at the same level of sparsity indicate more faithful explanation. 
To streamline the evaluation, the assessment of fidelity-sparsity curves can be simplified into its Area Under the Curve (AUC) over the sparsity from zero to one, as shown in Figures~\ref{fig:auc-illustration}~(a) and (b).
The Fidelity-sparsity-AUC aligns with the average fidelity value. 
Specifically, a higher $\mathrm{Fidelity}^+$-sparsity-AUC~($\uparrow$) indicates better $\mathrm{Fidelity}^+$ performance across most sparsity levels, reflecting a more faithful explanation. 
Similarly, a lower $\mathrm{Fidelity}^-$-sparsity-AUC signifies a more faithful explanation
For the given example in Figures~\ref{fig:auc-illustration}~(a) and (b), explanation A is more faithful than B.

\begin{figure}[H]
\centering
\subfigure[]{
    \centering
	\begin{minipage}[t]{0.3\linewidth}
		\includegraphics[width=0.99\linewidth]{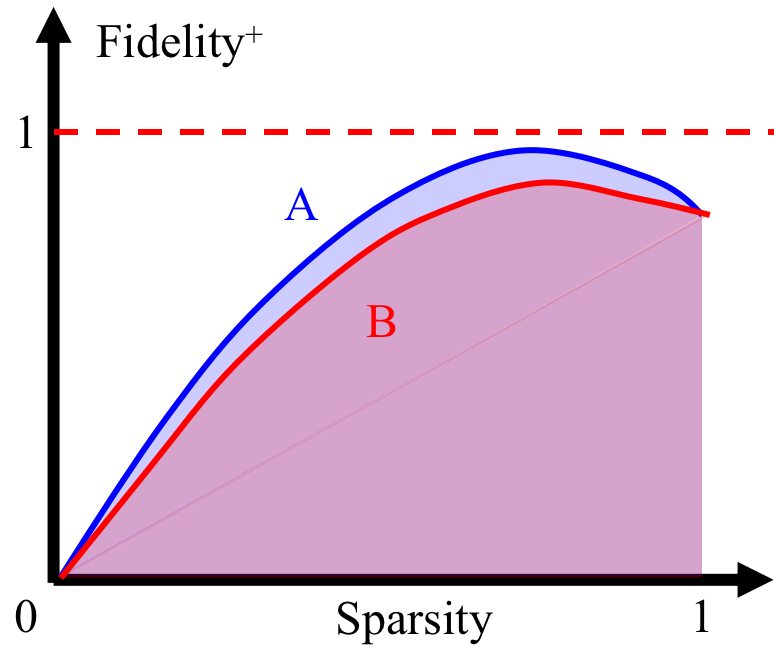}
	\end{minipage}

}
\quad\quad
\subfigure[]{
    \centering
	\begin{minipage}[t]{0.3\linewidth}
		\includegraphics[width=0.99\linewidth]{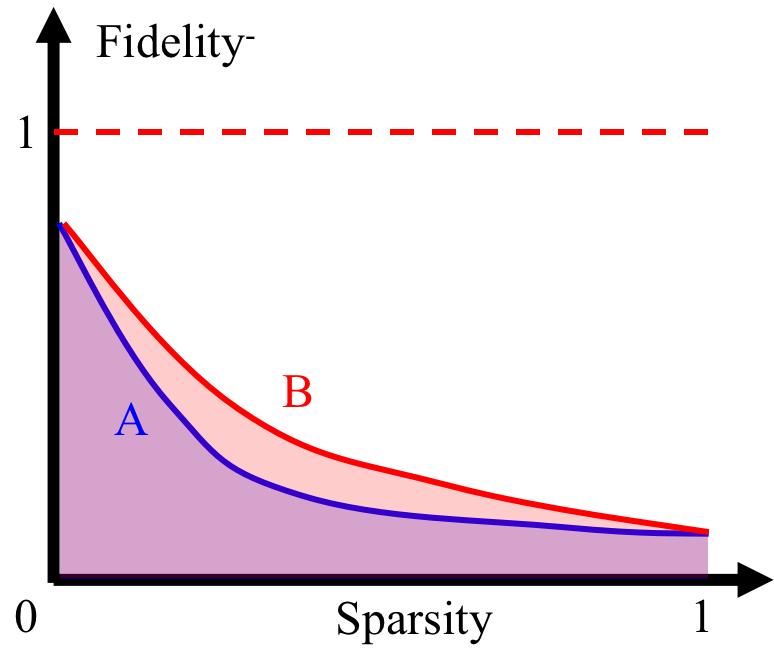}
	\end{minipage}

}

    \caption{Illustration of $\mathrm{Fidelity}^+$-sparsity-AUC~(a) and $\mathrm{Fidelity}^-$-sparsity-AUC~(b)}
    \label{fig:auc-illustration}

\end{figure}

% \paragraph{Hyper-parameter Settings.}
% The experiment follows the pipeline of \Algnameabbr{} pre-training, explanation generation and evaluation on multiple downstream datasets.
% Specifically, \Algnameabbr{} adopts the Mask-AutoEncoder~\cite{he2022masked} as the backbone, followed by a multiple Feed-Forward~(FFN) layers\footnote{\scriptsize A Mask-AutoEncoder consists of a ViT encoder followed by a ViT decoder; and the FFN layers are widely used in Transformers, consisting of Linear layers, Layer-norm, and activation function. More details about the architecture are given in Appendix~\ref{appendix:implement-details}} to generate the meta-attribution.
% More details about the explainer backbone, explainer architecture and hyper-parameters of \Algnameabbr{} pre-training are given in Appendix~\ref{appendix:implement-details}.
% When deploying \Algnameabbr{} to explaining downstream tasks, the explanation is aligned with the classification result of the target model.

\section{Implementation Details about \Algnameabbr{}}
\label{appendix:implement-details}

% We give the details about \Algnameabbr{} pre-training and target model training on each downstream datasets in \ref{appendix:alg-hyperparam} and \ref{appendix:target-model-hyperparam}, respectively.

% \subsection{\Algnameabbr{} Pretraining}
% \label{appendix:alg-hyperparam}

\paragraph{Architecture of Generic Explainer.}
The architecture of the transferable explainer is shown in Figure~\ref{fig:explainer-architecture}~(a).
Specifically, the explainer takes the \texttt{Mask-AutoEncoder-Base}~\cite{he2022masked} for the backbone.
As shown in Figure~\ref{fig:mae-architecture}, the \texttt{Mask-AutoEncoder-Base} architecture is a pipeline of 12-layer ViT encoder and 8-layer ViT decoder, where the input and output shape are $[\mathrm{BS}, 3, 224, 224]$ and $[\mathrm{BS}, \mathrm{P}, \mathrm{P}, 768]$, respectively.
More details about the \texttt{Mask-AutoEncoder-Base} can be referred to its source code\footnote{\scriptsize \url{https://github.com/huggingface/transformers/blob/main/src/transformers/models/vit_mae/modeling_vit_mae.py}}.

Since the output shape of the \texttt{Mask-AutoEncoder-Base} is $[\mathrm{BS}, \mathrm{P} \times \mathrm{P}, 768]$ is not matched with that of the meta-attribution $[\mathrm{BS}, \mathrm{P} \times \mathrm{P}, \mathrm{D}]$, where $\mathrm{BS}$ denotes the mini-batch size.
We adopt $n\times$ FFN-layers as explainer heads to map the output tensor of the \texttt{Mask-AutoEncoder-Base} into meta-attribution, where we found $n=17$ enables the expalainer to have strong generalization ability to explain various downstream tasks.
The structure of an explainer head is given in Figure~\ref{fig:explainer-architecture}~(b).
The first explainer head does not have the skip connection due to the mismatch of tensor shapes.
The last explainer head does not have the GELU activation.
% In our experiment, we found 17-layer explainer head enables to learn the generic attribution.

\paragraph{Backbone Encoder.}
We comprehensively consider three backbone encoders for during the pre-training of transferable explainer, including the \texttt{ViT-Base/Large}, \texttt{Swin-Base/Large}, \texttt{Deit-Base} transformers.
Their pre-trained weights are loaded from the HuggingFace library~\cite{wolf2020transformers}.
The hyper-parameter setting of \Algnameabbr{} pre-training is given in Table~\ref{tab:alg_hyperparam}.

% These have been implemented in HuggingFace framework~\cite{wolf2020transformers} into the classes $\texttt{ViTModel}$, $\texttt{SwinModel}$, and $\texttt{DeiTModel}$, respectively.
% The hyper-parameter settings of \Algnameabbr{} pre-training on the ImageNet dataset are given in Table~\ref{tab:alg_hyperparam}.

\begin{table}[H]
    \centering
    \caption{Hyper-parameters of \Algnameabbr{} pre-training on the ImageNet dataset.}
    \begin{tabular}{lccc}
    \toprule
         Target Encoder & \texttt{ViT-Base} & \texttt{Swin-Base} & \texttt{DeiT-Base} \\
         % Datasets & {Census Income} & {German Credit} & {GAS Sensor} & {NATICUSdroid} \\
    \midrule
         Explainer Architecture & \multicolumn{3}{c}{Figure~\ref{fig:explainer-architecture}} \\
         Pixel \# per image $\mathrm{W} \times \mathrm{W}$ & \multicolumn{3}{c}{$224 \times 224$} \\
         Patch \# per image $\mathrm{P} \times \mathrm{P}$ & \multicolumn{3}{c}{$14 \times 14$} \\
         Pixel \# per patch $\mathrm{C} \times \mathrm{C}$ & \multicolumn{3}{c}{$16 \times 16$} \\
         Shape of $\textbf{\textsl{g}}_k$ and $\textbf{h}_k$ & $14 \times 14 \times 768$ & $14 \times 14 \times 1024$ & $14 \times 14 \times 768$ \\
         Optimizer & \multicolumn{3}{c}{ADAM} \\
         Learning rate & \multicolumn{3}{c}{$1\times 10^{-3}$} \\
         Mini-batch size & \multicolumn{3}{c}{64 per GPU $\times$ 4 GPUs} \\
         Scheduler & \multicolumn{3}{c}{CosineAnnealingLR} \\
         Warm-up-ratio & \multicolumn{3}{c}{0.05} \\
         Weight-decay & \multicolumn{3}{c}{0.05} \\
         Training steps & \multicolumn{3}{c}{$2 \times 10^{5}$} \\
         Neighbor patches & \multicolumn{3}{c}{0-, 1-, 2-hop neighbor patches} \\
    \bottomrule     
    \end{tabular}
    \label{tab:alg_hyperparam}
\end{table}

% To demonstrate the effectiveness of our proposed framework on a wide range of vision tasks, we consider three architectures of vision transformers as the target models to be explained, including

% \section{Hyper-parameter Setting of \Algnameabbr{} Pre-training}
% We give the hyper-parameters of fine-tuning the target models in Table~\ref{tab:alg_hyperparam}.

% \section{Hyper-parameter Setting of target model Fine-tuning}
% \label{appendix:target-model-hyperparam}

% We give the hyper-parameters of fine-tuning the target models in Table~\ref{tab:target_model_hyperparam}.

\begin{figure}[t]
\centering
\subfigure[]{
\includegraphics[width=0.45\linewidth]{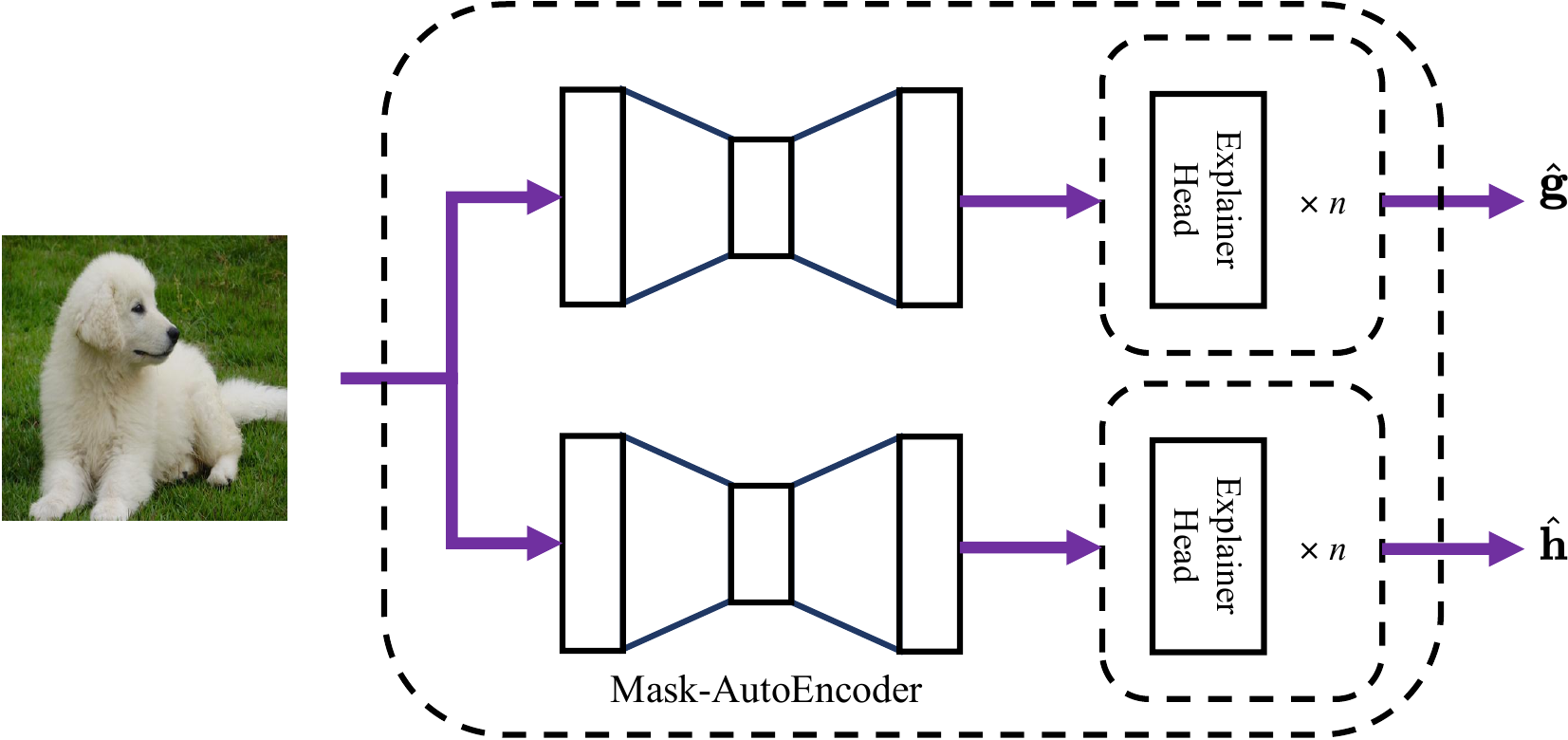}
}
\subfigure[]{
\raisebox{0.2\height}{\includegraphics[width=0.45\linewidth]{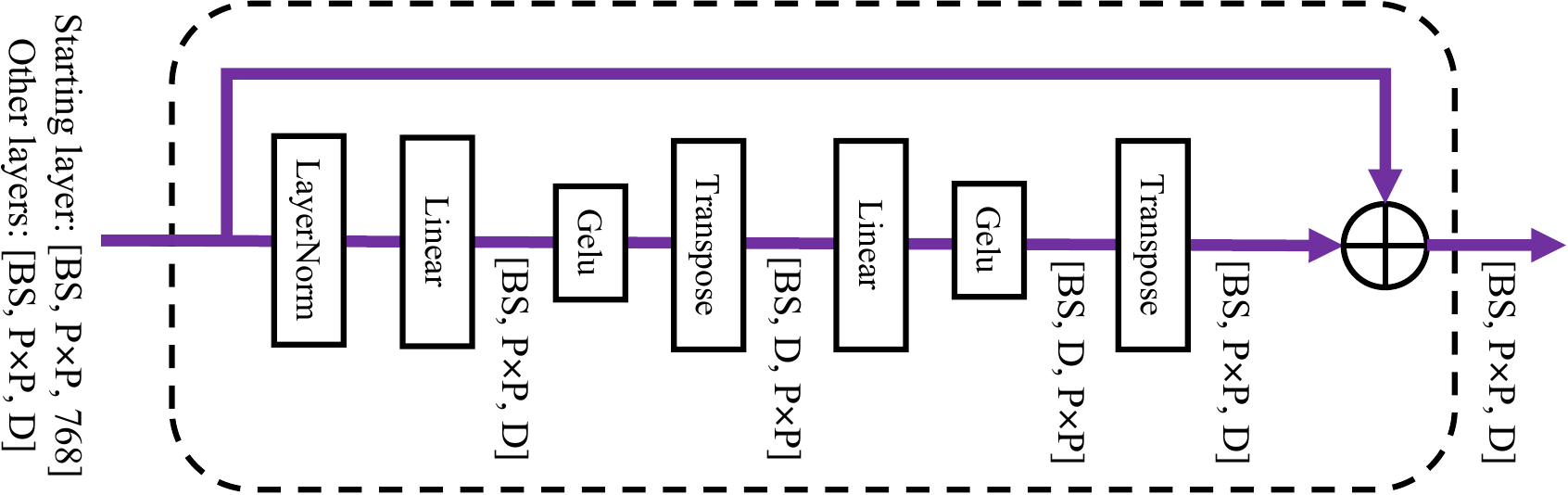}}
}
\caption{Transferable explainer architecture. (a) Explainer architecture. (b) FFN-layers for the explainer head.}
\label{fig:explainer-architecture}
\end{figure}

\begin{figure}[t]
    \centering
    \includegraphics[width=0.85\linewidth]{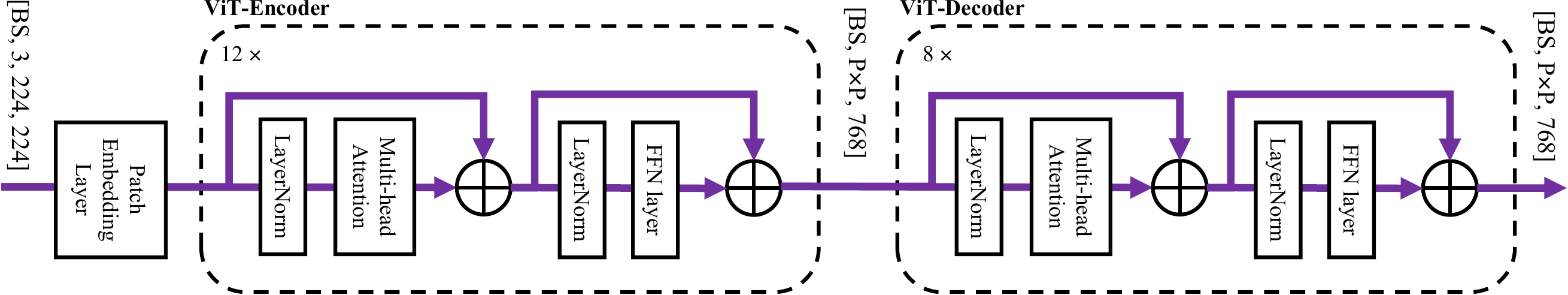}
    \caption{Structure of Mask-Autoencoder.}
    \label{fig:mae-architecture}
\end{figure}

\begin{figure}[H]
\flushright
\subfigure[]{
\includegraphics[width=0.15\linewidth]{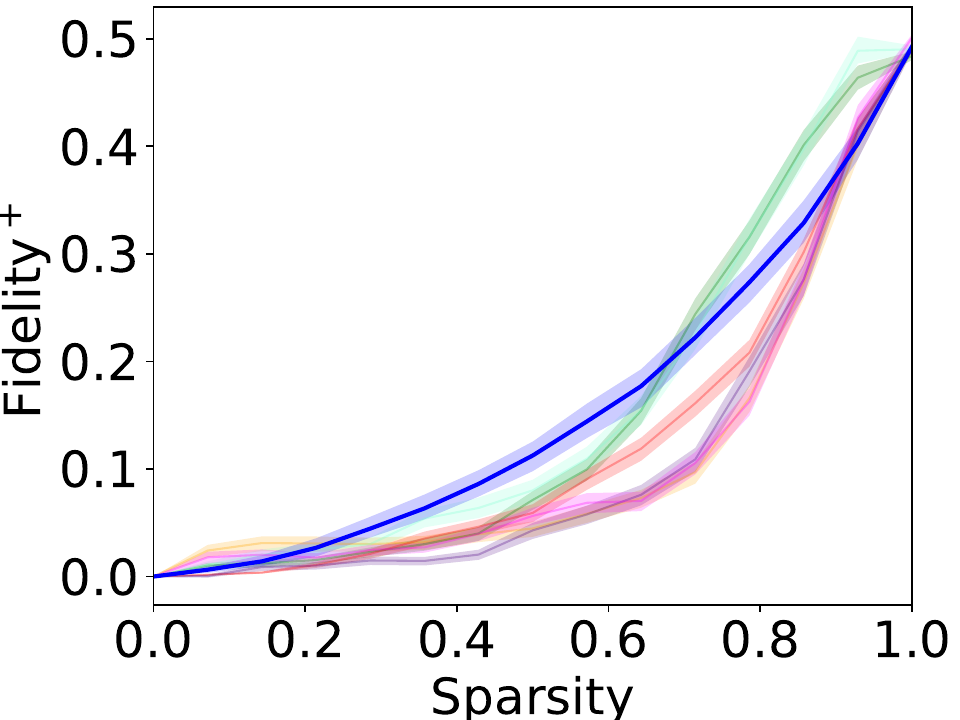}
}
\!\!\!\!
\subfigure[]{
\includegraphics[width=0.15\linewidth]{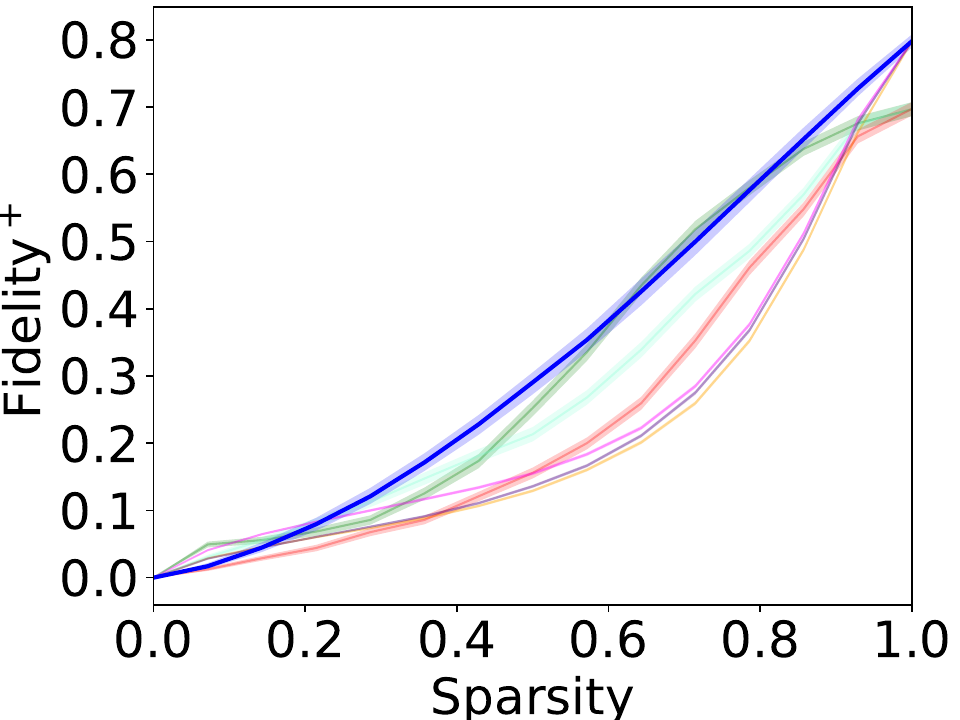}
}
\!\!\!\!
\subfigure[]{
\includegraphics[width=0.15\linewidth]{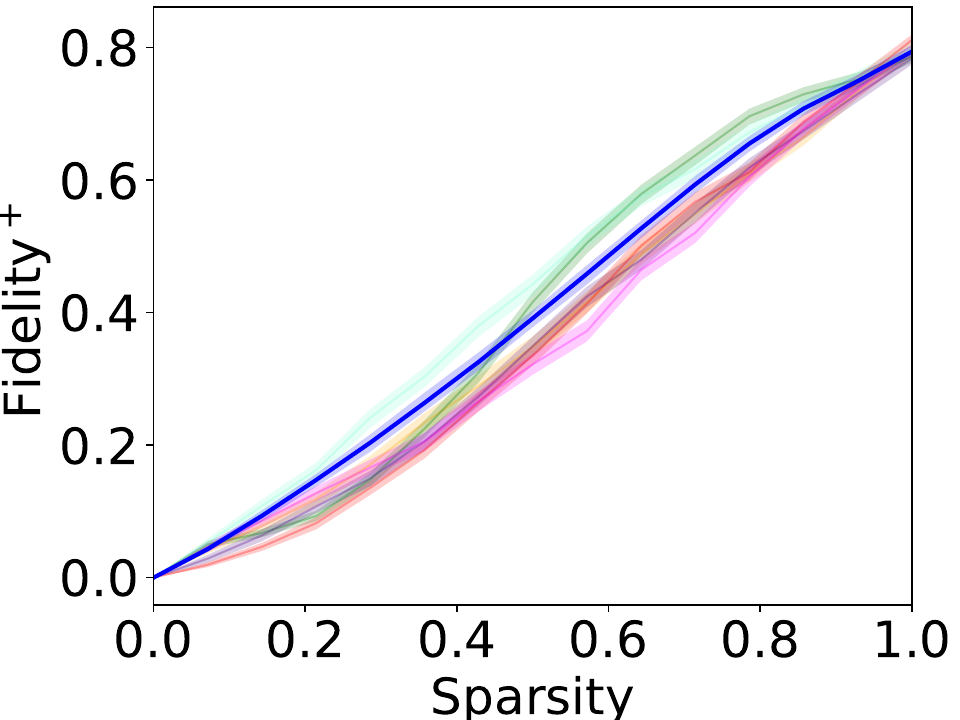}
}
\!\!\!\!
\subfigure[]{
\includegraphics[width=0.15\linewidth]{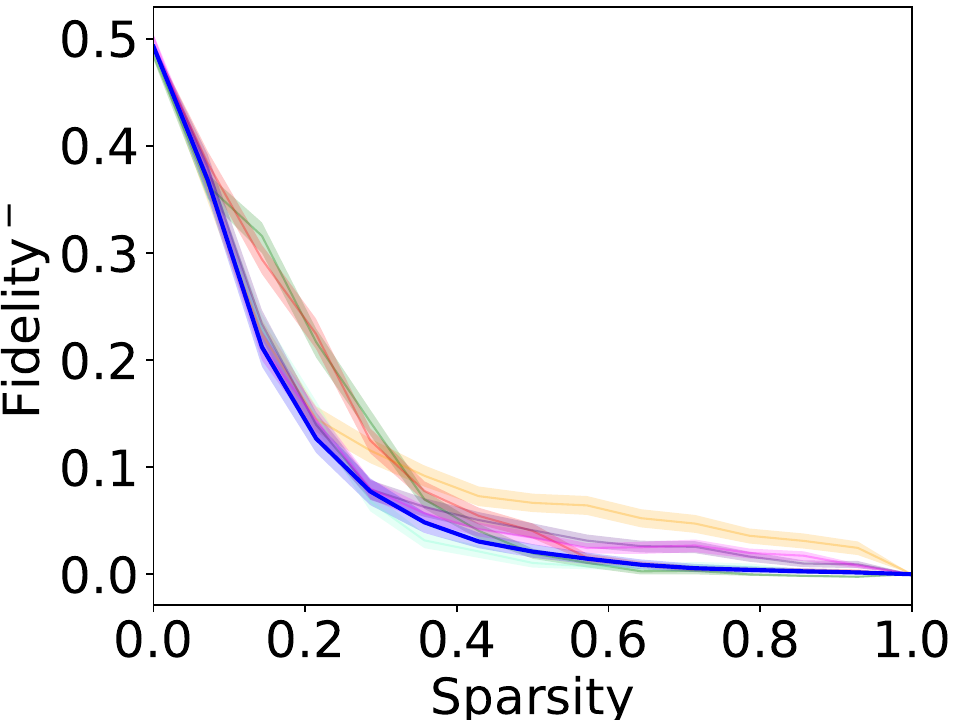}
}
\!\!\!\!
\subfigure[]{
\includegraphics[width=0.15\linewidth]{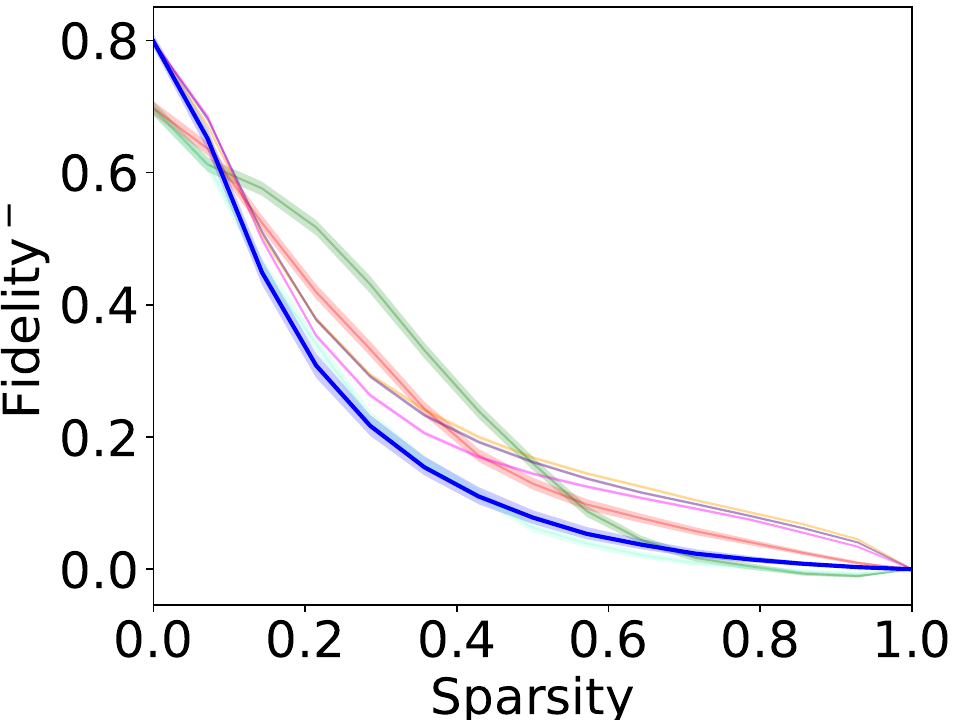}
}
\!\!\!\!
\subfigure[]{
\includegraphics[width=0.15\linewidth]{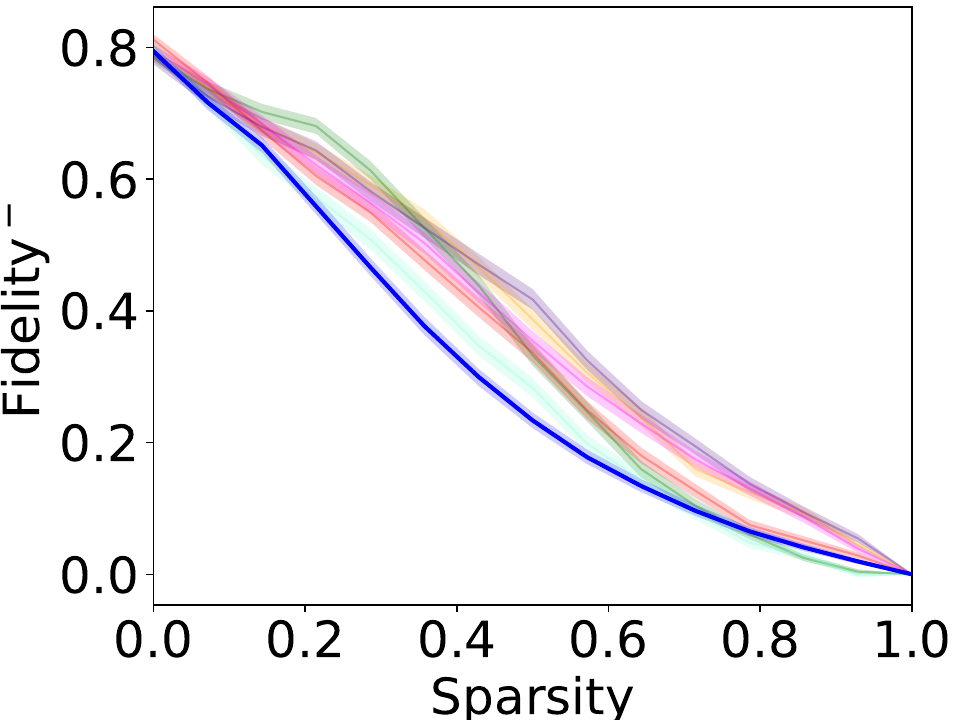}
}
\!\!\!\!
\subfigure[]{
\includegraphics[width=0.15\linewidth]{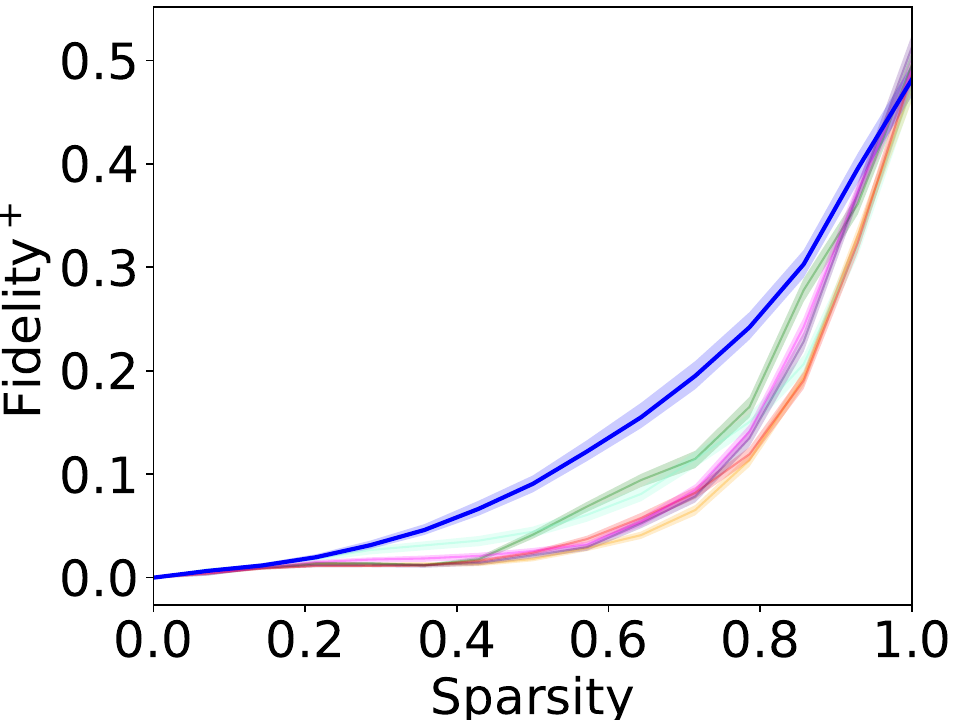}
}
\!\!\!\!
\subfigure[]{
\includegraphics[width=0.15\linewidth]{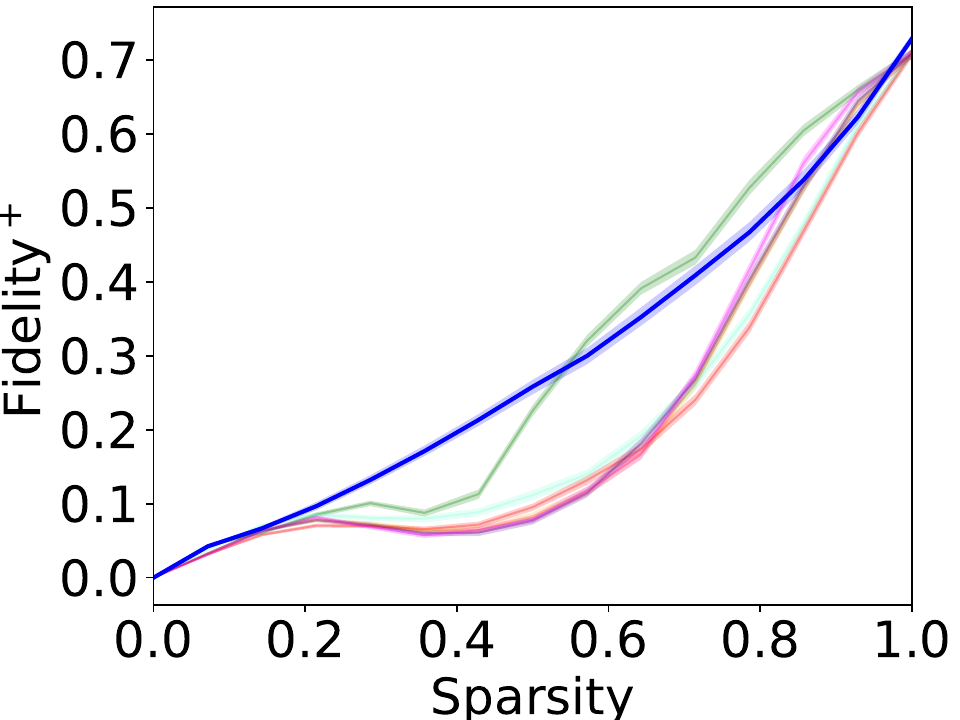}
}
\!\!\!\!
\subfigure[]{
\includegraphics[width=0.15\linewidth]{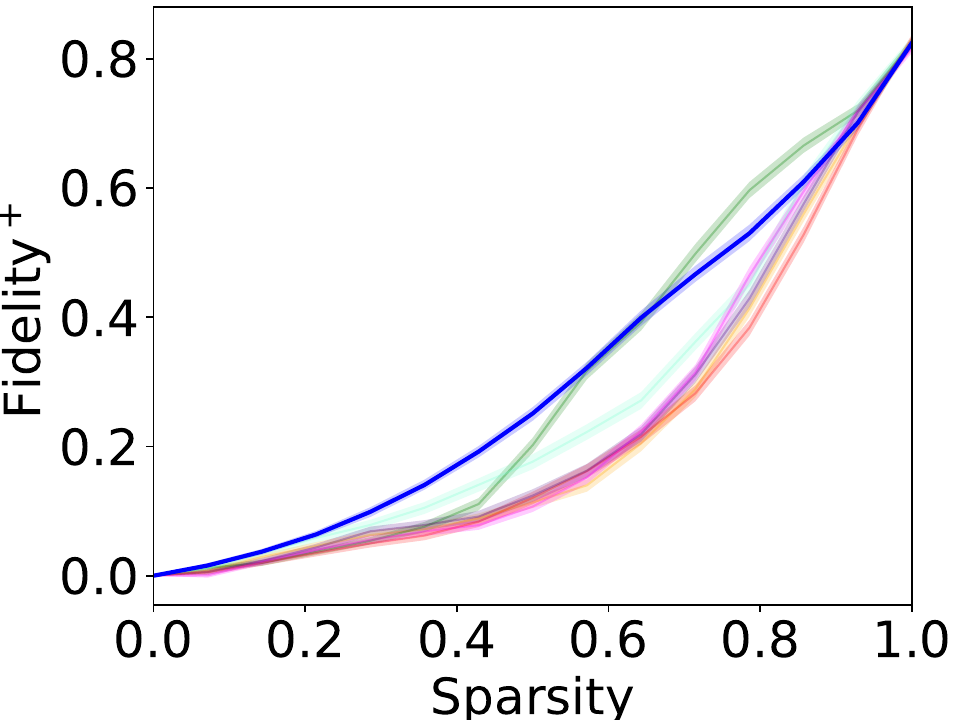}
}
\!\!\!\!
\subfigure[]{
\includegraphics[width=0.15\linewidth]{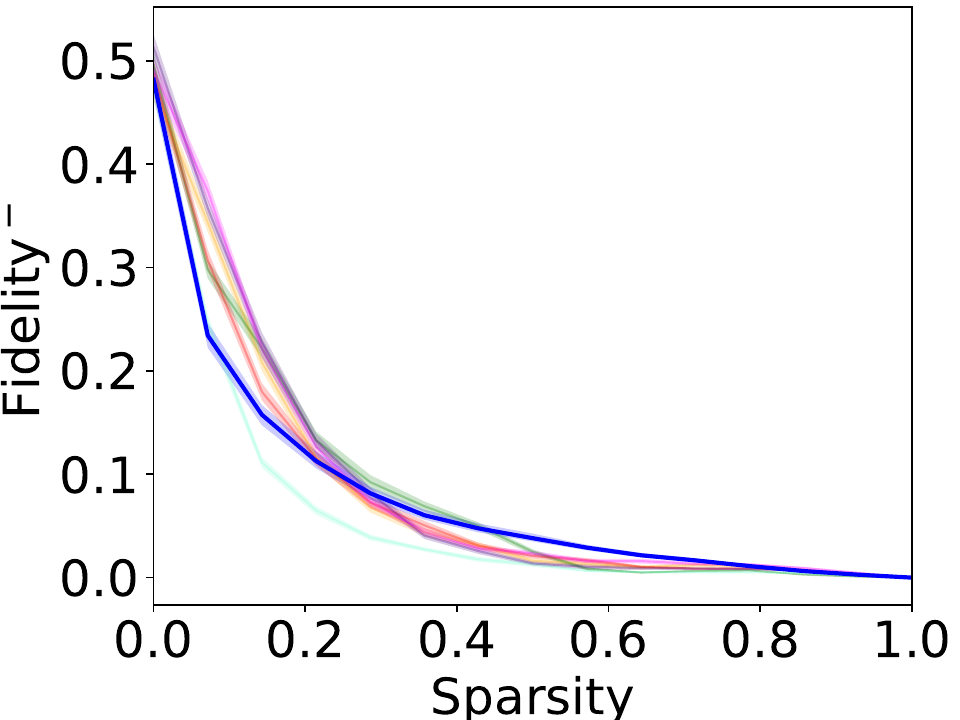}
}
\!\!\!\!
\subfigure[]{
\includegraphics[width=0.15\linewidth]{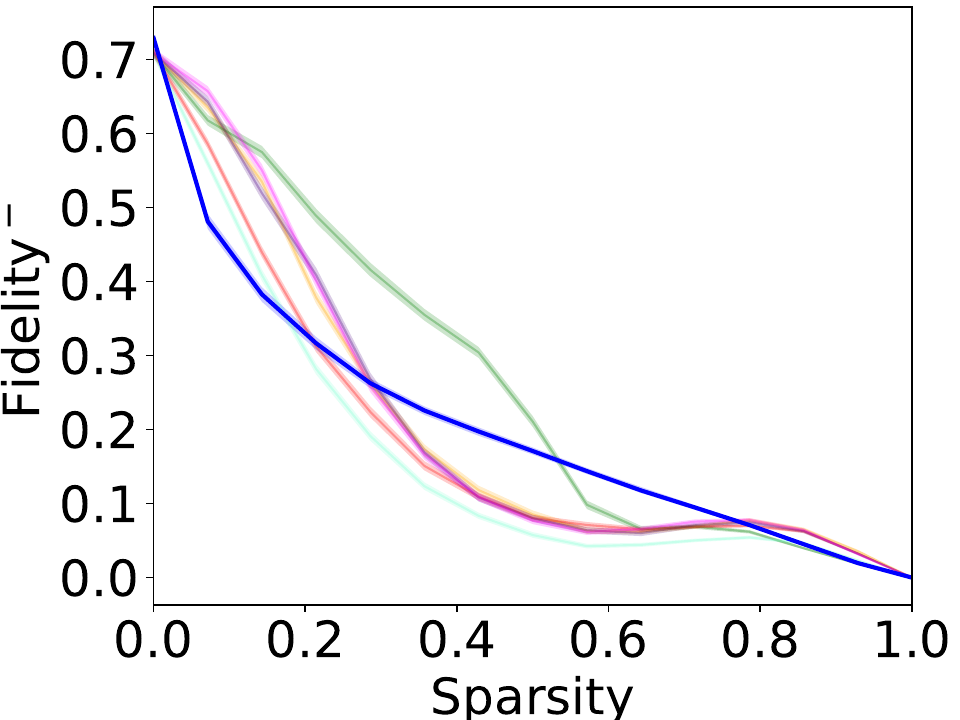}
}
\!\!\!\!
\subfigure[]{
\includegraphics[width=0.15\linewidth]{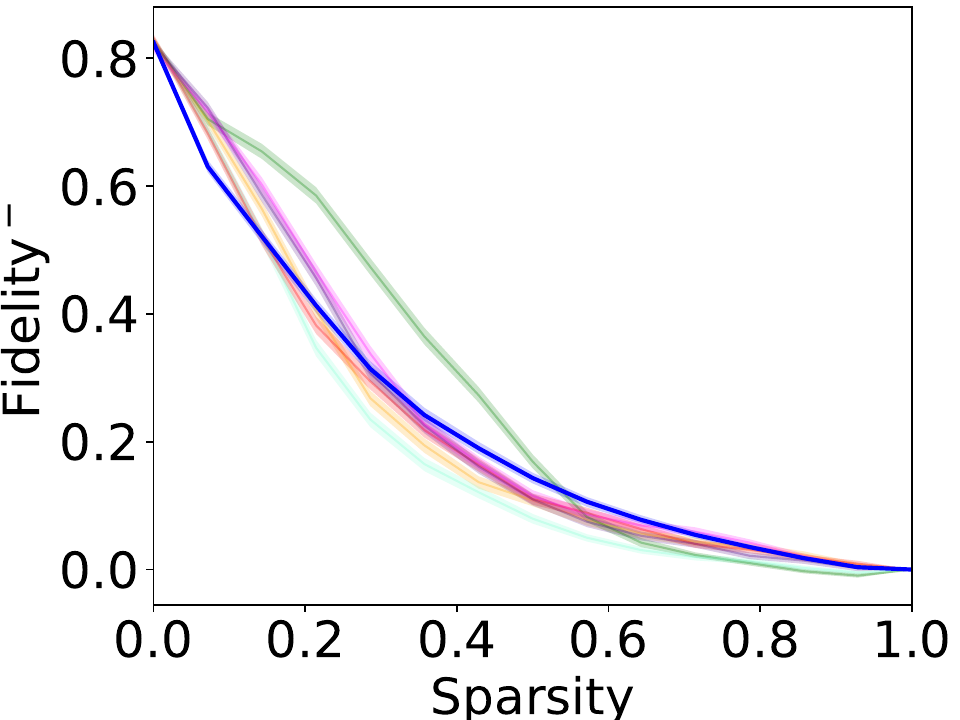}
}
\!\!\!\!
\subfigure[]{
\includegraphics[width=0.15\linewidth]{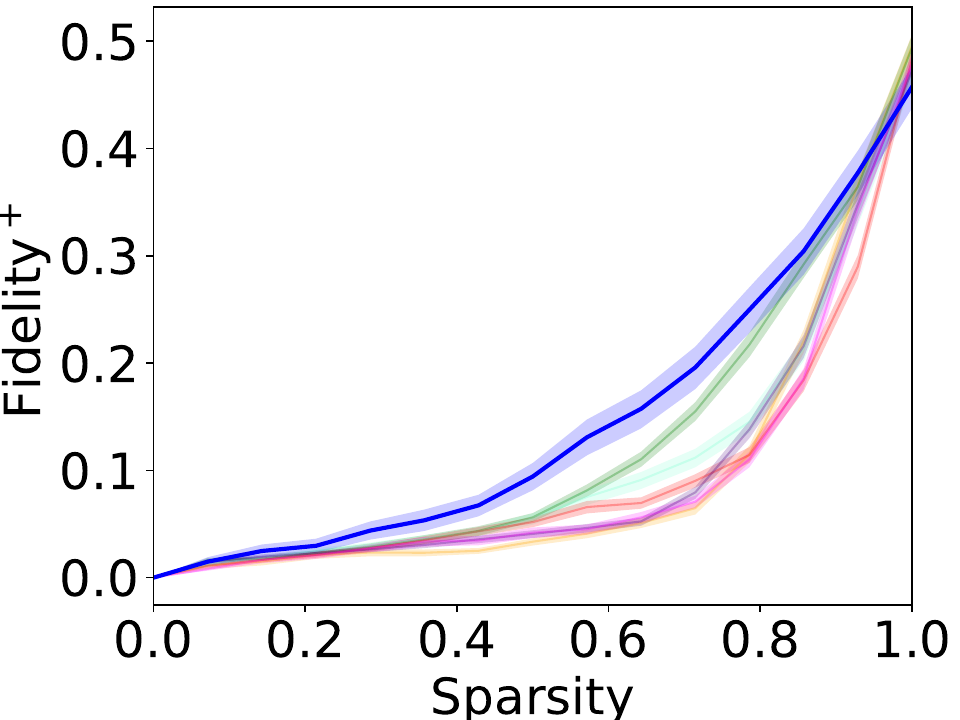}
}
\!\!\!\!
\subfigure[]{
\includegraphics[width=0.15\linewidth]{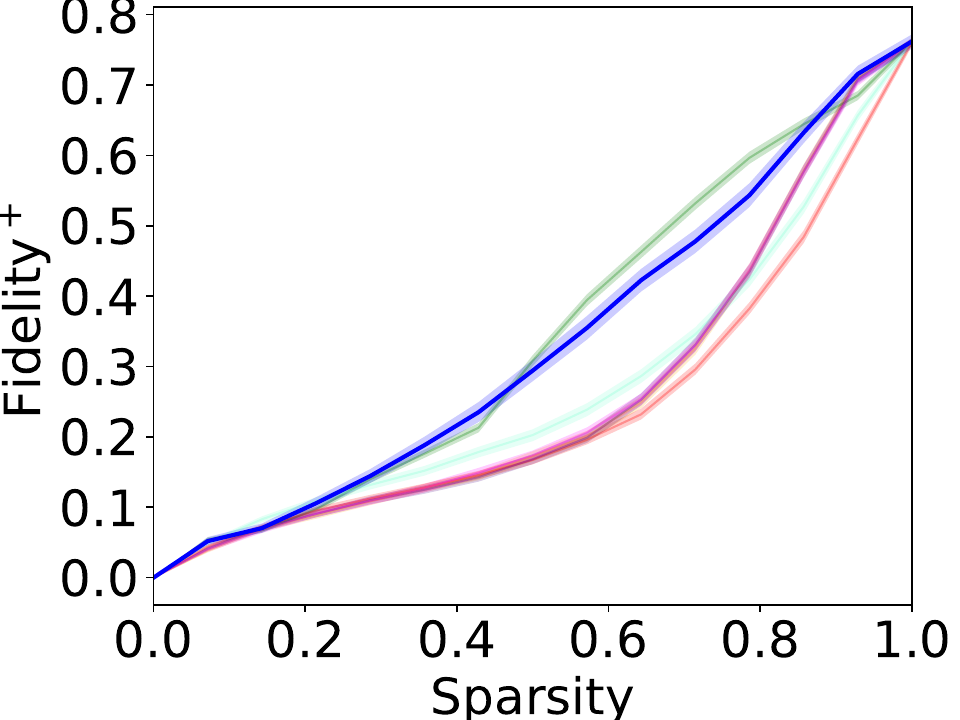}
}
\!\!\!\!
\subfigure[]{
\includegraphics[width=0.15\linewidth]{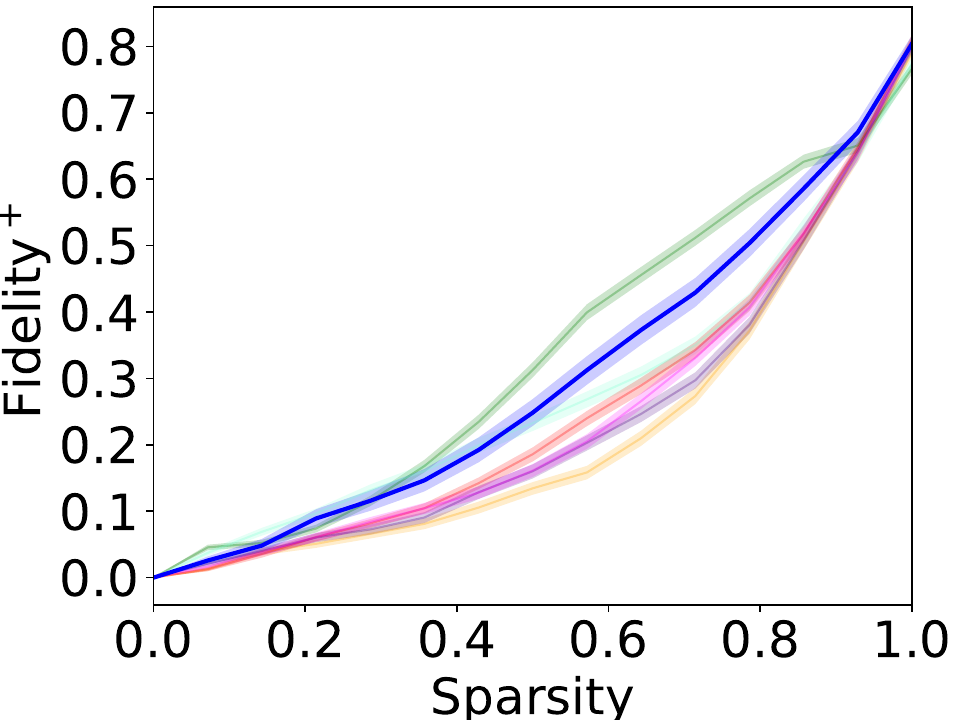}
}
\!\!\!\!
\subfigure[]{
\includegraphics[width=0.15\linewidth]{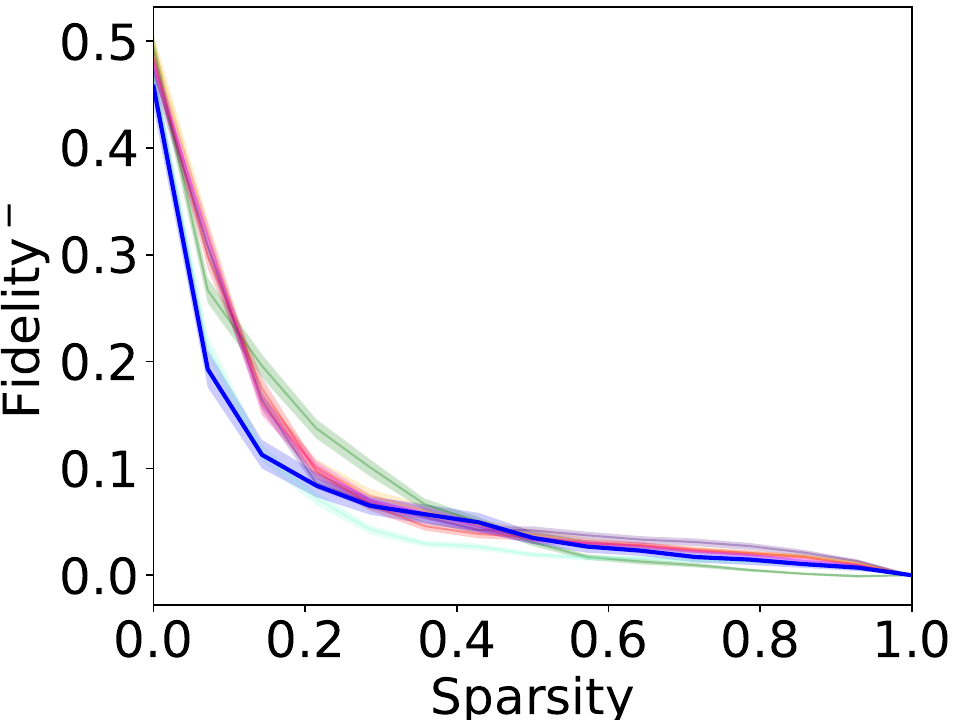}
}
\!\!\!\!
\subfigure[]{
\includegraphics[width=0.15\linewidth]{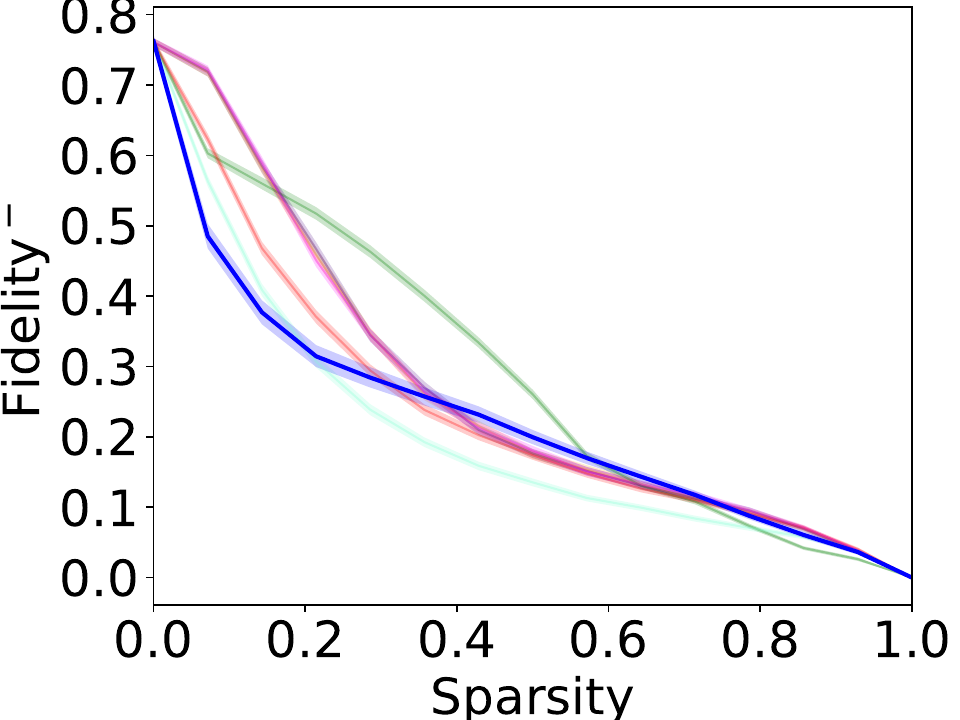}
}
\!\!\!\!
\subfigure[]{
\includegraphics[width=0.15\linewidth]{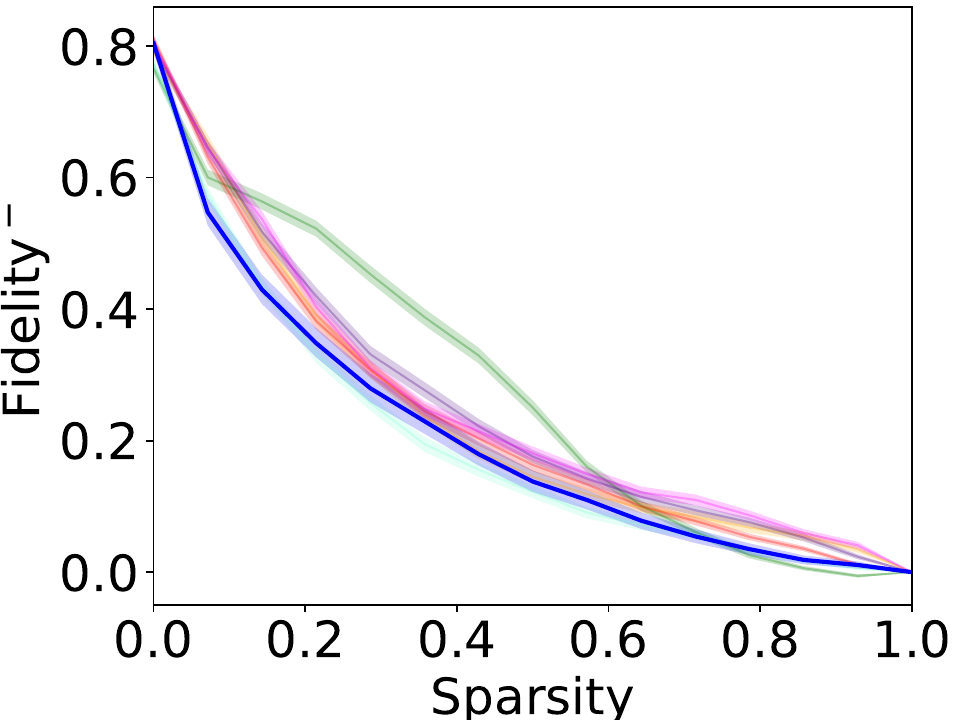}
}
\!\!\!\!
\includegraphics[width=0.8\linewidth]{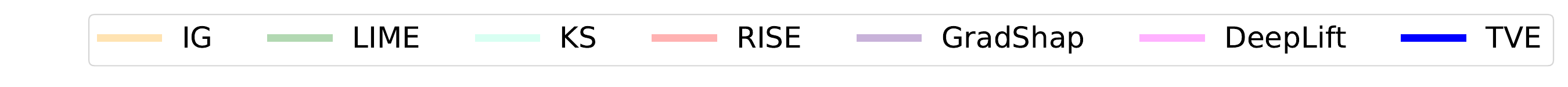}

\caption{\label{fig:fidelity_curve} 
$\mathrm{Fidelity}^+$-sparsity curve for explaining \texttt{ViT-Base} on \texttt{Cats-vs-dogs}~(a), \texttt{Imagenette}~(a), and \texttt{CIFAR-10}~(c).
$\mathrm{Fidelity}^-$-sparsity curve of \texttt{ViT-Base} on \texttt{Cats-vs-dogs}~(d), \texttt{Imagenette}~(e), and \texttt{CIFAR-10}~(f).
$\mathrm{Fidelity}^+$-sparsity curve of \texttt{Swin-Base} on \texttt{Cats-vs-dogs}~(g), \texttt{Imagenette}~(h), and \texttt{CIFAR-10}~(i).
$\mathrm{Fidelity}^-$-sparsity curve of \texttt{Swin-Base} on \texttt{Cats-vs-dogs}~(j), \texttt{Imagenette}~(k), and \texttt{CIFAR-10}~(l).
$\mathrm{Fidelity}^+$-sparsity curve of \texttt{Deit-Base} on \texttt{Cats-vs-dogs}~(m), \texttt{Imagenette}~(n), and \texttt{CIFAR-10}~(o).
$\mathrm{Fidelity}^-$-sparsity curve of \texttt{Deit-Base} on \texttt{Cats-vs-dogs}~(p), \texttt{Imagenette}~(q), and \texttt{CIFAR-10}~(r).
}
\end{figure}

\section{Fidelity-Sparsity Curve of Section~\ref{sec:fidelity_eval}}
\label{appendix:fidelity-sparsity-curve}

We show the fidelity-sparsity curve for explaining \texttt{ViT-Base}, \texttt{Swin-Base}, and \texttt{Deit-Base} on the \texttt{Cats-vs-dogs}, \texttt{Imagenette}, and \texttt{CIFAR-10} datasets in Figures~\ref{fig:fidelity_curve}~(a)-(r).
It is observed that \Algnameabbr{} consistently exhibits promising performance in terms of both $\mathrm{Fidelity}^+$($\uparrow$) and $\mathrm{Fidelity}^-$($\downarrow$), surpassing the majority of baseline methods. 
This indicates \Algnameabbr{}'s ability to faithfully explain various downstream tasks.

\section{Computational Infrastructure}
\label{appendix:hardware}

The computational infrastructure information is given in Table~\ref{tab:computing_infrastructure}.

\begin{table}[H]
\centering
\caption{Computing infrastructure for the experiments.}
\begin{tabular}{l|c}
\toprule
Device Attribute & Value \\
\hline
Computing Infrastructure & GPU \\
GPU Model & NVIDIA-A5000 \\ %, NVIDIA-A5000 \\
GPU Memory & 24564MB \\ %, 24564MB \\
GPU Number & 8 \\
CUDA Version & 12.1 \\
% CPU model & \\
CPU Memory & 512GB \\
\bottomrule
\end{tabular}
\label{tab:computing_infrastructure}
\end{table}

\clearpage

\section{More Case Studies}
\label{appendix:case_study}

We give more explanation heatmaps of \texttt{ViT-Base} on the \texttt{ImageNet} dataset in Figure~\ref{fig:more_case_study_vit_imagenet-1}, which are generated by \Algnameabbr{}.

\begin{figure}[H]
\centering
\includegraphics[width=0.13\linewidth]{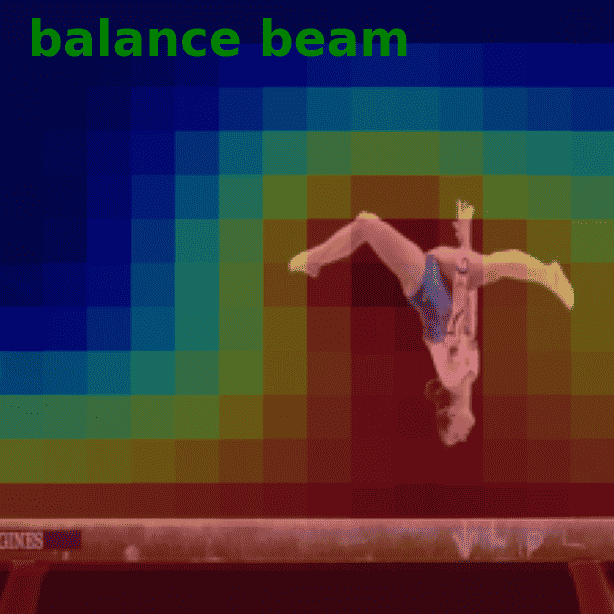}
\includegraphics[width=0.13\linewidth]{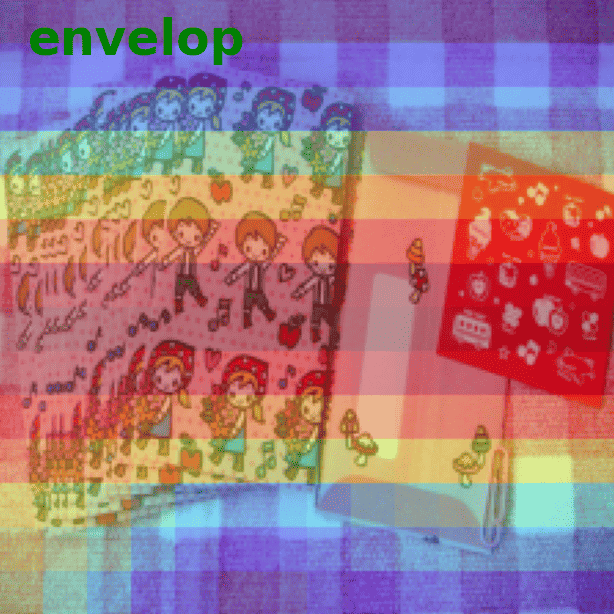}
\includegraphics[width=0.13\linewidth]{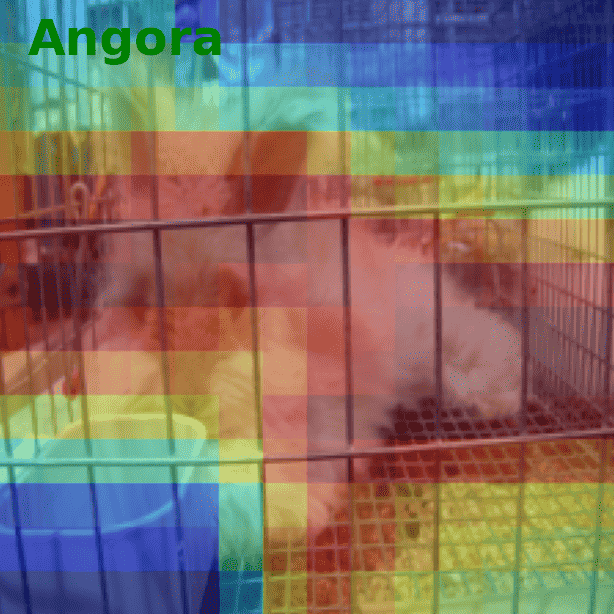}
\includegraphics[width=0.13\linewidth]{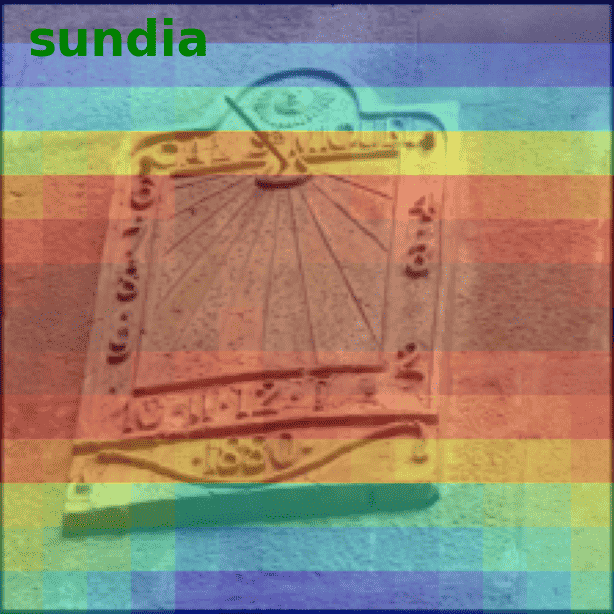}
\includegraphics[width=0.13\linewidth]{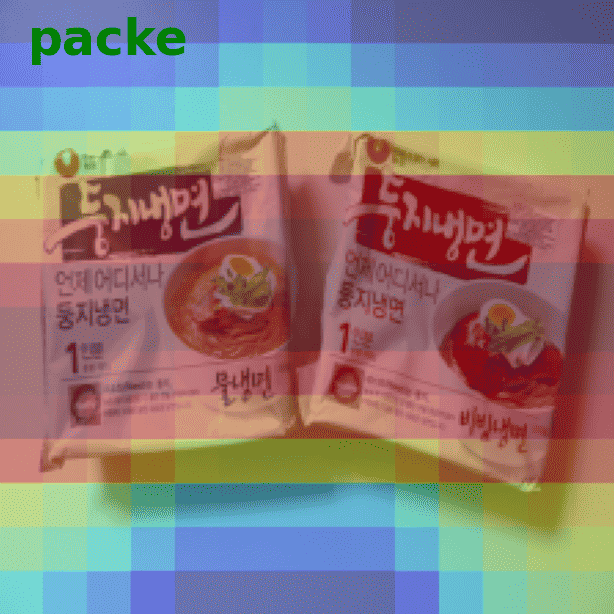}
\includegraphics[width=0.13\linewidth]{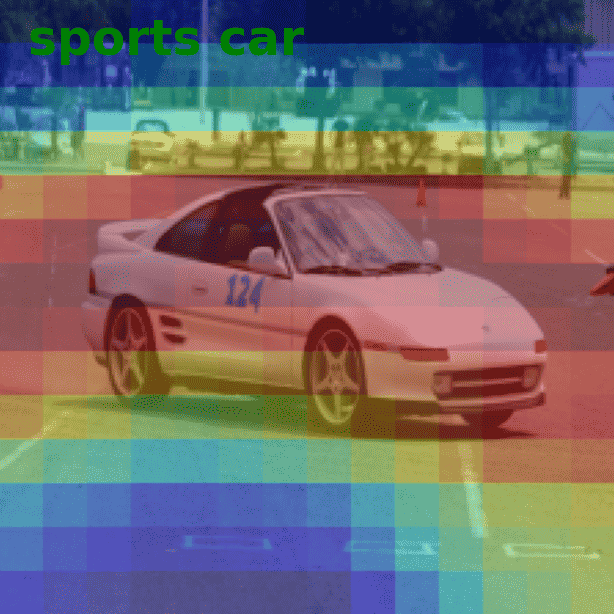}
\\
\includegraphics[width=0.13\linewidth]{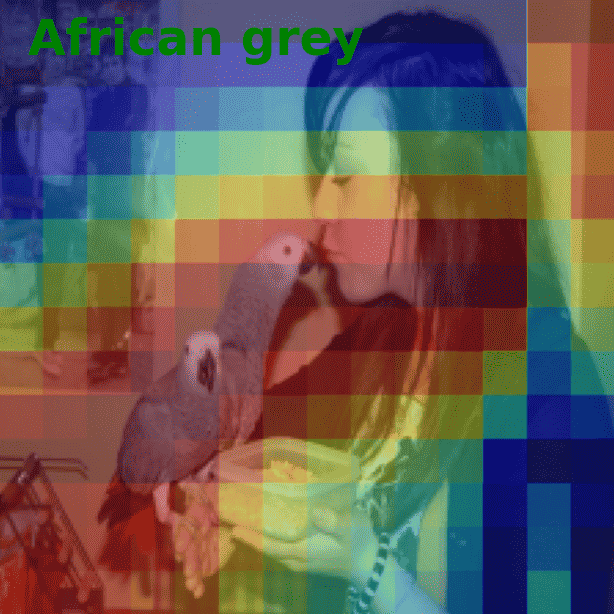}
\includegraphics[width=0.13\linewidth]{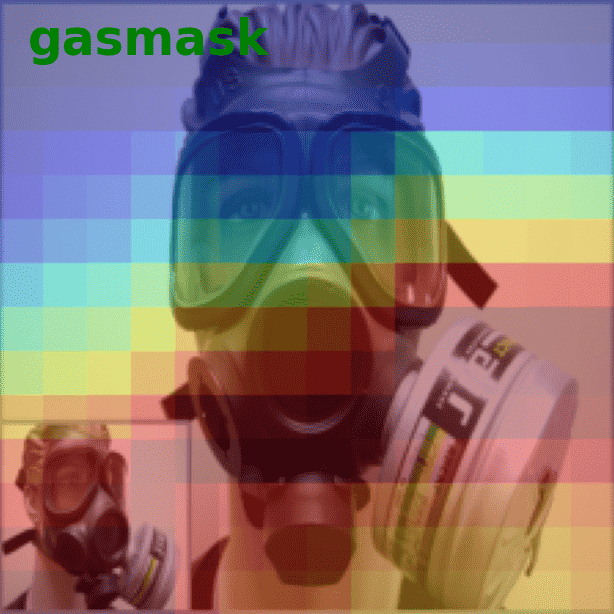}
\includegraphics[width=0.13\linewidth]{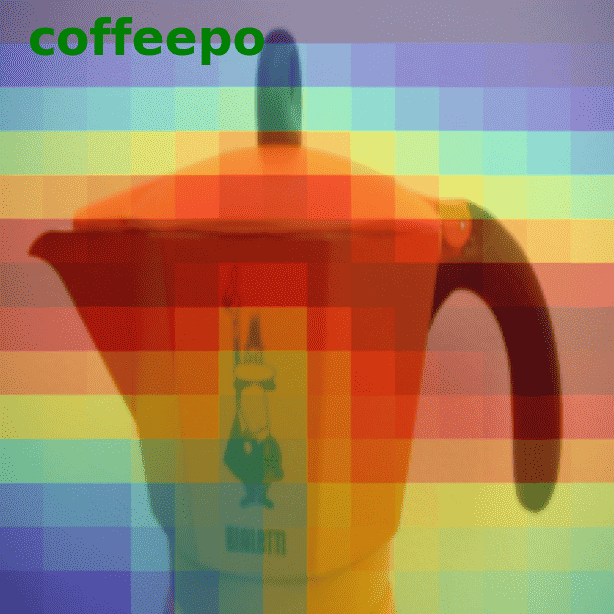}
\includegraphics[width=0.13\linewidth]{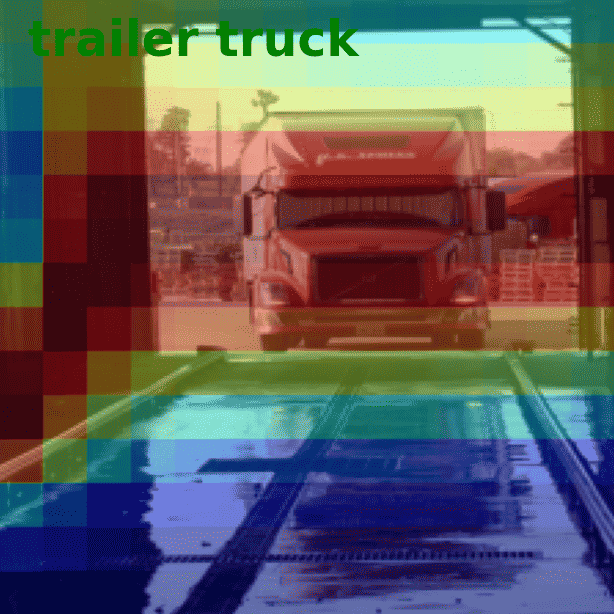}
\includegraphics[width=0.13\linewidth]{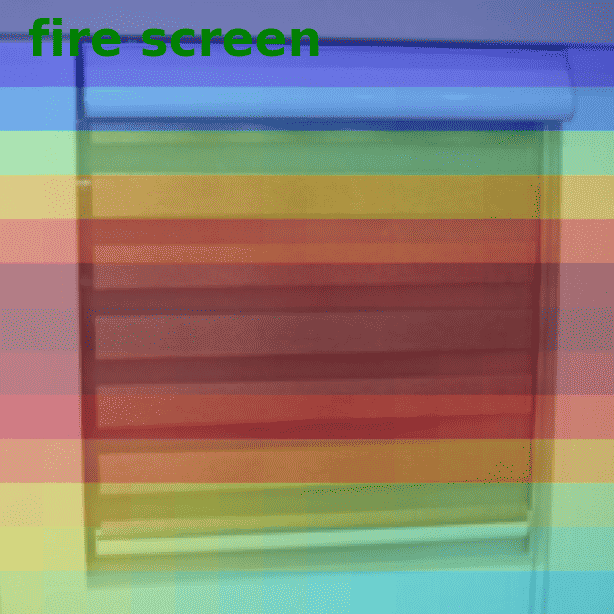}
\includegraphics[width=0.13\linewidth]{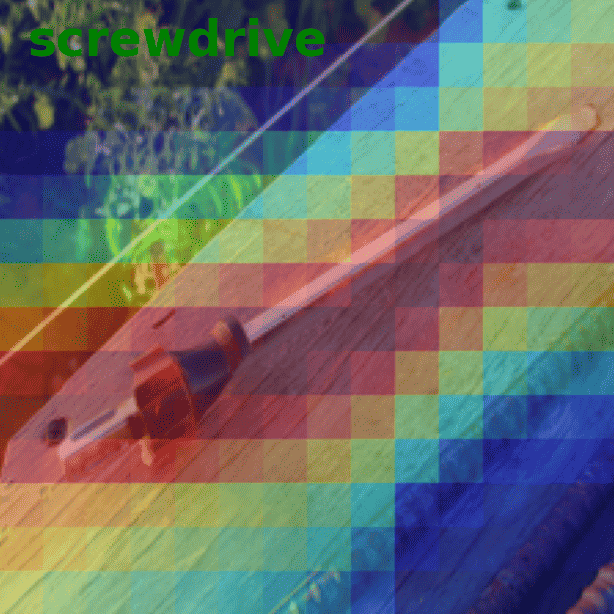}
\\
\includegraphics[width=0.13\linewidth]{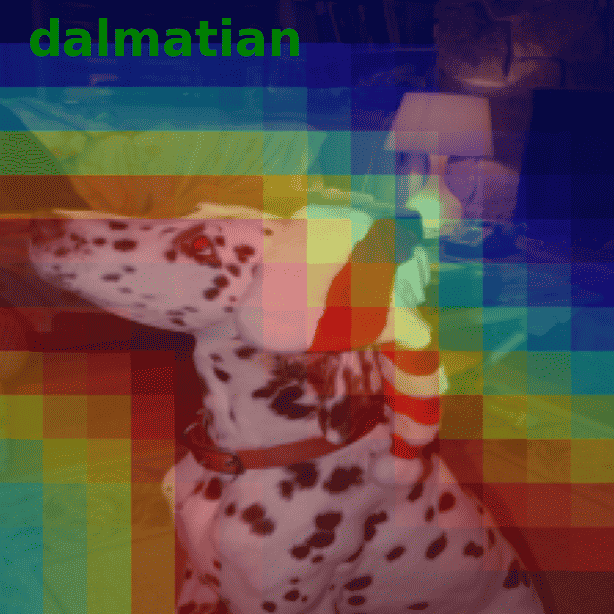}
\includegraphics[width=0.13\linewidth]{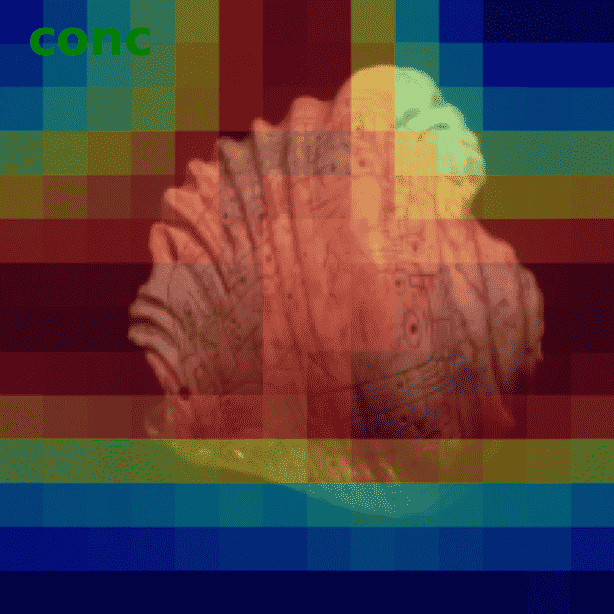}
\includegraphics[width=0.13\linewidth]{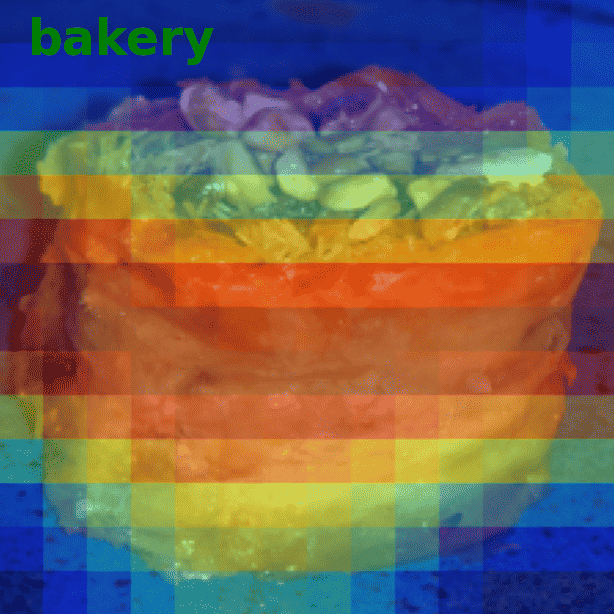}
\includegraphics[width=0.13\linewidth]{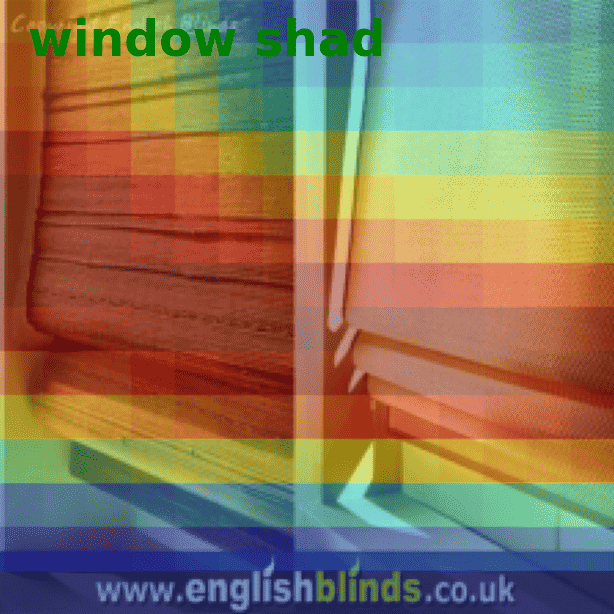}
\includegraphics[width=0.13\linewidth]{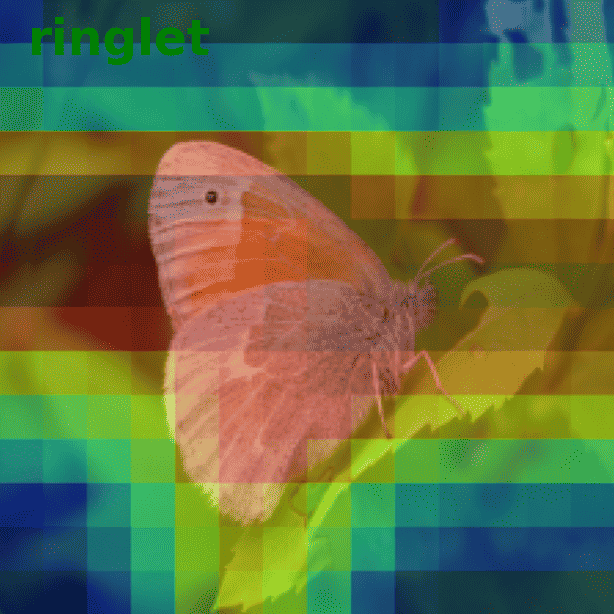}
\includegraphics[width=0.13\linewidth]{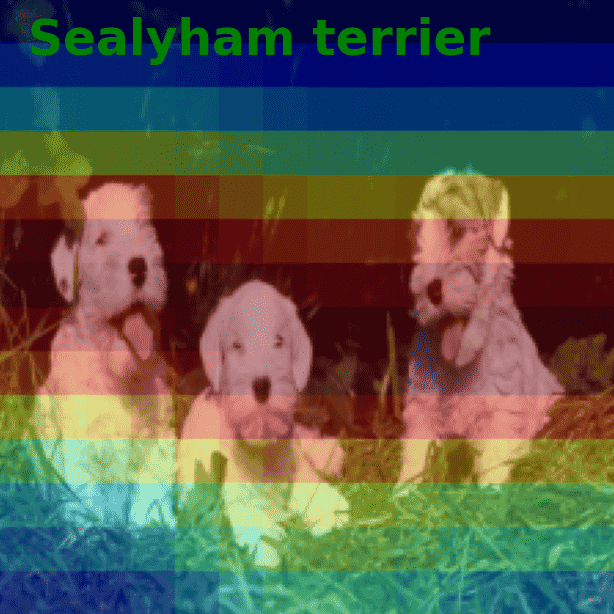}
\\
\includegraphics[width=0.13\linewidth]{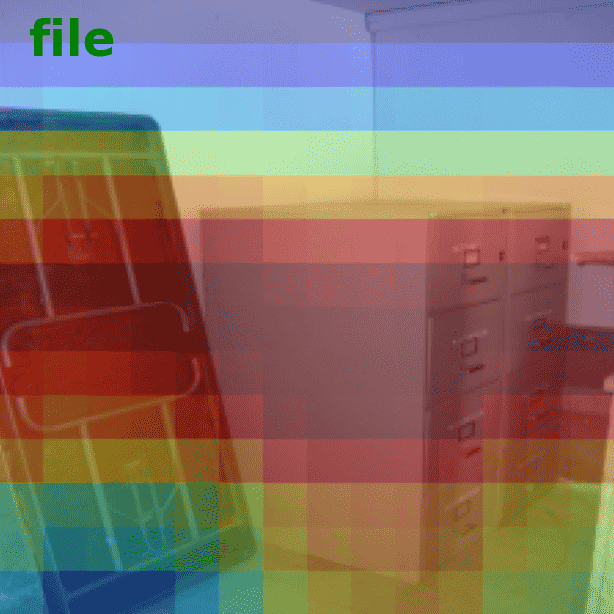}
\includegraphics[width=0.13\linewidth]{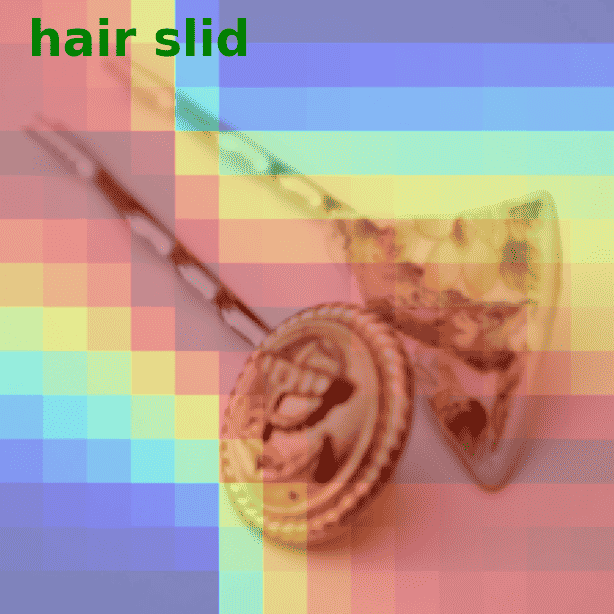}
\includegraphics[width=0.13\linewidth]{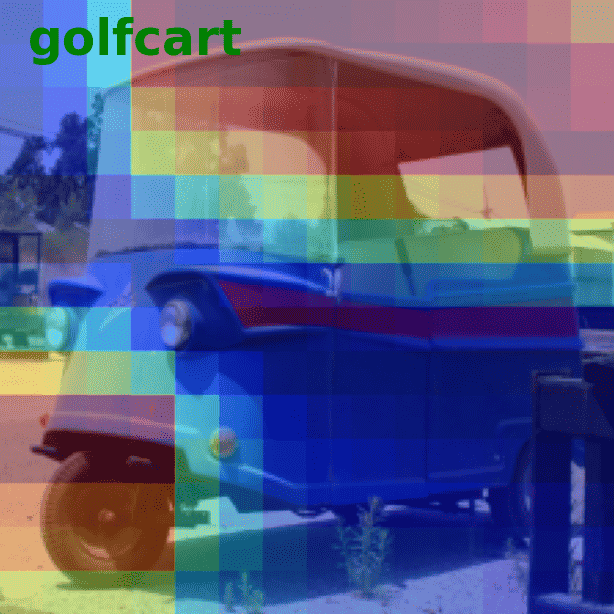}
\includegraphics[width=0.13\linewidth]{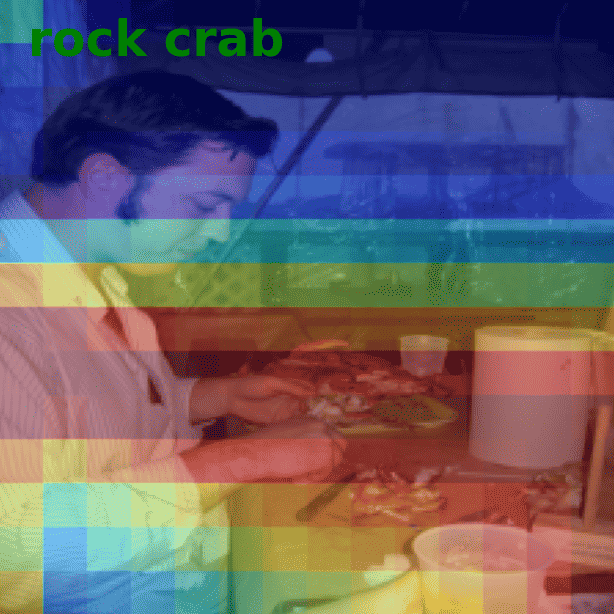}
\includegraphics[width=0.13\linewidth]{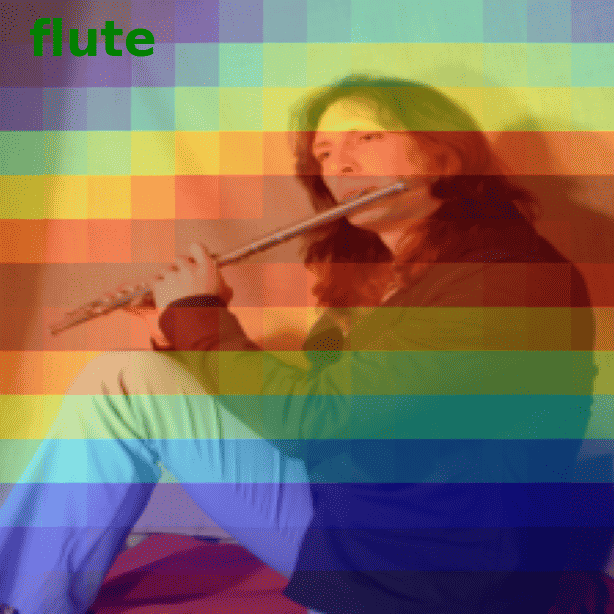}
\includegraphics[width=0.13\linewidth]{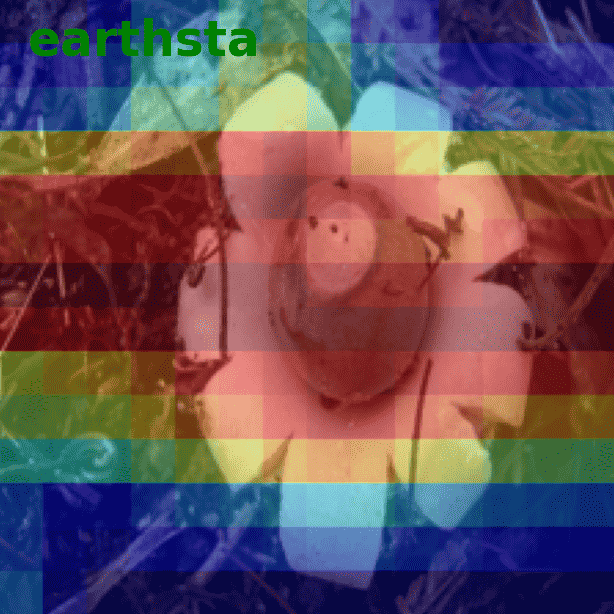}
\\
\includegraphics[width=0.13\linewidth]{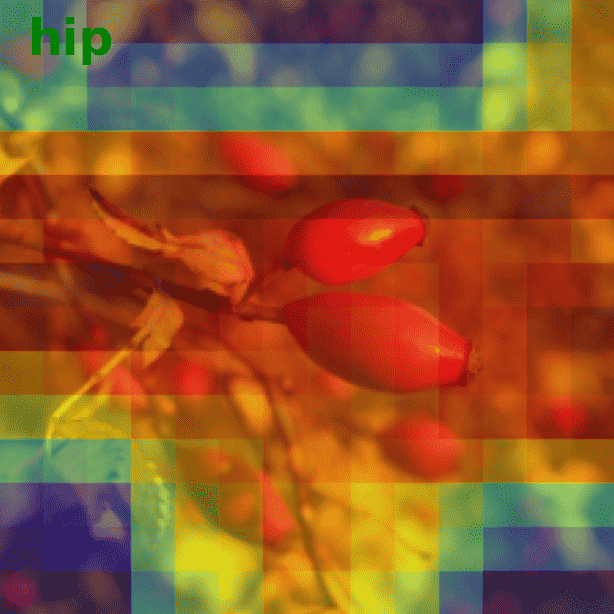}
\includegraphics[width=0.13\linewidth]{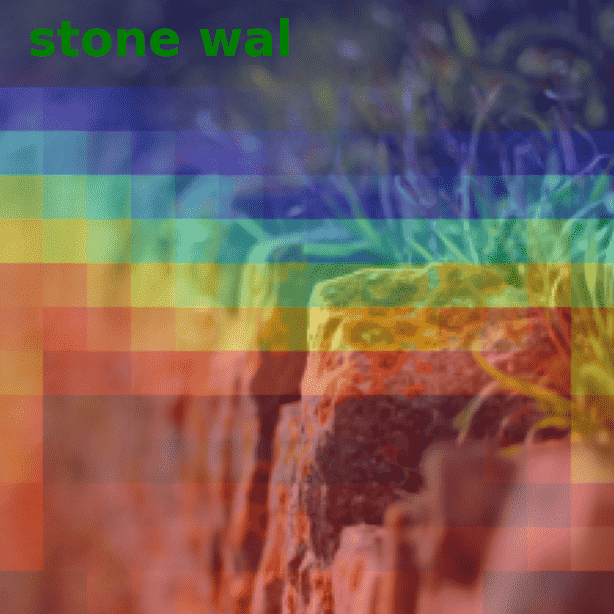}
\includegraphics[width=0.13\linewidth]{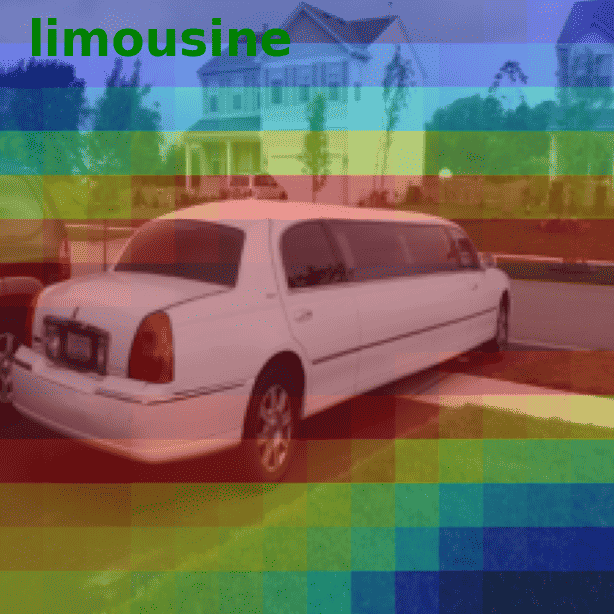}
\includegraphics[width=0.13\linewidth]{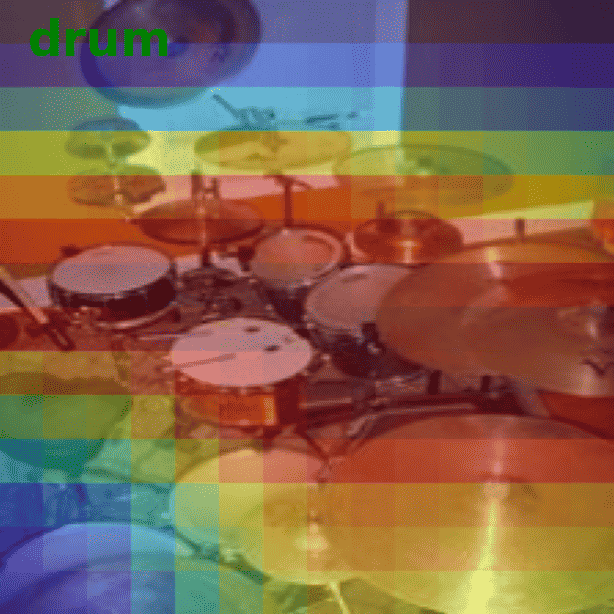}
\includegraphics[width=0.13\linewidth]{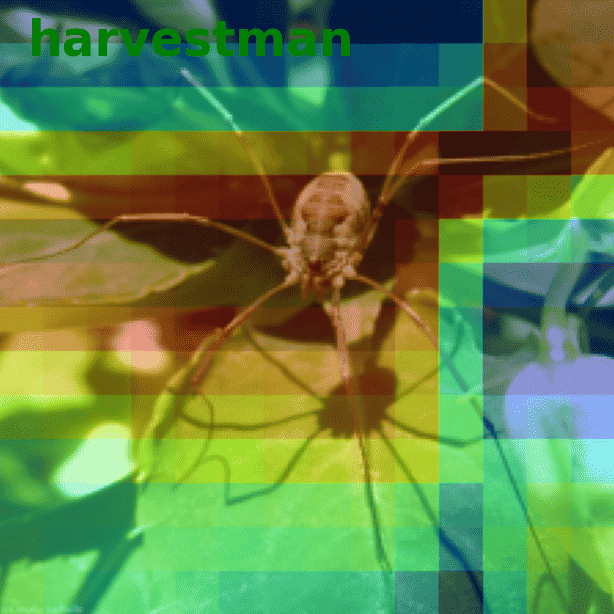}
\includegraphics[width=0.13\linewidth]{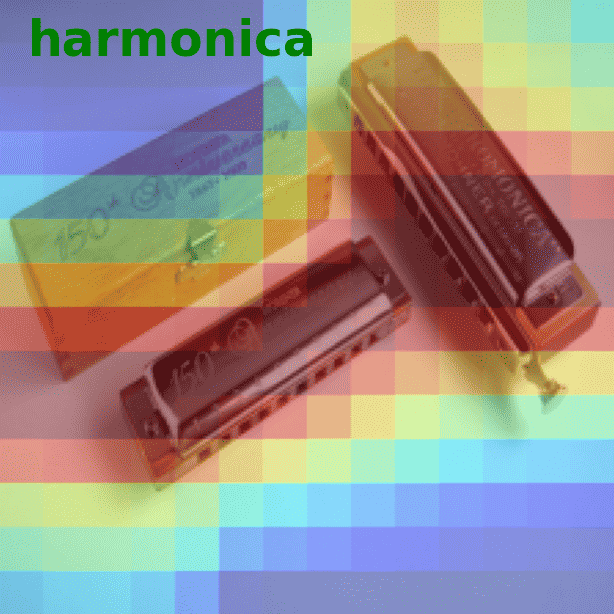}
\\
\includegraphics[width=0.13\linewidth]{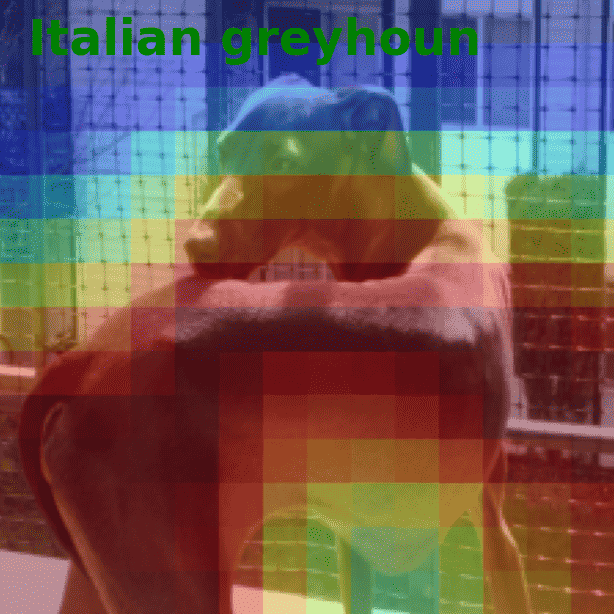}
\includegraphics[width=0.13\linewidth]{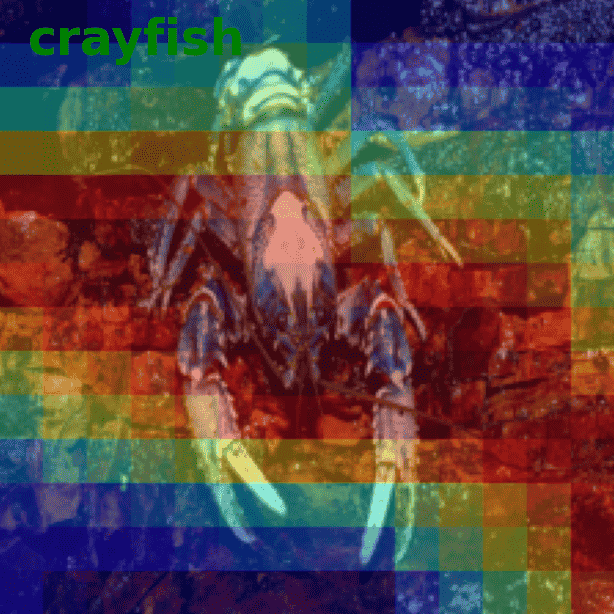}
\includegraphics[width=0.13\linewidth]{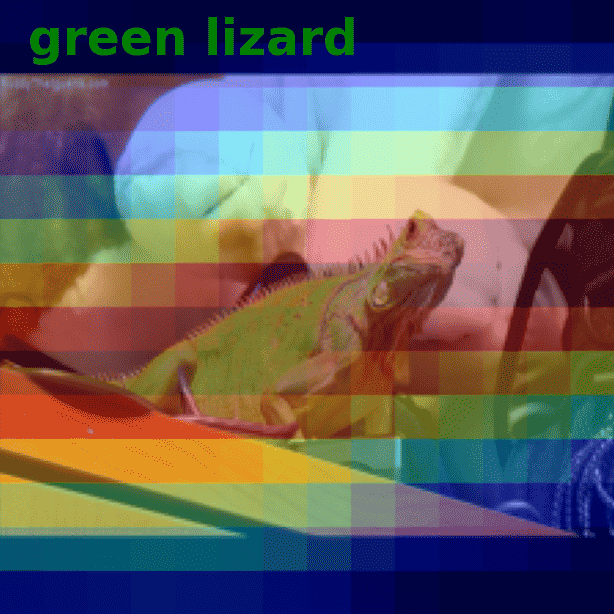}
\includegraphics[width=0.13\linewidth]{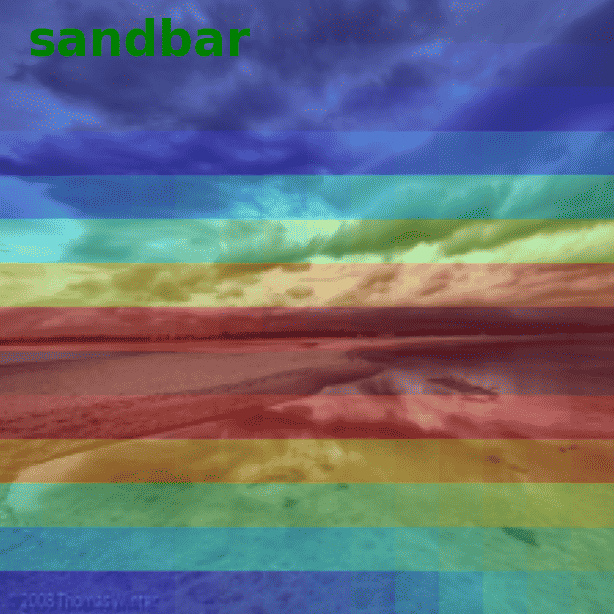}
\includegraphics[width=0.13\linewidth]{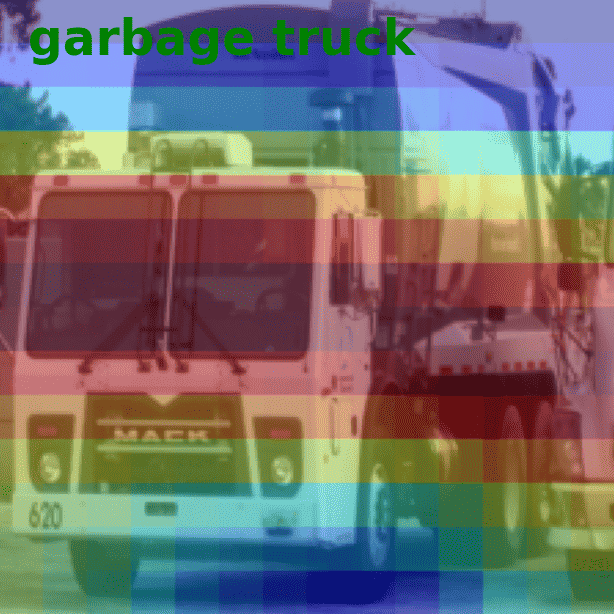}
\includegraphics[width=0.13\linewidth]{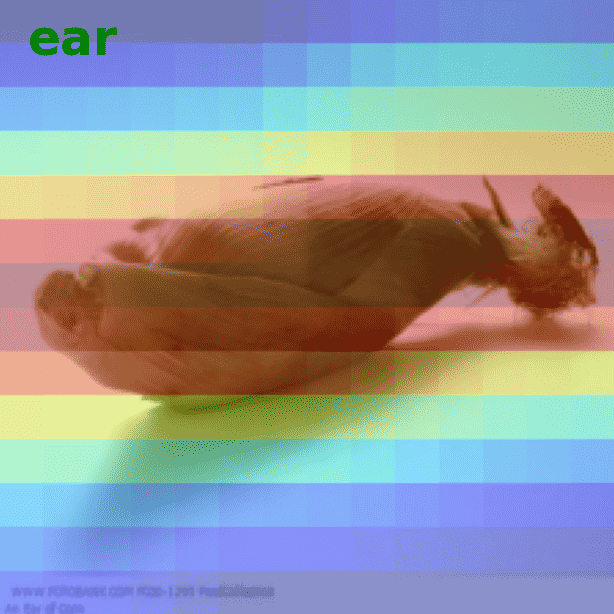}
\\
\includegraphics[width=0.13\linewidth]{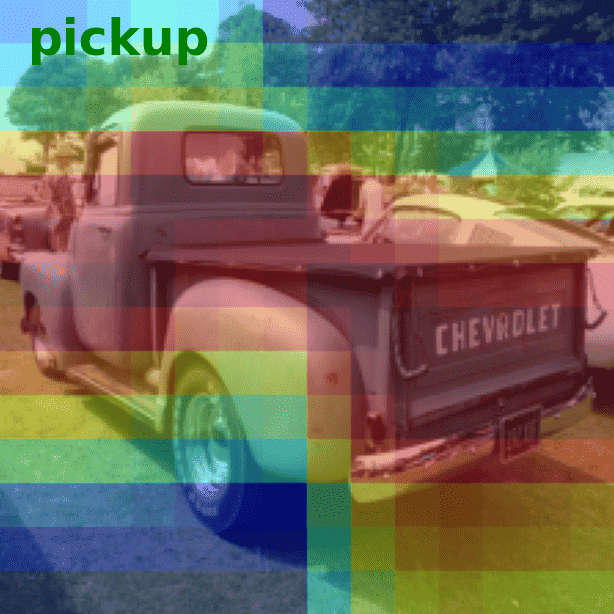}
\includegraphics[width=0.13\linewidth]{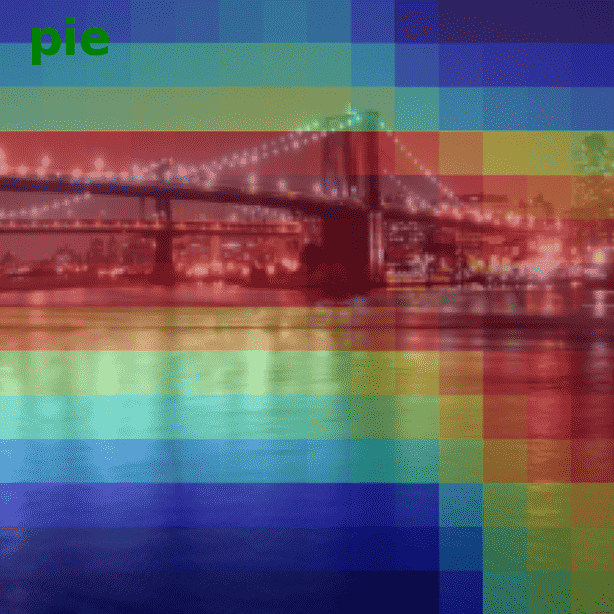}
\includegraphics[width=0.13\linewidth]{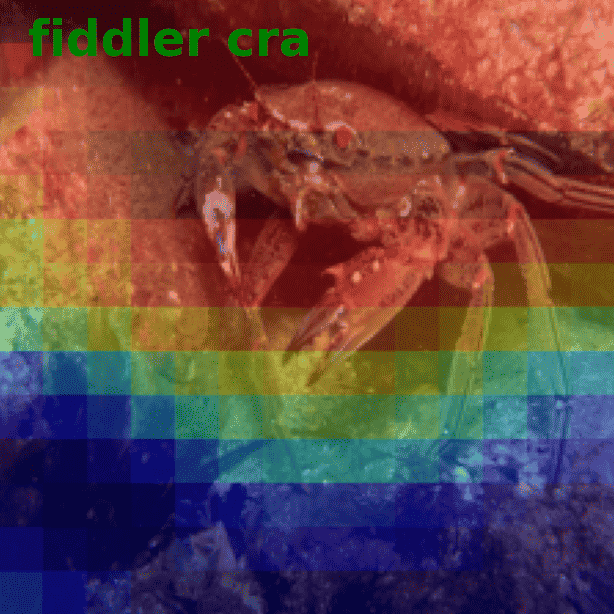}
\includegraphics[width=0.13\linewidth]{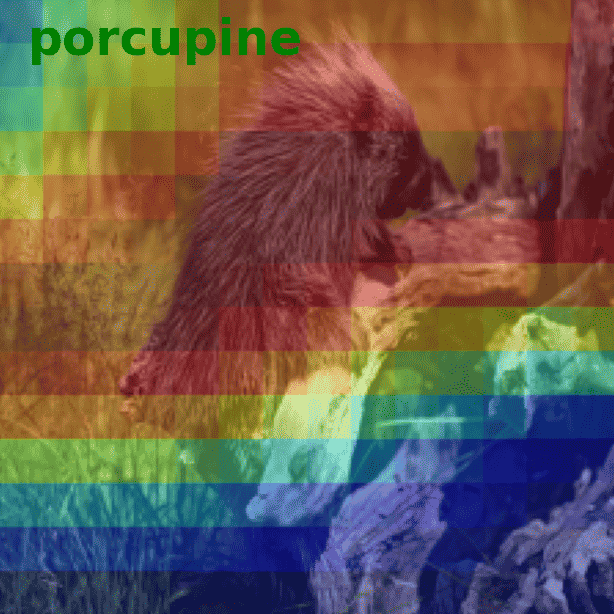}
\includegraphics[width=0.13\linewidth]{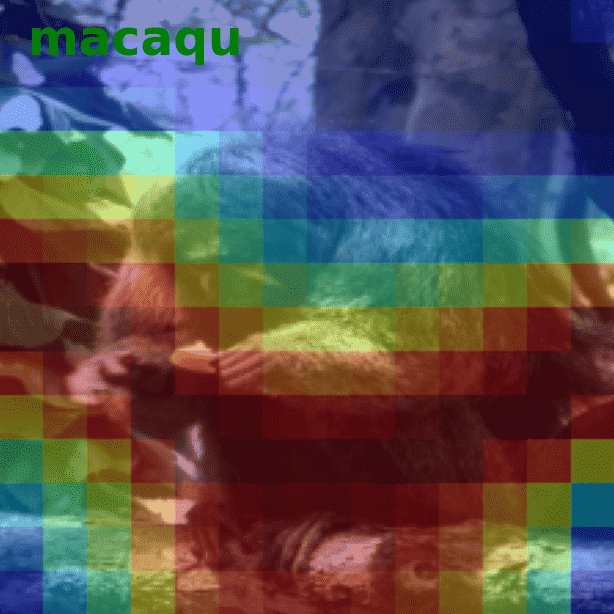}
\includegraphics[width=0.13\linewidth]{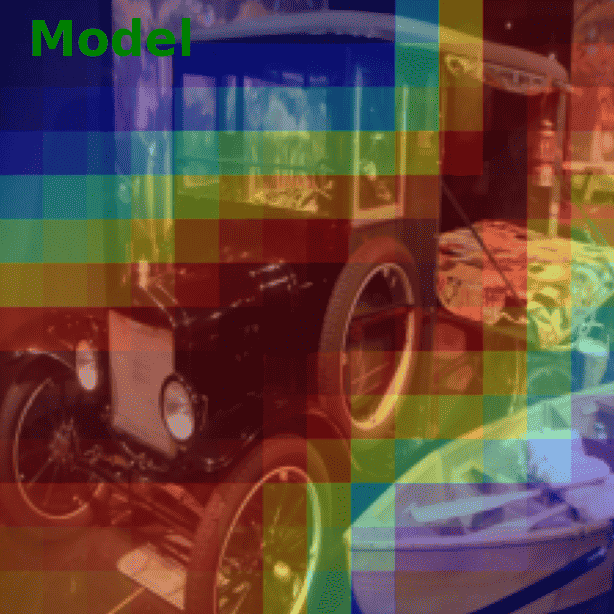}
\\
\includegraphics[width=0.13\linewidth]{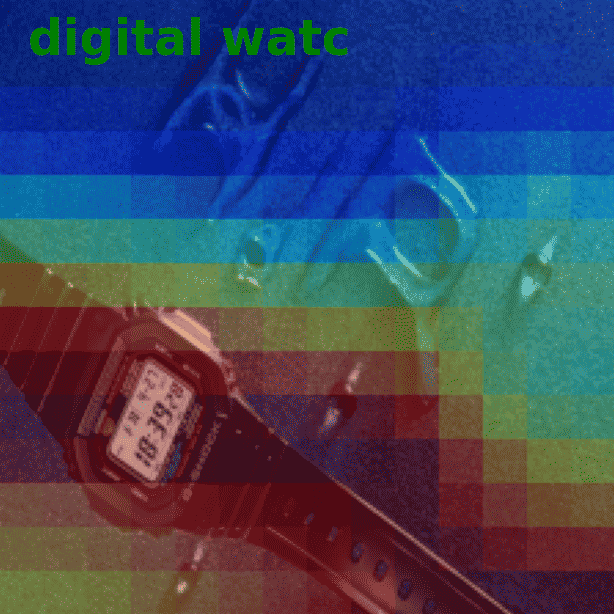}
\includegraphics[width=0.13\linewidth]{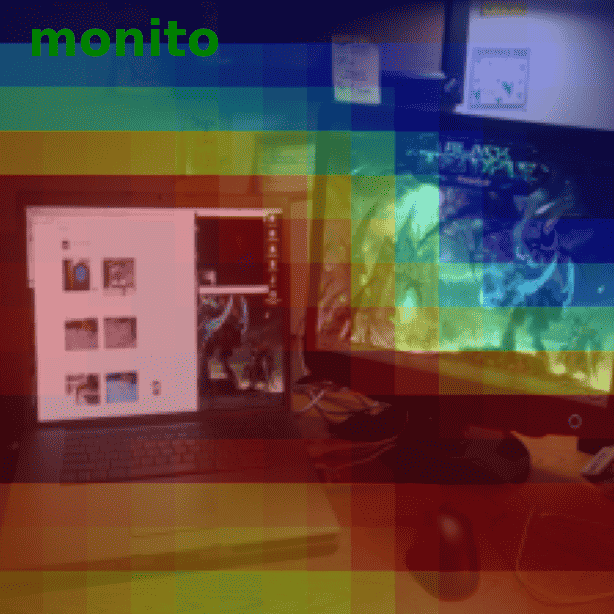}
\includegraphics[width=0.13\linewidth]{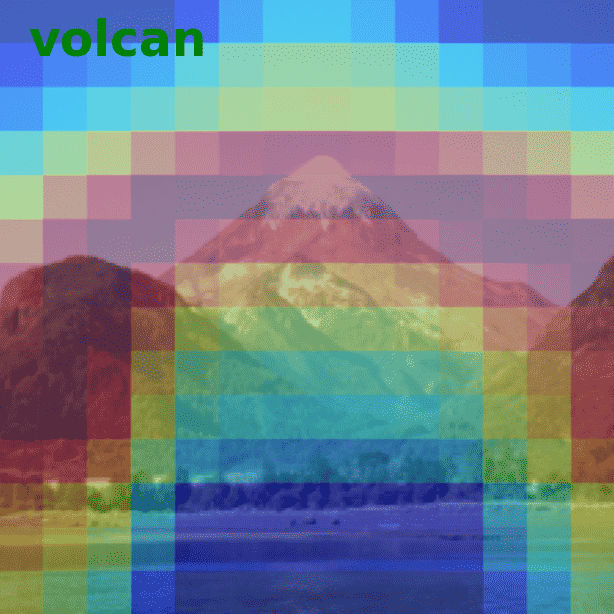}
\includegraphics[width=0.13\linewidth]{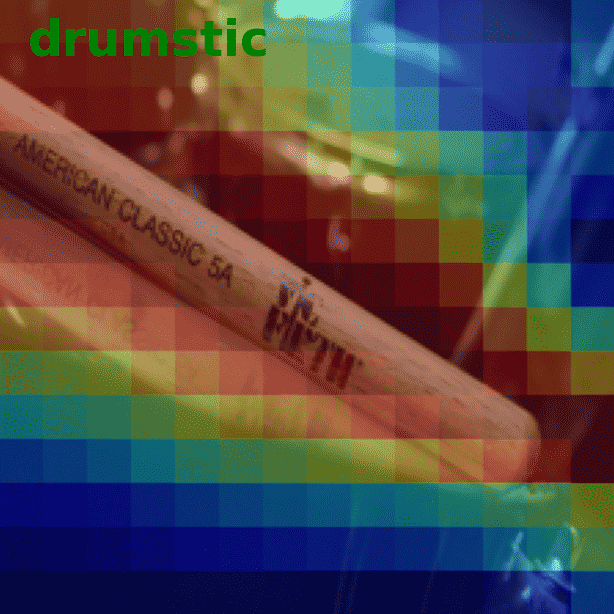}
\includegraphics[width=0.13\linewidth]{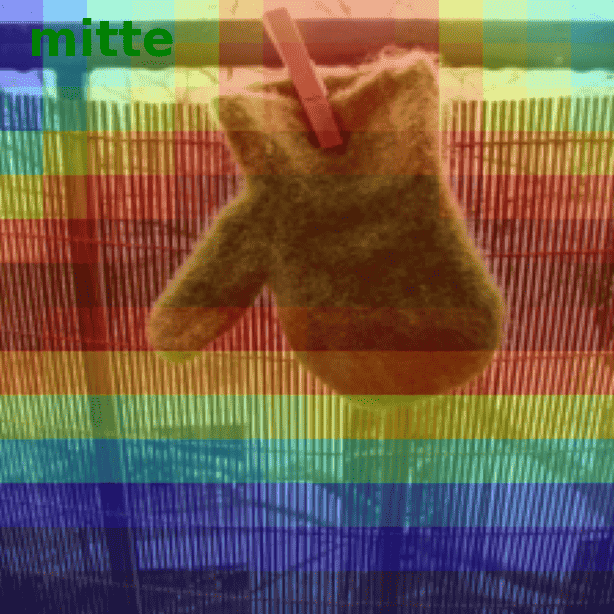}
\includegraphics[width=0.13\linewidth]{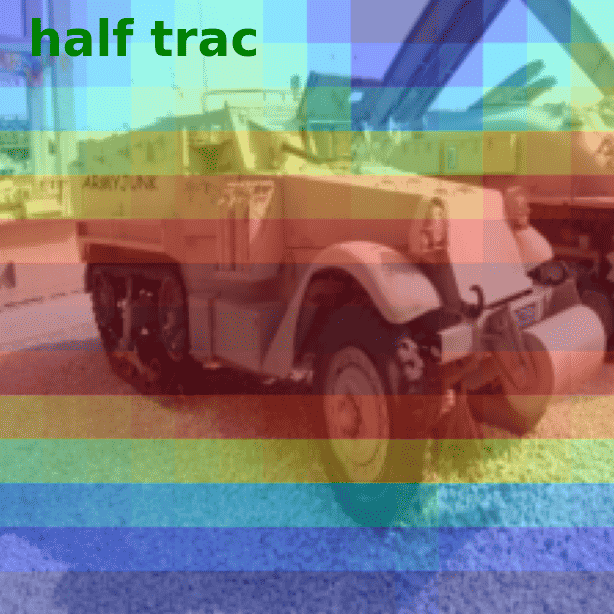}

\caption{\label{fig:more_case_study_vit_imagenet-1} Explanation heatmaps of \texttt{ViT-Base} on the \texttt{ImageNet} dataset.}
\end{figure}

\end{document}